\documentclass{article}
\usepackage{geometry}
\usepackage[pdftex]{graphicx}
\usepackage{subcaption}
\usepackage{cite}
\usepackage{array}
\usepackage{amsmath}
\usepackage{amssymb}
\usepackage{mathtools}
\usepackage{amsthm}
\usepackage{booktabs}
\usepackage{multirow}
\usepackage{makecell}
\usepackage{enumerate}
\usepackage{color}
\usepackage{algorithmic}
\usepackage{algorithm}

\newtheorem{theorem}{Theorem}
\newtheorem{lemma}{Lemma}
\newtheorem{question}{Question}
\newtheorem{example}{Example}
\newtheorem{assumption}{Assumption}
\newtheorem{definition}{Definition}

\begin{document}

\title{On the Necessity of Collaboration for Online Model Selection\\ with Decentralized Data}

\date{}

\author{
Junfan Li\textsuperscript{1}, Zheshun Wu\textsuperscript{1}, Zenglin Xu\textsuperscript{2}*,
Irwin King\textsuperscript{3}\\
\textsuperscript{1}~Harbin Institute of Technology Shenzhen\\
\textsuperscript{2}~Fudan University\\
\textsuperscript{3}~The Chinese University of Hong Kong\\
lijunfan@hit.edu.cn, \{zenglin, wuzhsh23\}@gmail.com, king@cse.cuhk.edu.hk\\
*~Corresponding Author}

\maketitle

\begin{abstract}
  We consider online model selection with decentralized data over $M$ clients,
  and study the necessity of collaboration among clients.
  Previous work proposed various federated algorithms without demonstrating their necessity,
  while we answer the question from a novel perspective of computational constraints.
  We prove lower bounds on the regret, and propose a federated algorithm and analyze the upper bound.
  Our results show
  (i) collaboration is unnecessary in the absence of computational constraints on clients;
  (ii) collaboration is necessary if the computational cost on each client is limited to $o(K)$,
  where $K$ is the number of candidate hypothesis spaces.
  We clarify the unnecessary nature of collaboration in previous federated algorithms
  for distributed online multi-kernel learning,
  and improve the regret bounds at a smaller computational and communication cost.
  Our algorithm relies on three new techniques including
  an improved Bernstein's inequality for martingale,
  a federated online mirror descent framework,
  and decoupling model selection and prediction,
  which might be of independent interest.
\end{abstract}


%

\section{Introduction}

    Model selection which is a fundamental problem for offline machine learning
    focuses on how to select a suitable hypothesis space for a machine learning algorithm
    \cite{Mitchell1997Machine,Bartlett2002Model,Mohri2018Foundations}.
    Model selection for online machine learning is called online model selection (OMS),
    such as model selection for online supervised learning \cite{Foster2017Parameter,Zhang2018Online,Zhang2021Regret,Li2022Improved},
    model selection for online active learning \cite{Karimi2021Online},
    and model selection for contextual bandits \cite{Foster2019,Pacchiano2020,Ghosh2022Model}.
    We consider model selection for online supervised learning.
    Let $\mathcal{F}=\{\mathcal{F}_1,\ldots,\mathcal{F}_K\}$ contain $K$ hypothesis spaces
    and $\ell(\cdot,\cdot)$ be a loss function.
    For a sequence of examples $\{({\bf x}_t,y_t)\}_{t=1,\ldots,T}$,
    we aim to adapt to the case that
    the optimal hypothesis space $\mathcal{F}_{i^\ast}\in\mathcal{F}$ is given by an oracle
    and we run an online learning algorithm in $\mathcal{F}_{i^\ast}$.
    OMS can be defined by minimizing the \textit{regret}, i.e.,
    \begin{align*}
        \min_{f_1,\ldots,f_T}\left(\sum^T_{t=1}\ell(f_t({\bf x}_t),y_t)
        -\min_{f\in\mathcal{F}_{i^\ast}}\sum^T_{t=1}\ell(f({\bf x}_t),y_t)\right),
    \end{align*}
    where $f_t\in\cup^K_{i=1}\mathcal{F}_i$ is the hypothesis used by an OMS algorithm at the $t$-th round.
    The optimal value of the regret depends on the complexity of $\mathcal{F}_{i^\ast}$
    \cite{Foster2017Parameter,Foster2019}.

    In this work,
    we consider online model selection with decentralized data (OMS-DecD) over $M$ clients,
    in which each client observes a sequence of examples
    $\left\{\left({\bf x}^{(j)}_t,y^{(j)}_t\right)\right\}_{t=1,\ldots,T}$, $j=1,\ldots,M$,
    and but does not share personalized data with others.
    There is a central server that coordinates the clients
    by sharing personalized models or gradients
    \cite{Konecny2016Federated,McMahan2017Communication,Kairouz2021Advances}.
    OMS-DecD captures some real-world applications where
    the data may be collected by sensors on $M$ different remote devices or mobile phones
    \cite{Li2020Federated,Patel2023Federated,Kwon2023Tighter},
    or a local device can not store all of data due to low storage
    and thus it is necessary to store the data on more local devices
    \cite{Slavakis2014Modeling,Bouboulis2018Online}.
    OMS-DecD can be defined by minimizing the following regret,
    \begin{align*}
        \min_{f^{(j)}_t,t=1,\ldots,T,j=1,\ldots,M}&
        \left(\sum^M_{j=1}\sum^T_{t=1}\ell\left(f^{(j)}_t\left({\bf x}^{(j)}_t\right),y^{(j)}_t\right)-
        \min_{f\in\mathcal{F}_{i^\ast}}\sum^M_{j=1}\sum^T_{t=1}
        \ell\left(f\left({\bf x}^{(j)}_t\right),y^{(j)}_t\right)\right),
    \end{align*}
    where
    $f^{(j)}_t\in\cup^K_{i=1}\mathcal{F}_i$ is the hypothesis adopted by the $j$-th client at the $t$-th round.
    Solving OMS-DecD must achieve two goals:
    \textbf{G}1 minimizing the regret, and
    \textbf{G}2 providing privacy protection.

    A trivial approach is to use a \textit{noncooperative algorithm}
    that independently runs a copy of an OMS algorithm on the $M$ clients.
    It naturally provides strong privacy protection, that is,
    it achieves \textbf{G}2,
    but suffers a regret bound that increases linearly with $M$.
    It is unknown whether it achieves \textbf{G}1.
    Another approach is \textit{federated learning}
    which is a framework of cooperative learning with privacy protection
    and is provably effective in stochastic convex optimization
    \cite{McMahan2017Communication,Woodworth2020Is,Wang2021A,Reddi2021Adaptive}.
    It is natural to ask:
    \begin{question}
    \label{que:ICML2024:question}
    \textit{Whether collaboration is effective for OMS-DecD.}
    \end{question}

    The question reveals the hardness of OMS-DecD
    and is helpful to understand the limitations of federated learning.
    Previous work studied a special instance of OMS-DecD
    called distributed online multi-kernel learning (OMKL)
    where $\mathcal{F}_i$ is a reproducing kernel Hilbert space (RKHS),
    and proposed three federated OMKL algorithms including
    vM-KOFL, eM-KOFL \cite{Hong2022Communication}
    and POF-MKL \cite{Ghari2022Personalized}.
    The three algorithms also suffer regret bounds that increase linearly with $M$,
    and thus can not answer the question.
    If $K=1$,
    then OMS-DecD is equivalent to distributed online learning
    \cite{Mitra2021Online,Kwon2023Tighter,Patel2023Federated}.
    A noncooperative algorithm that
    independently runs online gradient descent (OGD) on each client
    achieves the two goals simultaneously
    \cite{Patel2023Federated}.
    Collaboration is unnecessary in the case of $K=1$.

    In summary,
    previous work can not answer the question well.
    On one hand,
    previous work can not answer the question in the case of $K>1$.
    On the other hand,
    in the case of $K=1$,
    previous work has answered the question
    only using the statistical property of algorithms,
    i.e., the worst-case regret,
    but omitted the computational property
    which is very important for real-world applications.

\subsection{Main Results}

    In this paper,
    we will answer the question from a new perspective of computational constraints on the problem
    (Section \ref{sec:ICML2024:answer_to_question}).
    Our main results are as follows.
    \begin{enumerate}[(1)]
      \item \textbf{An upper bound on the regret.}~
            We propose a federated algorithm, FOMD-OMS,
            and prove an upper bound on the regret (Theorem \ref{thm:ICML2024:regret_bound:FOMD_DOKS}).
            Besides, if $\mathcal{F}_1,...,\mathcal{F}_K$ are RKHSs,
            then
            our algorithm improves the regret bounds of FOP-MKL \cite{Ghari2022Personalized}
            and eM-KOFL \cite{Hong2022Communication}
            at a smaller computational and communication cost.
            Table \ref{tab:ICML2023:comparison_results} summarizes the results.
      \item \textbf{Lower bounds on the regret.}~
            We separately prove a lower bounds on the regret of any (possibly cooperative) algorithm
            and any noncooperative algorithm (Theorem \ref{thm:NeurIPS24:lower_bound}).
      \item \textbf{A new perspective of computational constraints for Question \ref{que:ICML2024:question}.}~
            By the upper bound and lower bounds,
            we conclude that
            (i) collaboration is unnecessary when there are no computational constraints on clients,
            thereby generalizing the result for distributed online learning, i.e., $K=1$;
            (ii) collaboration is necessary if the computational cost on each client is limited to $o(K)$
             where irrelevant parameters are omitted.
            Our results clarify the unnecessary nature of collaboration
            in previous federated algorithms for distributed OMKL.
            Table \ref{tab:ICML2024:comparison_with_trivial_approach} gives several results.
    \end{enumerate}


    \begin{table*}[!t]
      \centering
      \caption{Comparison with noncooperative algorithm (NCO-OMS).
      NCO-OMS independently runs a copy of an OMS algorithm on $M$ clients.
      $\Xi_{i^\ast}=\mathfrak{C}_{i^\ast}M\sqrt{T\ln{K}}$.
      $\mathfrak{C}_i$ measures the complexity of $\mathcal{F}_i$.
      $\mathfrak{C}_{i^\ast}$ measures the complexity of $\mathcal{F}_{i^\ast}$.
      $\mathfrak{C}=\max_{i\in\{1,\ldots,K\}}\mathfrak{C}_i\geq \mathfrak{C}_{i^\ast}$.
      The communication cost is the upload cost or download cost (bits).
      Comp-cost represents the per-round time complexity (s).}
      \begin{tabular}{l|l|r|r|r}
      \toprule
        Constraint&{Algorithm}       & {Regret bound}  & Comp-cost &Communication cost\\
      \toprule
        \multirow{4}{*}{No(i.e., $R=T$)}
        & NCO-OMS   & $O\left(\mathfrak{C}_{i^\ast}M\sqrt{T\ln{(KT)}}\right)$ & $O(K)$ & $0$ \\
        & {\color{blue}FOMD-OMS} &$O\left(\mathfrak{C}_{i^\ast}M\sqrt{T\ln{(KT)}}\right)$ & $O(K)$
        & $O(KM)$ \\
        \cline{2-5}
        & NCO-OMS   & $\tilde{O}\left(\sqrt{K}\left(\Xi_{i^\ast}
        +M\sqrt{\mathfrak{C}\mathfrak{C}_{i^\ast}T}\right)\right)$ & $O(1)$ & $0$ \\
        & {\color{blue}FOMD-OMS} &{\color{blue}$\tilde{O}\left(\Xi_{i^\ast}
        +\sqrt{MK}\sqrt{\mathfrak{C}\mathfrak{C}_{i^\ast}T}\right)$} & $O(1)$
        & $O\left(M\log{K}\right)$\\
      \bottomrule
      \end{tabular}
      \label{tab:ICML2024:comparison_with_trivial_approach}
    \end{table*}

    \begin{table*}[!t]
      \centering
      \caption{Comparison with POF-MKL \cite{Ghari2022Personalized}
      and eM-KOFL \cite{Hong2022Communication}.
      $D$ is the number of random features \cite{Rahimi2007Random}.
      $R$ is the rounds of communication.
      $\tilde{O}(\cdot)$ hides polylogarithmic factor in $T$.
      For the sake of simplicity,
      we omit the factor $O(\log{K})$ in the communication cost of eM-KOFL and FOMD-OMS.
      The unit of upload cost and download cost is bits.}
      \begin{tabular}{l|l|r|r|r|r}
      \toprule
        Constraint&{Algorithm}       & {Regret bound}  & Comp-cost &Upload & download\\
      \toprule
        \multirow{4}{*}{No ($R=T$)}
        & eM-KOFL & $\tilde{O}\left(\mathfrak{C}M\sqrt{T\ln{K}}
        +\frac{\mathfrak{C}_{i^\ast}MT}{\sqrt{D}}\right)$ & $O(DK)$ &$O(DM)$&$O(DM)$\\
        & POF-MKL & $\tilde{O}\left(\mathfrak{C}M\sqrt{KT}
        +\frac{\mathfrak{C}_{i^\ast}MT}{\sqrt{D}}\right)$ & $O(DK)$ &$O(DM)$&$O(DKM)$\\
        & {\color{blue}FOMD-OMS} & {\color{blue}$\tilde{O}\left(\Xi_{i^\ast}
        +\sqrt{\mathfrak{C}\mathfrak{C}_{i^\ast}MKT}
        +\frac{\mathfrak{C}_{i^\ast}MT}{\sqrt{D}}\right)$} & {\color{blue}$O(D)$}
        & {\color{blue}$O(DM)$}& {\color{blue}$O(DM)$}\\
      \hline
        \multirow{4}{*}{Yes ($R<T$)}
        & eM-KOFL   & - & - &- &-\\
        & POF-MKL   & - & - &- &-\\
        & {\color{blue}FOMD-OMS} & {\color{blue}$\tilde{O}\left(\frac{\Xi_{i^\ast}}{\sqrt{R/T}}
        +\frac{\sqrt{\mathfrak{C}\mathfrak{C}_{i^\ast}MK}T}{\sqrt{R}}+\frac{\mathfrak{C}_{i^\ast}MT}{\sqrt{D}}\right)$}
        & {\color{blue}$O(D)$}
        &{\color{blue}$O\left(\frac{DMR}{T}\right)$}&{\color{blue}$O\left(\frac{DMR}{T}\right)$}\\
      \bottomrule
      \end{tabular}
      \label{tab:ICML2023:comparison_results}
    \end{table*}

\subsection{Technical Challenges}

    There are two main technical challenges on designing a federated online model selection algorithm.

    The first challenge lies in obtaining high-probability regret bounds
    that adapt to the complexity of individual hypothesis space,
    a fundamental problem in online model selection \cite{Foster2017Parameter}.
    While acquiring expected regret bounds
    that adapt to the complexity of individual hypothesis spaces is straightforward,
    the crux is to derive high-probability bounds from expected bounds.
    To this end,
    we introduce a new Bernstein's inequality for martingale
    (Lemma \ref{lemma:AISTATS2020:improved:Bernstein_ineq_for martingales}),
    which might be of independent interest.

    The second challenge involves achieving a per-round communication cost of $o(K)$.
    To tackle this challenge,
    we propose two techniques:
    (i) decoupling model selection and prediction;
    (ii) an algorithmic framework, named FOMD-No-LU, which might be of independent interest.
    Specifically,
    when clients execute model selection,
    server must broadcast an aggregated probability distribution,
    denoted by ${\bf p}\in\mathbb{R}^K$, to clients, naturally incurring a $O(K)$ download cost.
    Our algorithm conducts model selection on server and makes predictions on clients,
    thereby eliminating the need to broadcast the aggregated probability distribution to clients.
    Additionally,
    if we use the local updating approach \cite{Mitra2021Online,Patel2023Federated},
    then server must broadcast $K$ aggregated models to clients,
    also resulting in a $O(K)$ download cost \cite{Ghari2022Personalized}.
    By utilizing FOMD-No-LU,
    our algorithm only broadcasts the selected models to clients
    and can achieve a $o(K)$ download cost.

\section{Related work}

    Previous work has studied the necessity of collaboration for
    distributed bandit convex optimization \cite{Patel2023Federated},
    where a federated algorithmic framework named FEDPOSGD was proposed.
    Although the regret bounds of FEDPOSGD are smaller than some noncooperative algorithms,
    there is not a lower bound on the regret of any noncooperative algorithm \cite{Patel2023Federated}.
    Moreover,
    the regret analysis of FEDPOSGD is based on the analysis for federated online gradient descent
    that is not applicable to our algorithm, FOMD-OMS.
    The regret analysis of FOMD-OMS requires the
    analysis for federated online mirror descent with negative entropy regularizer.

    Our work is also different from federated bandits, such as federated $K$-armed bandits
    \cite{Wang2020Distributed} and federated linear contextual bandits \cite{Huang2021Federated}.
    For OMS-DecD,
    we do not assume that the examples $({\bf x}^{(j)}_t,y^{(j)}_t)$, $t=1,\ldots,T$,
    on each client are independent and identically distributed (i.i.d.).
    In contrast,
    in both federated $K$-armed bandits or federated linear contextual bandits,
    the rewards must be i.i.d.,
    thereby making collaboration effective.
    This is similar to the approach used in federated stochastic optimization.
    However, this may not hold true for OMS-DedD.
    Therefore,
    it is a distinctive problem for OMS-DecD to study whether collaboration is effective.

\section{Preliminaries and Problem Setting}

\subsection{Notations}

    Let $\mathcal{X}=\{\mathbf{x}\in\mathbb{R}^d\vert\Vert\mathbf{x}\Vert_2<\infty\}$
    be an instance space,
    $\mathcal{Y}=\{y\in\mathbb{R}:\vert y\vert<\infty\}$ be an output space,
    and $\mathcal{I}_T=\{(\mathbf{x}_t,y_t)\}_{t\in[T]}$ be a sequence of examples,
    where $[T] = \{1,\ldots,T\}$, $\mathbf{x}_t\in\mathcal{X}$ and $y_t\in\mathcal{Y}$.
    Let $S=\{s_1,s_2,\ldots\}$ be a finite set,
    $\mathrm{Uni}(S)$ be the uniform distribution over the elements in $S$
    and $s_{[T]}$ be the abbreviation of the sequence $s_1, s_2,\ldots,s_T$.
    Denote by $\mathbb{P}[A]$ the probability that an event $A$ occurs,
    $a\wedge b =\min\{a,b\}$, $a\vee b =\max\{a,b\}$ and $\log(a)=\log_2(a)$.
    Let $\psi_t(\cdot):\Omega\rightarrow \mathbb{R},t\in[T]$
    be a sequence of time-variant strongly convex regularizers defined on a domain $\Omega$.
    The Bregman divergence denoted by $\mathcal{D}_{\psi_t}(\cdot,\cdot)$,
    associated with $\psi_t(\cdot)$ is defined by
    \begin{align*}
        \forall {\bf u},{\bf v}\in \Omega,\quad\mathcal{D}_{\psi_t}(\mathbf{u},\mathbf{v})
        =\psi_t(\mathbf{u})-\psi_t(\mathbf{v})-\langle\nabla \psi_t(\mathbf{v}),\mathbf{u}-\mathbf{v}\rangle.
    \end{align*}

\subsection{Online Model Selection (OMS)}

    Let $\mathcal{F}=\{\mathcal{F}_1,...,\mathcal{F}_K\}$ contain $K$ hypothesis spaces
    where
    \begin{equation}
    \label{eq:ICML2024:definition:hypothesis_space}
        \mathcal{F}_i=\left\{f({\bf x})={\bf w}^\top\phi_i({\bf x}):
        \phi_i({\bf x})\in \mathbb{R}^{d_i},\Vert {\bf w}\Vert_2\leq U_i\right\},
    \end{equation}
    where $\Vert\cdot\Vert_2$ is the $L_2$ norm.
    Let $\mathcal{F}_{i^\ast}\in\mathcal{F}$ be the optimal but unknown hypothesis space
    for a given $\mathcal{I}_T$.
    OMS can be defined as follows:
    generating a sequence of hypotheses $f_{[T]}$ that minimizes the following \textit{regret},
    \begin{align*}
        \forall i\in[K],\quad
        \mathrm{Reg}(\mathcal{F}_i)
        =\sum^T_{t=1}\ell(f_t({\bf x}_t),y_t)-\min_{f\in\mathcal{F}_i}\sum^T_{t=1}\ell(f({\bf x}_t),y_t),
    \end{align*}
    where $f_t\in\cup^K_{i=1}\mathcal{F}_i$.
    The optimal hypothesis space $\mathcal{F}_{i^\ast}$
    must contain a good hypothesis and has a low complexity \cite{Foster2017Parameter,Foster2019},
    and is defined by
    \begin{align*}
        \mathcal{F}_{i^\ast}=\mathop{\arg\min}_{\mathcal{F}_i\in\mathcal{F}}
        \left[\min_{f\in\mathcal{F}_i}\sum^T_{t=1}\ell(f({\bf x}_t),y_t)
        +\Theta\left(\sqrt{T\cdot\mathfrak{C}_i}\right)\right],
    \end{align*}
    where $\mathfrak{C}_i$ measures the complexity of $\mathcal{F}_i$,
    such as $U_i$ and $d_i$.

    OMS is more challenge than online learning,
    since we not only learn the optimal hypothesis space,
    but also learn the optimal hypothesis in the space.
    Next we give some examples of OMS.
    \begin{example}[Online Hyper-parameters Tuning]
    \label{ex:ICML24:regularization_tunning}
        Let $\mathcal{F}_i$ consist of linear functions of the form
        \begin{align*}
            \mathcal{F}_i=\left\{f({\bf x})=\langle {\bf w},{\bf x}\rangle,\Vert{\bf w}\Vert_2\leq U_i\right\},
        \end{align*}
        where $U_i>0$ is a regularization parameter.
        Let $\mathcal{U}=\{U_i,i\in[K]:U_1<U_2<\ldots<U_K\}$.
        The hypothesis spaces are nested, i.e.,
        $
            \mathcal{F}_1\subseteq \mathcal{F}_2\subseteq\ldots\subseteq\mathcal{F}_K.
        $
        The optimal regularization parameter $U_{i^\ast}\in\mathcal{U}$ corresponds to
        the optimal hypothesis space $\mathcal{F}_{i^\ast}\in\mathcal{F}$.
    \end{example}
    \begin{example}[Online Kernel Selection \cite{Shen2019Random,Li2022Improved}]
    \label{ex:ICML24:OKS}
        Let $\kappa_i(\cdot,\cdot):\mathbb{R}^d \times \mathbb{R}^d \rightarrow \mathbb{R}$
        be a positive semidefinite kernel function,
        and $\phi_i:\mathbb{R}^d\rightarrow\mathbb{R}^{d_i}$ be the associated feature mapping.
        $\mathcal{F}_i$ is the RKHS associated with $\kappa_i$, i.e.,
        \begin{align*}
            \mathcal{F}_i=\left\{f({\bf x})=\langle {\bf w},\phi_i({\bf x})\rangle:
            \Vert {\bf w}\Vert_2\leq U_i\right\}.
        \end{align*}
        The optimal kernel function $\kappa_{i^\ast}\in\{\kappa_1,\ldots,\kappa_K\}$ corresponds to
        the optimal RKHS $\mathcal{F}_{i^\ast}\in\mathcal{F}$.
    \end{example}

    \begin{example}[Online Pre-trained Classifier Selection \cite{Karimi2021Online}]
        Generally,
        $\mathcal{F}_i$ can be a well-trained machine learning model.
        Let $\mathcal{F}$ contain $K$ pre-trained classifiers.
        For a new instance ${\bf x}_t$,
        we select a (combinational) pre-trained classifier and make a prediction.
        The selection of a pre-trained classifier has an important implication in practical scenarios.
    \end{example}

\subsection{Online Model Selection with Decentralized Data (OMS-DecD)}

    We formally define OMS-DecD as follows.
    Assuming that there are $M$ clients and a server.
    At any round $t$,
    each client observes an instance ${\bf x}^{(j)}_t$,
    and selects a hypothesis $f^{(j)}_t\in\cup^K_{i=1}\mathcal{F}_i$, $j\in[M]$.
    Then clients output predictions $\{f^{(j)}_t({\bf x}^{(j)}_t)\}^M_{j=1}$.
    The goal is to minimize the following regret
    \begin{align*}
        \forall i\in[K],\quad \mathrm{Reg}_{D}(\mathcal{F}_i)
        =\sum^T_{t=1}\sum^M_{j=1}\ell\left(f^{(j)}_t({\bf x}^{(j)}_t),y^{(j)}_t\right)
        -\min_{f\in\mathcal{F}_i}\sum^T_{t=1}\sum^M_{j=1}\ell\left(f({\bf x}^{(j)}_t),y^{(j)}_t\right),
    \end{align*}
    where $y^{(j)}_t$ is the label or true output.
    Each client can not share personalized data with others,
    but can share personalized models or gradients via the central server.
    For simplicity,
    we define OMS-DecD in Protocol \ref{pro:ICML2024:DOMS}.

    \begin{algorithm}[!t]
        \captionsetup{labelformat=empty} 
        \caption{\textbf{Protocol 1} OMS-DecD}
        \footnotesize
        \label{pro:ICML2024:DOMS}
        \begin{algorithmic}[1]
        \FOR{$t=1,2,\ldots,T$}
            \FOR{$j=1,\ldots,M$ in parallel}
                \STATE The adversary sends ${\bf x}^{(j)}_t$ to the $j$-th client
                \STATE The learner selects a hypothesis space $\mathcal{F}_{I_t}\in\mathcal{F}$
                \STATE The learner selects $f^{(j)}_t\in\mathcal{F}_{I_t}$
                        and outputs $f^{(j)}_t({\bf x}^{(j)}_t)$
                \STATE The learner observes the true output $y^{(j)}_t$
            \ENDFOR
        \ENDFOR
        \end{algorithmic}
    \end{algorithm}

\section{FOMD-No-LU}

    In this section,
    we propose a federated algorithmic framework,
    FOMD-No-LU (Federated Online Mirror Descent without Local Updating)
    for online collaboration.

\subsection{Federated Algorithmic Framework}

    Let $\Omega$ be a convex and bounded decision set.
    At any round $t$,
    each client $j\in[M]$ first selects a decision ${\bf u}^{(j)}_t\in \Omega$,
    and then observes a loss function $l^{(j)}_t(\cdot):\Omega\rightarrow\mathbb{R}$.
    The client computes the loss $l^{(j)}_t({\bf u}^{(j)}_t)$
    and an estimator of the gradient denoted by $\tilde{g}^{(j)}_t$
    (or the gradient denoted by $g^{(j)}_t$).
    To reduce the communication cost,
    we adopt the intermittent communication (IC) protocol \cite{Woodworth2021The},
    in which the clients communicate with the server every $N$ rounds.
    Assuming that $T=N\times R$ where $N,R\in\mathbb{Z}$,
    the IC protocol limits the rounds of communication to $R$.

    We divide $[T]$ into $R$ disjoint sub-intervals denoted by $\{T_r\}^R_{r=1}$,
    in which
    \begin{equation}
    \label{eq:ICML2024:T_r}
        T_r=\left\{(r-1)N+1,(r-1)N+2,\ldots,rN\right\}.
    \end{equation}
    For any $t\in T_r$,
    all clients always select the initial decision,
    \begin{equation}
    \label{eq:ICML2024:FOMD-No-LU:batching}
        \forall j\in[M],~\forall t\in T_r,\quad {\bf u}^{(j)}_{t}={\bf u}^{(j)}_{(r-1)N+1}.
    \end{equation}
    At the end of the $rN$-round,
    all of clients send $\frac{1}{N}\sum_{t\in T_r}\tilde{g}^{(j)}_t$, $j\in[M]$ to server.
    Then the server updates the decision using
    online mirror descent framework \cite{Bubeck2012Regret,Agarwal2017Corralling},
    \begin{align}
        {\bf u}_t&=\frac{1}{M}\sum^M_{j=1}{\bf u}^{(j)}_t,\label{eq:ICML2024:FOMD-No-LU:first_step}\\
        \bar{g}_t&=\frac{1}{M}\sum^M_{j=1}\left(\frac{1}{N}\sum_{t\in T_r}\tilde{g}^{(j)}_{t}\right),
        \label{eq:ICML2024:FOMD-No-LU:second_step}\\
        \nabla_{\bar{{\bf u}}_{t+1}}\psi_t(\bar{{\bf u}}_{t+1})
        &=\nabla_{{\bf u}_t}\psi_t({\bf u}_{t})-\bar{g}_t,
        \label{eq:ICML2024:FOMD-No-LU:third_step}\\
        {\bf u}_{t+1}&=\mathop{\arg\min}_{{\bf u}\in\Omega}
        \mathcal{D}_{\psi_t}({\bf u},\bar{{\bf u}}_{t+1}).
        \label{eq:ICML2024:FOMD-No-LU:fourth_step}
    \end{align}
    \eqref{eq:ICML2024:FOMD-No-LU:first_step}-\eqref{eq:ICML2024:FOMD-No-LU:third_step}
    is called model averaging
    \cite{McMahan2017Communication}
    and shows the collaboration among clients.
    Finally, the server may broadcast ${\bf u}_{t+1}$ to all clients, i.e.,
    $$
        \forall j\in[M],\quad {\bf u}^{(j)}_{t+1}={\bf u}_{t+1}.
    $$
    Let the initial decision ${\bf u}^{(j)}_1={\bf u}_1$ for all $j\in[M]$,
    then it must be ${\bf u}^{(j)}_t={\bf u}_t$ for all $t\in[T]$.
    Thus \eqref{eq:ICML2024:FOMD-No-LU:first_step} is unnecessary,
    and the clients do not transmit ${\bf u}^{(j)}_t$ to server.
    The pseudo-code of FOMD-No-LU is shown in Algorithm
    \ref{alg:ICML2024:FOMD-No-LU}.

    \begin{algorithm}[!t]
        \caption{FOMD-No-LU}
        \footnotesize
        \label{alg:ICML2024:FOMD-No-LU}
        \begin{algorithmic}[1]
        \REQUIRE{$\Omega$.}
        \ENSURE{${\bf u}^{(j)}_1,j\in[M]$}
        \FOR{$r=1,2,\ldots,R$}
            \FOR{$t=(r-1)N+1,\ldots,rN$}
                \FOR{$j=1,\ldots,M$ in parallel}
                    \STATE Selecting ${\bf u}^{(j)}_{(r-1)N+1}$
                    \STATE Observing loss function $l^{(j)}_t(\cdot)$
                    \STATE Computing gradient (or an estimator of gradient) $\tilde{g}^{(j)}_t$
                    \IF{$t==rN$}
                        \STATE Transmitting $\frac{1}{N}\sum_{t\in T_r}\tilde{g}^{(j)}_t$ to server
                    \ENDIF
                \ENDFOR
            \IF{$t==rN$}
                \STATE Server computes ${\bf u}_{t+1}$
                        following
                        \eqref{eq:ICML2024:FOMD-No-LU:second_step},
                        \eqref{eq:ICML2024:FOMD-No-LU:third_step}
                        and
                        \eqref{eq:ICML2024:FOMD-No-LU:fourth_step}
                \STATE Server may broadcast ${\bf u}_{t+1}$: ${\bf u}^{(j)}_{t+1}={\bf u}_{t+1}$, $j\in[M]$
            \ENDIF
            \ENDFOR
        \ENDFOR
        \end{algorithmic}
    \end{algorithm}

\subsection{Regret Bound}

    We give the regret bounds of FOMD-No-LU.
    \begin{theorem}
    \label{lemma:ICML2024:regret_OMD}
        Let $R=T$.
        Assuming that $l^{(j)}_t:\Omega\rightarrow \mathbb{R},t\in[T],j\in[M]$ is convex.
        Let $g^{(j)}_t=\nabla_{{\bf u}^{(j)}_t} l^{(j)}_t({\bf u}^{(j)}_t)$
        and $\tilde{g}^{(j)}_t$ be an estimator of $g^{(j)}_t$.
        At any round $t$,
        let ${\bf q}_{t+1}$ and ${\bf r}_{t+1}$ be two auxiliary decisions
        defined as follows,
        \begin{align}
            \nabla_{{\bf q}_{t+1}}\psi_t({\bf q}_{t+1})
            =&\nabla_{{\bf u}_{t}}\psi_t({\bf u}_{t})
            -2\sum^M_{j=1}\frac{\tilde{g}^{(j)}_t-g^{(j)}_t}{M},
            \label{eq:ICML2024:OMD_auxiliary:q}\\
            \nabla_{{\bf r}_{t+1}}\psi_t({\bf r}_{t+1})
            =&\nabla_{{\bf u}_{t}}\psi_t({\bf u}_{t})
            -\frac{2}{M}\sum^M_{j=1}g^{(j)}_t.
            \label{eq:ICML2024:OMD_auxiliary:r}
        \end{align}
        Then FOMD-No-LU guarantees that,
        \begin{align*}
            &\forall {\bf v}\in\Omega,\quad \frac{1}{M}\sum^T_{t=1}\sum^M_{j=1}\left(l^{(j)}_t({\bf u}^{(j)}_{t})-l^{(j)}_t({\bf v})\right)\\
            \leq&\underbrace{\sum^T_{t=1}\left[\mathcal{D}_{\psi_t}({\bf v},{\bf u}_t)
            -\mathcal{D}_{\psi_t}({\bf v},{\bf u}_{t+1})
            +\frac{\mathcal{D}_{\psi_t}({\bf u}_{t},{\bf r}_{t+1})}{2}\right]}_{\Xi_1}+
            \underbrace{\sum^T_{t=1}\left[\frac{\mathcal{D}_{\psi_t}({\bf u}_{t},{\bf q}_{t+1})}{2}
            +\sum^M_{j=1}\frac{\left\langle \tilde{g}^{(j)}_t-g^{(j)}_t,{\bf u}_{t}-{\bf v}\right\rangle}{M}
            \right]}_{\Xi_2}.
        \end{align*}
    \end{theorem}

    It is intriguing that the regret bound comprises two components:
    the first part, $\Xi_1$, cannot be reduced by collaboration,
    while the second part, $\Xi_2$, highlights the benefits of collaboration.
    $\Xi_1$ is the regret induced by exact gradients,
    while
    $\Xi_2$ is the regret induced by estimated gradients
    and shows how collaboration controls the regret.
    It is worth mentioning that
    Theorem \ref{lemma:ICML2024:regret_OMD} gives a general regret bound,
    from which various types of regret bounds can be readily derived
    by instantiating the decision set $\Omega$ and the regularizer $\psi_t(\cdot)$.
    For instance,
    if $\Omega=\mathcal{F}_i$ where $\mathcal{F}_i$ follows
    Example \ref{ex:ICML24:regularization_tunning},
    $\psi_t({\bf v})=\frac{1}{2\lambda}\Vert {\bf v}\Vert^2_2$
    and $\mathbb{E}[\Vert \tilde{g}^{(j)}_t\Vert^2_2]\leq C\Vert g^{(j)}_t\Vert^2_2$,
    then FOMD-No-LU becomes a federated online descent descent.
    It is easy to give a $O(MU_i\sqrt{(1+\frac{C}{M})T})$ expected regret
    from Theorem \ref{lemma:ICML2024:regret_OMD}.
    Besides, $N>1$ increases the regret and
    shows the trade-off between communication cost and regret bound.

    Theorem \ref{lemma:ICML2024:regret_OMD} requires a novel analysis on
    how the bias of estimators, i.e.,
    $\sum^M_{j=1}\Vert \tilde{g}^{(j)}_t-g^{(j)}_t\Vert^2_2$, is controlled by cooperation.
    To this end, we introduce two virtual decisions ${\bf q}_{t+1}$ and ${\bf r}_{t+1}$
    that are updated by $2\sum^M_{j=1}\frac{\tilde{g}^{(j)}_t-g^{(j)}_t}{M}$ and
    $2\sum^M_{j=1}\frac{g^{(j)}_t}{M}$, respectively.
    Previous federated online mirror descent uses exact gradients $g^{(j)}_t,j\in[M]$ \cite{Mitra2021Online}.
    Thus its analysis is different from ours.

    \begin{theorem}
    \label{lemma:ICML2024:regret_OMD_2}
        Let $R<T$
        and $\mathcal{R}=\{N,2N,\ldots,RN\}$.
        At any round $t\in\mathcal{R}$,
        let ${\bf q}_{t+1}$ and ${\bf r}_{t+1}$ be two auxiliary decisions
        which follow \eqref{eq:ICML2024:OMD_auxiliary:q}
        and \eqref{eq:ICML2024:OMD_auxiliary:r}.
        Under the assumptions in Theorem \ref{lemma:ICML2024:regret_OMD},
        FOMD-No-LU guarantees that,
        \begin{align*}
            &\forall {\bf v}\in\Omega,\quad \frac{1}{NM}\sum^T_{t=1}\sum^M_{j=1}\left(l^{(j)}_t({\bf u}^{(j)}_{t})-l^{(j)}_t({\bf v})\right)\\
            \leq&\sum_{t\in \mathcal{R}}\left[\mathcal{D}_{\psi_t}({\bf v},{\bf u}_t)
            -\mathcal{D}_{\psi_t}({\bf v},{\bf u}_{t+1})
            +\frac{\mathcal{D}_{\psi_t}({\bf u}_{t},{\bf r}_{t+1})}{2}\right]+
            \sum_{t\in \mathcal{R}}\left[\frac{\mathcal{D}_{\psi_t}({\bf u}_{t},{\bf q}_{t+1})}{2}
            +\sum^M_{j=1}\frac{\left\langle \tilde{g}^{(j)}_t-g^{(j)}_t,{\bf u}_{t}-{\bf v}\right\rangle}{M}
            \right].
        \end{align*}
    \end{theorem}

    It is obvious that $N>1$ increases the regret,
    that is, the reduction on the communication cost is at the cost of regret,
    which shows the trade-off between communication cost and regret bound.
    We will explicitly give the trade-off.

\subsection{Comparison with Previous Work}

    In fact, FOMD-No-LU adopts the batching technique \cite{Dekel2011Optimal},
    that is,
    it divides $[T]$ into $R$ sub-intervals
    and executes \eqref{eq:ICML2024:FOMD-No-LU:batching} during each sub-intervals.
    The batching technique (also known as mini-batch)
    has been used in the multi-armed bandit problem \cite{Dekel2012Online}
    and distributed stochastic convex optimization \cite{Karimireddy2020SCAFFOLD,Woodworth2020Minibatch}.
    We use the batching technique for the first time to distributed online learning.

    FOMD-No-LU is different from FedOMD (federated online mirror descent) \cite{Mitra2021Online}.
    (i) FedOMD only transmits exact gradients,
    while FOMD-No-LU can transmit estimators of gradient.
    Thus the regret bound of FedOMD did not contain $\Xi_2$
    in Theorem \ref{lemma:ICML2024:regret_OMD}.
    (ii) FedOMD uses local updating,
    such as local OGD \cite{Patel2023Federated} and
    local SGD \cite{McMahan2017Communication,Reddi2021Adaptive}.
    Thus FedOMD induces the client drift, i.e., ${\bf u}^{(j)}_t\neq {\bf u}_t$.
    Besides, if we use FedOMD,
    then the download cost is in $O(MK)$.

\section{OMS-DecD without Communication Constraints}

    At a high level, our algorithm comprises two components
    both of which are critical for achieving a communication cost in $o(K)$:
    (i) decoupling model selection and online prediction;
    (ii) collaboratively updating decisions within the framework of FOMD-No-LU.

\subsection{Decoupling Model Selection and Prediction}

\subsubsection{Model Selection on Server}

    At any round $t$,
    server maintains $K$ hypotheses $\{f^{(j)}_{t,i}\in\mathcal{F}_i\}^K_{i=1}$
    and a probability distribution ${\bf p}^{(j)}_t$ over the $K$ hypotheses for all $j\in[M]$.
    The model selection process aims to select a hypothesis
    from $\{f^{(j)}_{t,i}\}^K_{i=1}$ and then predicts the output of ${\bf x}^{(j)}_t$.
    An intuitive idea is that,
    for each $j\in[M]$,
    the client samples a hypothesis following ${\bf p}^{(j)}_t$.
    However,
    such an approach requires that server broadcasts ${\bf p}^{(j)}_t$ to clients,
    and will cause a download cost in $O(K)$.

    The sampling operation (or model selection process) can be executed on server.
    Specifically, server just broadcasts the selected hypotheses,
    and thus saves the communication cost.
    For each $j\in[M]$,
    server selects $J\in[2,K]$ hypotheses denoted by
    $f^{(j)}_{t,A_{t,a}}$, $a\in[J]$ where $A_{t,a}\in[K]$.
    For simplicity,
    let $O^{(j)}_t=\{A_{t,1},\ldots,A_{t,J}\}$.
    We instantiate ${\bf u}_t={\bf p}_t$ in FOMD-No-LU.
    Then FOMD-No-LU ensures ${\bf p}^{(j)}_t={\bf p}_t$ for all $j\in[M]$.
    We sample $A_{t,1},\ldots,A_{t,J}$ in order and follow \eqref{eq:ICML2024:kernel_selection}.
    \begin{equation}
    \label{eq:ICML2024:kernel_selection}
    \begin{split}
        &A_{t,1}\sim {\bf p}_t,\\
        &A_{t,a}\sim\mathrm{Uni}([K]\setminus\{A_{t,1},\ldots,A_{t,a-1}\}),~a\in[2,J].
    \end{split}
    \end{equation}
    It is easy to prove that
    $$
        \forall i\in[K],\quad
        \mathbb{P}\left[i\in O^{(j)}_t\right]
        =\frac{K-J}{K-1}p_{t,i}+\frac{J-1}{K-1}.
    $$
    Server samples $O^{(j)}_t$ for all $j\in[M]$
    and thus must independently execute \eqref{eq:ICML2024:kernel_selection} $M$ times
    which only pays an additional computational cost in $O(M\log{K})$.
    The factor $\log{K}$ arises from the process of sampling a number from $\{1,...,K\}$.
    Server only sends $f^{(j)}_{A_{t,a}}$, $a\in[J]$ to the $j$-th client.
    It is worth mentioning that server does not send ${\bf p}_t$.
    The total download cost is $O(\sum^M_{j=1}\sum^J_{a=1}(d_{A_{t,a}}+\log{K}))$.
    If $J$ is independent of $K$,
    then the download cost is only $O(M\log{K})$.

\subsubsection{Prediction on Clients}

    For each $j\in[M]$,
    the $j$-th client receives $f^{(j)}_{A_{t,a}}$, $a\in[J]$,
    and uses $f^{(j)}_{t,A_{t,1}}$ to output a prediction, i.e.,
    \begin{align*}
        \hat{y}^{(j)}_t=f^{(j)}_{t,A_{t,1}}\left({\bf x}^{(j)}_t\right)
        =\left\langle{\bf w}^{(j)}_{t,A_{t,1}},\phi_{A_{t,1}}({\bf x}^{(j)}_t)\right\rangle,
    \end{align*}
    where we assume that $f^{(j)}_{t,i}$ is parameterized by ${\bf w}^{(j)}_{t,i}\in\mathbb{R}^{d_i}$
    (see \eqref{eq:ICML2024:definition:hypothesis_space}).
    After observing the true output $y^{(j)}_t$,
    the client suffers a loss $\ell(f^{(j)}_{t,A_{t,1}}({\bf x}^{(j)}_{t}),y^{(j)}_t)$.

    It is worth mentioning that the other $J-1$ hypotheses $f^{(j)}_{t,A_{t,a}}$, $a\geq 2$
    are just used to obtain more information on the loss function.
    We will explain more in the following subsection.
    Thus we do not cumulate the loss $\ell(f^{(j)}_{t,A_{t,a}}({\bf x}^{(j)}_{t}),y^{(j)}_t)$, $a\geq 2$.

\subsection{Online Collaboration Updating}

    We use FOMD-No-LU to update the sampling probabilities and the hypotheses.

\subsubsection{Updating sampling probabilities}

    For each $j\in[M]$,
    let ${\bf c}^{(j)}_t=(c^{(j)}_{t,1},\ldots,c^{(j)}_{t,K})$
    where $c^{(j)}_{t,i}=\ell(f^{(j)}_{t,i}({\bf x}^{(j)}_t),y^{(j)}_t)$
    is the loss of $f^{(j)}_{t,i}$, $i\in[K]$.
    The $j$-th client will send $c^{(j)}_{t,i}$, $i\in O^{(j)}_t$, to server.
    Since $c^{(j)}_{t,i}$, $i\notin O^{(j)}_t$ can not be observed,
    it is necessary to construct an estimated loss vector
    $\tilde{{\bf c}}^{(j)}_t=(\tilde{c}^{(j)}_{t,1},\ldots,\tilde{c}^{(j)}_{t,K})$
    where
    $$
        \tilde{c}^{(j)}_{t,i}=\frac{c^{(j)}_{t,i}}{\mathbb{P}[i\in O^{(j)}_{t}]}
    \cdot\mathbb{I}_{i\in O^{(j)}_{t}}, i\in[K].
    $$
    It is easy to prove that $\mathbb{E}_t\left[\tilde{c}^{(j)}_{t,i}\right]
    =c^{(j)}_{t,i}$
    and $\mathbb{E}_t\left[(\tilde{c}^{(j)}_{t,i})^2\right]\leq \frac{K-1}{J-1}(c^{(j)}_{t,i})^2$
    where
    $\mathbb{E}_t[\cdot]:=\mathbb{E}\left[\cdot\vert O^{(j)}_{[t-1]}\right]$.
    Thus sampling $A_{t,a}$, $a\geq 2$ reduces the variance of the estimators
    which is equivalent to obtain more information on the true loss.

    Server aggregates $\tilde{{\bf c}}^{(j)}_t$, $j\in[M]$ and updates ${\bf p}_t$ following
    \eqref{eq:ICML2024:FOMD-No-LU:second_step}-\eqref{eq:ICML2024:FOMD-No-LU:fourth_step}.
    Let $\Delta_K$ be the $(K-1)$-dimensional simplex,
    $\Omega=\Delta_K$ and $\tilde{g}^{(j)}_t=\tilde{{\bf c}}^{(j)}_t$.
    Then the server executes \eqref{eq:ICML2024:updating_probability}.
    \begin{equation}
    \label{eq:ICML2024:updating_probability}
    \left\{
    \begin{split}
        \bar{\bf c}_t&=\frac{1}{M}\sum^M_{j=1}\tilde{{\bf c}}^{(j)}_t,\\
        \nabla_{\bar{{\bf p}}_{t+1}}\psi_t(\bar{{\bf p}}_{t+1})
        &=\nabla_{{\bf p}_t}\psi_t({\bf p}_{t})-\bar{{\bf c}}_t,\\
        {\bf p}_{t+1}&=\mathop{\arg\min}_{{\bf p}\in\Delta_K}
        \mathcal{D}_{\psi_t}({\bf p},\bar{{\bf p}}_{t+1}),\\
        \psi_t(\mathbf{p})&=\sum^K_{i=1}\frac{C_i}{\eta_t}p_{i}\ln{p_i},\\
    \end{split}
    \right.
    \end{equation}
    where $\psi_t(\mathbf{p})$ is the weighted negative entropy regularizer \cite{Bubeck2017Online},
    $C_i>0$ is the weight
    and $\eta_t>0$ is a time-variant learning rate.
    $C_i$ satisfies that $\max_tc^{(j)}_{t,i} \leq C_i$ for all $j\in[M]$.
    The server does not broadcast ${\bf p}_{t+1}$.

\subsubsection{Updating hypotheses}

    For each $j\in[M]$ and $i\in[K]$,
    let
    $
        \nabla^{(j)}_{t,i}=\nabla_{{\bf w}^{(j)}_{t,i}}
        \ell\left(\left\langle{\bf w}^{(j)}_{t,i},\phi_i({\bf x}^{(j)}_t)\right\rangle,y^{(j)}_t\right)$.
    Since $\nabla^{(j)}_{t,i},i\notin O^{(j)}_t$ are unknown,
    it is necessary to construct an estimator of the gradient, denoted by
    $$
        \tilde{\nabla}^{(j)}_{t,i} =\frac{\nabla^{(j)}_{t,i}}{\mathbb{P}[i\in O^{(j)}_{t}]}
        \cdot\mathbb{I}_{i\in O^{(j)}_{t}}
    $$
    for all $j\in[M], i\in[K]$.
    Clients send $\{\nabla^{(j)}_{t,i}, i\in O^{(j)}_t\}, j\in[M]$ to server.
    Then server aggregates $\{\tilde{\nabla}^{(j)}_{t,i},i\in[K]\}$, $j\in[M]$
    and updates the hypotheses following
    \eqref{eq:ICML2024:FOMD-No-LU:second_step}-\eqref{eq:ICML2024:FOMD-No-LU:fourth_step}.
    For each $i\in[K]$,
    let $\Omega=\mathcal{F}_i$ and $\tilde{g}^{(j)}_t=\tilde{\nabla}^{(j)}_{t,i}$.
    Server executes \eqref{eq:ICML2024:updating_hypothesis}.
    \begin{equation}
    \label{eq:ICML2024:updating_hypothesis}
    \left\{
    \begin{split}
        \bar{\nabla}_{t,i}=&\frac{1}{M}\sum^M_{j=1}\tilde{\nabla}^{(j)}_{t,i},\\
        \nabla_{\bar{{\bf w}}_{t+1,i}}\psi_{t,i}(\bar{{\bf w}}_{t+1,i})
        =&\nabla_{{\bf w}_{t,i}}\psi_{t,i}({\bf w}_{t,i})-\bar{\nabla}_{t,i},\\
        {\bf w}_{t+1,i}=&\mathop{\arg\min}_{{\bf w}\in\mathcal{F}_i}
        \mathcal{D}_{\psi_{t,i}}({\bf w},\bar{{\bf w}}_{t+1,i}),\\
        \psi_{t,i}({\bf w})=&\frac{1}{2\lambda_{t,i}}\cdot\Vert{\bf w}\Vert^2_{\mathcal{F}_i},
    \end{split}
    \right.
    \end{equation}
    where $\psi_{t,i}({\bf w})
    =\frac{1}{2\lambda_{t,i}}\Vert{\bf w}\Vert^2_2$ is the Euclidean regularizer
    and $\lambda_{t,i}$ is a time-variant learning rate.

    We name this algorithm FOMD-OMS (FOMD-No-LU for OMS-DecD)
    and show it in Algorithm \ref{alg:ICML2024:FOMD-DOKS}.

    \begin{algorithm}[!t]
        \caption{\small{FOMD-OMS}~$(R=T)$}
        \footnotesize
        \label{alg:ICML2024:FOMD-DOKS}
        \begin{algorithmic}[1]
        \REQUIRE{$T$, $J$, $\eta_1$, $\{U_i,\lambda_{1,i},i\in[K]\}$}
        \ENSURE{$f^{(j)}_{1,i}=0$, $p_{1,i}$, $i\in[K]$, $j\in[M]$}
        \FOR{$t=1,2,\ldots,T$}
            \FOR{$j=1,\ldots,M$}
                \STATE Server samples $O^{(j)}_t$ following \eqref{eq:ICML2024:kernel_selection}
                \STATE Server broadcasts $f^{(j)}_{t,i},i\in O^{(j)}_t$ to the $j$-th client
            \ENDFOR
                \FOR{$j=1,\ldots,M$ in parallel}
                    \STATE The client outputs $f^{(j)}_{t,A_{t,1}}({\bf x}^{(j)}_t)$
                    \STATE The client computes and transmits $\{\nabla^{(j)}_{t,i},c^{(j)}_{t,i}\}_{i\in O^{(j)}_t}$
                \ENDFOR
                \STATE Server computes ${\bf p}_{t+1}$ following \eqref{eq:ICML2024:updating_probability}
                \STATE Server computes ${\bf w}_{t+1,i},i\in[K]$
                        following \eqref{eq:ICML2024:updating_hypothesis}
        \ENDFOR
        \end{algorithmic}
    \end{algorithm}

\subsection{Regret bounds}

    To obtain high-probability regret bounds that adapt to the complexity of individual hypothesis space,
    we establish a new Bernstein's inequality for martingale.
    \begin{lemma}
    \label{lemma:AISTATS2020:improved:Bernstein_ineq_for martingales}
        Let $X_1,\ldots,X_n$ be a bounded martingale difference sequence w.r.t. the filtration
        $\mathcal{H}=(\mathcal{H}_k)_{1\leq k\leq n}$ and with $\vert X_k\vert\leq a$.
        Let $Z_t=\sum^t_{k=1}X_{k}$ be the associated martingale.
        Denote the sum of the conditional variances by
        $
            \Sigma^2_n=\sum^n_{k=1}\mathbb{E}\left[X^2_k\vert\mathcal{H}_{k-1}\right]\leq v,
        $
        where $v\in[0,B]$ is a random variable and $B\geq 2$ is a constant.
        Then for any constant $a>0$,
        with probability at least $1-2\lceil\log{B}\rceil\delta$,
        $$
            \max_{t=1,\ldots,n}Z_t < \frac{2a}{3}\ln\frac{1}{\delta}+\sqrt{\frac{2}{B}\ln\frac{1}{\delta}}+2\sqrt{v\ln\frac{1}{\delta}}.
        $$
    \end{lemma}
    Note that $v$ is a random variable in Lemma \ref{lemma:AISTATS2020:improved:Bernstein_ineq_for martingales},
    while it is a constant in standard Bernstein's inequality for martingale
    (see Lemma A.8 \cite{Cesa-Bianchi2006Prediction}).
    Lemma \ref{lemma:AISTATS2020:improved:Bernstein_ineq_for martingales}
    is derived from the standard Bernstein's inequality along with the well-known peeling technique
    \cite{Bartlett2005Local}.

    \begin{assumption}
    \label{ass:ICML24:assumption}
        For each $i\in[K]$,
        there is a constant $b_i$ such that $\Vert\phi_i({\bf x})\Vert_2\leq b_i$
        where $\phi_i(\cdot)$ is defined in \eqref{eq:ICML2024:definition:hypothesis_space}.
    \end{assumption}

    \begin{lemma}
    \label{lemma:ICML24:lipschitz_assumption}
        Under Assumption \ref{ass:ICML24:assumption},
        for each $i\in[K]$, there are two constants $C_i>0, G_i>0$ that depend on $U_i$ or $b_i$
        such that
        $\max_{t,j}c^{(j)}_{t,i}\leq C_i$ and $\max_{t,j}\Vert\nabla^{(j)}_{t,i}\Vert_2\leq G_i$.
    \end{lemma}

    \begin{theorem}
    \label{thm:ICML2024:regret_bound:FOMD_DOKS}
        Let $\ell(\cdot,\cdot)$ be convex.
        Under Assumption \ref{ass:ICML24:assumption},
        denote by $A_m=\mathrm{argmin}_{i\in[K]}C_i$.
        Let ${\bf p}_1$ satisfy
        $$
            p_{1,i}=\frac{1-\frac{\sqrt{K}}{\sqrt{T}}}{\vert A_m\vert}+\frac{1}{\sqrt{KT}}, i\in A_m,
            p_{1,i}=\frac{1}{\sqrt{KT}},i\neq A_m.
        $$
        Let $K\geq J\geq 2$ and
        \begin{align*}
            \forall t\in[T],~\eta_t=&\frac{\sqrt{\ln{(KT)}}}{2\sqrt{\left(1+\frac{K-J}{(J-1)M}\right)T}}\wedge\frac{J-1}{2(K-J)},\\
            \lambda_{t,i}=&\frac{U_i}{2G_i\sqrt{\left(1+\frac{K-J}{(J-1)M}\right)
            \cdot\left(\frac{(K-J)^2}{(J-1)^2}\vee t\right)}}.
        \end{align*}
        With probability at least $1-\Theta\left(M\log(T)+\log(KT/M)\right)\cdot\delta$,
        the regret of FOMD-OMS ($R=T$) satisfies:
        $\forall i\in[K]$,
        \begin{align*}
            \mathrm{Reg}_D&(\mathcal{F}_i)=
            O\left(MB_{i,1}\sqrt{\left(1+\frac{K-J}{(J-1)M}\right)T}+
            \frac{B_{i,2}(K-J)}{J-1}\ln\frac{1}{\delta}+
            B_{i,3}\sqrt{\frac{(K-J)MT}{J-1}\ln\frac{1}{\delta}}\right),
        \end{align*}
        where $B_{i,1}=U_iG_i+C_i\sqrt{\ln(KT)}$,
        $B_{i,2}=MC+U_iG_i$, $B_{i,3}=U_iG_i+\sqrt{CC_i}$ and $C=\max_{i\in[K]}C_i$.
    \end{theorem}

    Both $C_i$ and $G_i$ depend on $U_i$ or $b_i$
    (see Lemma \ref{lemma:ICML24:lipschitz_assumption}).
    Let $\mathfrak{C}_i=\Theta(U_iG_i+ C_i)$.
    Thus $\mathfrak{C}_i$ measures the complexity of $\mathcal{F}_i$.
    Then our regret bound adapts to $\sqrt{\mathfrak{C}\mathfrak{C}_i}$
    where $\mathfrak{C}=\max_{i\in[K]}\mathfrak{C}_i$,
    while previous regret bounds depend on $\mathfrak{C}$ \cite{Ghari2022Personalized,Hong2022Communication}
    that is, they can not adapt to the complexity of individual hypothesis space.
    If $\mathfrak{C}_{i^\ast}\ll \mathfrak{C}$,
    then our regret bound is much better.

    The regret bound in Theorem \ref{thm:ICML2024:regret_bound:FOMD_DOKS}
    is also called multi-scale regret bound \cite{Bubeck2017Online}.
    However, previous regret analysis can not yield a high-probability multi-scale bound.
    The reason is the lack of the new Bernstein's inequality for martingale
    (Lemma \ref{lemma:AISTATS2020:improved:Bernstein_ineq_for martingales}).
    If we use the new Freedman's inequality for martingale \cite{Lee2020Bias},
    then a high-probability bound can still be obtained,
    but is worse than the bound in Theorem \ref{thm:ICML2024:regret_bound:FOMD_DOKS}
    by a factor of order $O(\mathrm{poly}(\ln{T}))$.

\subsection{Complexity Analysis}
\label{sec:ICML2024:complexity_analysis}

    For each $j\in[M]$,
    the $j$-th client makes prediction and computes gradients in time $O(\sum_{i\in O^{(j)}_t}d_i)$.
    Server samples $O^{(j)}_{t},j\in[M]$,
    aggregates gradients and updates global models.
    The per-round time complexity on server is $O(\sum^M_{j=1}\sum_{i\in O^{(j)}_t}d_i+\sum^K_{i=1}d_i+JM\log{K})$.

    \textbf{Upload}~At any round $t\in[T]$,
    the $j$-th client transmits $\tilde{c}^{(j)}_{t,i}, \tilde{\nabla}^{(j)}_{t,i}$,
    $i\in O^{(j)}_t$ and the corresponding indexes to server.
    It requires $J(\sum_{i\in O^{(j)}_t}d_i+1)$ floating-point numbers and $J$ integers.
    If we use $32$ bits to represent a float,
    and use $\log{K}$ bits to represent an integer in $[K]$.
    Each client transmits $(32J(\sum_{i\in O^{(j)}_t}d_i+1)+J\log{K})$ bits to server.

    \textbf{Download}~Server broadcasts ${\bf w}_{t,i}\in\mathbb{R}^{d_i},i\in O^{(j)}_t$
    and the corresponding indexes to clients.
    The total download cost is $(32MJ(\sum_{i\in O^{(j)}_t}d_i+1)+MJ\log{K})$ bits.

\subsection{Answers to Question \ref{que:ICML2024:question}}
\label{sec:ICML2024:answer_to_question}

    Before discussing Question \ref{que:ICML2024:question},
    we give two lower bounds on the regret.

    \begin{theorem}[Lower Bounds]
    \label{thm:NeurIPS24:lower_bound}
        Assuming that $5\leq K\leq \min\{d,T\}$. For each $i\in[K]$,
        let $\mathcal{F}_i=\{f_i({\bf x})={\bf e}^\top_i{\bf x}\}$
        and $\mathcal{D}_i=[\min_{{\bf x}\in\mathcal{X}}f_i({\bf x}),\max_{{\bf x}\in\mathcal{X}}f_i({\bf x})]$,
        where ${\bf e}_i$ is the standard basis vector in $\mathbb{R}^d$.
        Denote by $\sup$ the supremum over all examples.

        (i)
        There are no computational constraints on clients.
        Let $\ell(v,y)=\vert v-y\vert$.
        The regret of any algorithm for OMS-DecD satisfies:
        $\lim_{T\rightarrow\infty}\sup\max_{i\in[K]}\mathrm{Reg}_{D}(\mathcal{F}_i)\geq 0.25M\sqrt{T\ln{K}}$;

        (ii) The per-round time complexity on each client is limited to $O(J)$.
        Let $\ell(v,y)=1- v\cdot y$.
        The regret of any, possibly randomized, noncooperative algorithm with outputs in $\cup_{i\in[K]}\mathcal{D}_i$
        satisfies: with probability at least $1-\delta$,
        $\sup\mathbb{E}[\max_{i\in[K]}\mathrm{Reg}_{D}(\mathcal{F}_i)]\geq 0.1M\sqrt{KTJ^{-1}}
        +M\sqrt{0.5T\ln{(M/\delta)}}$,
        where the expectation is taken over the randomization of algorithm.
    \end{theorem}

    The assumption that
    the outputs of any noncooperative algorithm belong to $\cup_{i=\in[K]}\mathcal{D}_i$
    is natural,
    and can be removed in the case of $J=1$.
    Next we define a noncooperative algorithm, NCO-OMS.
    \begin{definition}[NCO-OMS]
    \label{def:ICML2024:trivial_approach}
        NCO-OMS independently samples $O^{(j)}_t$ following \eqref{eq:ICML2024:kernel_selection}
        and executes
        \begin{align*}
            \forall j\in[M],\quad
        \nabla_{\bar{{\bf p}}_{t+1}}\psi_t(\bar{{\bf p}}_{t+1})
        =&\nabla_{{\bf p}^{(j)}_t}\psi_t\left({\bf p}^{(j)}_t\right)-\tilde{{\bf c}}^{(j)}_t,\quad\qquad
        {\bf p}^{(j)}_{t+1}=\mathop{\arg\min}_{{\bf p}\in\Delta_K}
        \mathcal{D}_{\psi_t}({\bf p},\bar{{\bf p}}_{t+1}).\\
        \nabla_{\bar{{\bf w}}_{t+1,i}}\psi_{t,i}(\bar{{\bf w}}_{t+1,i})
        =&\nabla_{{\bf w}^{(j)}_{t,i}}\psi_{t,i}\left({\bf w}^{(j)}_{t,i}\right)-\tilde{\nabla}^{(j)}_{t,i},\quad
        {\bf w}^{(j)}_{t+1,i}=\mathop{\arg\min}_{{\bf w}\in\mathcal{F}_i}
        \mathcal{D}_{\psi_{t,i}}({\bf w},\bar{{\bf w}}_{t+1,i}),
        \end{align*}
        where the definitions of $\tilde{{\bf c}}^{(j)}_t$ and $\tilde{\nabla}^{(j)}_{t,i}$ follow FOMD-OMS.
    \end{definition}

    \begin{algorithm}[!t]
        \caption{\small{NCO-OMS}}
        \footnotesize
        \label{alg:ICML2024:NCO-OMS}
        \begin{algorithmic}[1]
        \REQUIRE{$T$, $J$, $\eta_1$, $\{U_i,\lambda_{1,i},i\in[K]\}$}
        \ENSURE{$f^{(j)}_{1,i}=0$, $p_{1,i}$, $i\in[K]$, $j\in[M]$}
        \FOR{$t=1,2,\ldots,T$}
            \FOR{$j=1,\ldots,M$}
                \STATE The client samples $O^{(j)}_t$ following \eqref{eq:ICML2024:kernel_selection}
                \STATE The client outputs $f^{(j)}_{t,A_{t,1}}({\bf x}^{(j)}_t)$
                \STATE The client computes $f^{(j)}_{t,A_{t,a}}({\bf x}^{(j)}_t)$ for all $a=2,\ldots,J$
                \STATE The client computes $\tilde{\nabla}^{(j)}_{t,i}$ and $\tilde{c}^{(j)}_{t,i}$ for all $i\in O^{(j)}_{t}$
                \STATE The client computes ${\bf p}^{(j)}_{t+1}$ and ${\bf w}^{(j)}_{t+1,i},i\in[K]$
                        following Definition \ref{def:ICML2024:trivial_approach}
            \ENDFOR
        \ENDFOR
        \end{algorithmic}
    \end{algorithm}

    The pseudo-code of NCO-OMS is shown in Algorithm \ref{alg:ICML2024:NCO-OMS}.
    It is easy to prove the regret of NCO-OMS satisfies:
    with probability at least $1-\Theta\left(M\log(KT)\right)\cdot\delta$,
    \begin{align*}
        \forall i\in[K],~\mathrm{Reg}_D&(\mathcal{F}_i)=
        O\left(M \left(B_{i,1}\sqrt{\left(1+g_{K,J}\right)T}+
        B_{i,2}g_{K,J}\ln\frac{1}{\delta}+
        B_{i,3}\sqrt{g_{K,J}T\ln\frac{1}{\delta}}\right)\right),
    \end{align*}
    where $B_{i,1}=U_iG_i+C_i\sqrt{\ln(KT)}$,
    $B_{i,2}=C+U_iG_i$ and $B_{i,3}=U_iG_i+\sqrt{CC_i}$.
    We leave the pseudo-code of NCO-OMS and the corresponding regret analysis in appendix.

    Next we discuss Question \ref{que:ICML2024:question} by considering two cases.

    \textbf{Case 1}: There are no computational constraints on clients.
    Collaboration is unnecessary.

    Let $J=\Theta(K)$ in FOMD-OMS and NCO-OMS.
    By Theorem \ref{thm:ICML2024:regret_bound:FOMD_DOKS},
    both FOMD-OMS and NCO-OMS enjoy a $O(MU_iG_i\sqrt{T}+MC_i\sqrt{T\ln(KT)})$ regret.
    By Theorem \ref{thm:NeurIPS24:lower_bound},
    FOMD-OMS and NCO-OMS are nearly optimal in terms of the dependence on $M$ and $T$.
    Thus collaboration is unnecessary.

    \textbf{Case 2}: The per-round time complexity on each client is limited to $o(K)$.
    Collaboration is necessary.

    Let $J=o(K)$ in FOMD-OMS and Theorem \ref{thm:NeurIPS24:lower_bound}.
    By Theorem \ref{thm:ICML2024:regret_bound:FOMD_DOKS},
    FOMD-OMS enjoys a $O(MB_{i,1}\sqrt{T}+B_{i,3}\sqrt{MKTJ^{-1}\ln{\delta^{-1}}})$ regret,
    which is smaller than the lower bound on the regret of
    any noncooperative algorithm (see Theorem \ref{thm:NeurIPS24:lower_bound}).
    Thus collaboration is necessary.

\section{OMS-DecD with Communication Constraint}

    Let $R<T$.
    The clients communicate with server every $N$ rounds.
    For any $r\in[R]$,
    the clients transmit $\{\frac{1}{N}\sum_{t\in T_r}\nabla^{(j)}_{t,i},
    \frac{1}{N}\sum_{t\in T_r}c^{(j)}_{t,i}\}_{i\in O^{(j)}_t}$ to server
    at the last round in $T_r$.
    Then the server updates sampling probabilities and hypotheses.
    We give the pseudo-code Algorithm \ref{alg:ICML2024:DOKS-FedAvg-2}.

    \begin{algorithm}[!t]
        \caption{\small{FOMD-OMS}~$(R<T)$}
        \footnotesize
        \label{alg:ICML2024:DOKS-FedAvg-2}
        \begin{algorithmic}[1]
        \REQUIRE{$U$, $T$, $R$, $J$.}
        \ENSURE{$f^{(j)}_{1,i}=0$, $p_{1,i}$, $i\in[K]$, $j\in[M]$}
        \FOR{$r=1,2,\ldots,R$}
            \FOR{$t\in T_r$}
                \IF{$t==(r-1)N+1$}
                    \FOR{$j=1,\ldots,M$}
                        \STATE Server samples $O^{(j)}_t$ following \eqref{eq:ICML2024:kernel_selection}
                        \STATE Server transmits $f^{(j)}_{t,i},i\in O^{(j)}_t$ to the $j$-th client
                    \ENDFOR
                \ENDIF
                \FOR{$j=1,\ldots,M$ in parallel}
                    \STATE Output $f^{(j)}_{t,A_{t,1}}({\bf x}^{(j)}_t)$
                    \FOR{$i\in O^{(j)}_t$}
                        \STATE Computing $\nabla^{(j)}_{t,i}$ and $c^{(j)}_{t,i}$
                    \ENDFOR
                    \IF{$t==rN$}
                        \STATE Communicate to server:
                        $\{\frac{1}{N}\sum_{t\in T_r}\nabla^{(j)}_{t,i},
                        \frac{1}{N}\sum_{t\in T_r}c^{(j)}_{t,i}\}_{i\in O^{(j)}_t}$
                    \ENDIF
                \ENDFOR
            \IF{$t==rN$}
                \STATE Server computes ${\bf p}_{t+1}$ following \eqref{eq:ICML2024:updating_probability}
                \STATE Server computes ${\bf w}_{t+1,i},i\in[K]$
                        following \eqref{eq:ICML2024:updating_hypothesis}
            \ENDIF
            \ENDFOR
        \ENDFOR
        \end{algorithmic}
    \end{algorithm}

    \begin{theorem}
    \label{thm:ICML2024:regret_bound:FOMD_DOKS_R}
        For any $r\in[R]$,
        let ${\bf p}_1$, $\eta_r$ and $\lambda_{r,i}$
        follow Theorem \ref{thm:ICML2024:regret_bound:FOMD_DOKS},
        in which we replace $T$ with $R$.
        Under the condition of
        Theorem \ref{thm:ICML2024:regret_bound:FOMD_DOKS},
        with probability at least $1-\Theta\left(\frac{T}{R}M\log(R)+\frac{T}{R}\log(KR/M)\right)\cdot\delta$,
        the regret of FOMD-OMS $(R<T)$ satisfies
        \begin{align*}
            &\mathrm{Reg}_D(\mathcal{F})=
            O\left(
            MB_{i,1}\sqrt{\left(1+\frac{K-J}{(J-1)M}\right)}\cdot\frac{T}{\sqrt{R}}+
            \frac{T}{R}\cdot\frac{B_{i,2}M(K-J)}{J-1}\ln\frac{1}{\delta}+\frac{B_{i,3}T}{\sqrt{R}}
            \sqrt{\frac{M(K-J)}{J-1}\ln\frac{1}{\delta}}\right).
        \end{align*}
    \end{theorem}
    The regret bound depends on $O(\frac{1}{\sqrt{R}})$.
    Thus FOMD-OMS explicitly balances the prediction performance and the communication cost.

\section{Application to Distributed OMKL}

   For each $i\in[K]$, let $\mathcal{F}_i$ be a RKHS.
   FOMD-OMS ($R\leq T$) can solve distributed OMKL \cite{Ghari2022Personalized}.
   \begin{theorem}
   \label{thm:ICML2024:OKS}
        Let $\{\mathcal{F}_i\}^K_{i=1}$ be RHKSs.
        With probability at least $1-\Theta\left(TM\log(R)+T\log(KR/M)\right)\cdot\delta$,
        FOMD-OMS satisfies, $\forall i\in[K]$,
        \begin{align*}
            &\mathrm{Reg}_D(\mathcal{F}_i)=
            \tilde{O}\left(
            MB_{i,1}\sqrt{1+\frac{K-J}{(J-1)M}}\cdot\frac{T}{\sqrt{R}}+
            \frac{B_{i,2}M(K-J)}{R(J-1)/T}+\frac{B_{i,3}T}{\sqrt{R}}
            \sqrt{\frac{M(K-J)}{J-1}}+\frac{U_iG_iMT}{\sqrt{D}}\right),
        \end{align*}
        where $\tilde{O}(\cdot)$ omits $O(\mathrm{poly}(\ln\frac{1}{\delta}))$
        and $D=d_i$ follows \eqref{eq:ICML2024:definition:hypothesis_space}.
   \end{theorem}

   We defer the algorithm in appendix.
   Let $R=T$ and $J=2$.
   We compare FOMD-OMS with eM-KOFL \cite{Hong2022Communication} and POF-MKL \cite{Ghari2022Personalized}.
   Table \ref{tab:ICML2024:comparison_with_FOMKL_algorithms} gives the results.

   We observe that FOMD-OMS significantly improves the computational complexity of eM-KOFL and POF-MKL
   (by a factor of $O(K)$).
   The per-round time complexity of the two algorithms is $O(DK)$.
   Recalling the answer to Question \ref{que:ICML2024:question}
   (see Section \ref{sec:ICML2024:answer_to_question}),
   collaboration in eM-KOFL and POF-MKL is unnecessary.

   Next we compare the regret bound of the three algorithms.
   Recalling that $\mathfrak{C}_i\leq \mathfrak{C}$.
   The regret bounds of eM-KOFL and POF-MKL
   can not adapt to the complexity of individual hypothesis space.
   (i) The regret bound of FOMD-OMS is better than that of POF-MKL
   in relation to its dependence on $M$ and $\mathfrak{C}_i$.
   (ii) In the case of $K=O\left(\frac{\mathfrak{C}}{\mathfrak{C}_i}M\cdot \ln{K}\right)$,
   the regret bound of FOMD-OMS is better than that of eM-KOFL.
   (iii) In the case of $K=\Omega\left(\frac{\mathfrak{C}}{\mathfrak{C}_i}M\cdot \ln{K}\right)$,
   the regret bound of FOMD-OMS is worse than that of eM-KOFL.
   If $K$ is sufficiently large,
   the regret bound of eM-KOFL is better than that of FOMD-OMS by a factor of $O(\sqrt{K})$.


    \begin{table*}[!h]
      \centering
      \caption{Comparison with previous algorithms.
      $D$ is the number of random features \cite{Rahimi2007Random}.
      $\tilde{O}(\cdot)$ hides $O(\mathrm{poly}(\ln\frac{T}{\delta}))$.
      Time (s) is the per-round time complexity on client.}
      \begin{tabular}{l|r|r|r}
      \toprule
        {Algorithm}       & {Regret bound}  & Time (s) & Download (bits)\\
      \toprule
        eM-KOFL & $\tilde{O}\left(\mathfrak{C}M\sqrt{T\ln{K}}
        +\frac{\mathfrak{C}_iMT}{\sqrt{D}}\right)$ & $O(DK)$ &$O(DM\log{K})$\\
        POF-MKL & $\tilde{O}\left(\mathfrak{C}M\sqrt{KT}
        +\frac{\mathfrak{C}_iMT}{\sqrt{D}}\right)$ & $O(DK)$ &$O(DMK)$\\
        {\color{blue}FOMD-OMS} & {\color{blue}$\tilde{O}\left(\mathfrak{C}_iM\sqrt{T\ln{K}}
        +\sqrt{\mathfrak{C}\mathfrak{C}_iMKT}
        +\frac{\mathfrak{C}_iMT}{\sqrt{D}}\right)$} & {\color{blue}$O(D)$}
        & {\color{blue}$O(DM\log{K})$}\\
      \bottomrule
      \end{tabular}
      \label{tab:ICML2024:comparison_with_FOMKL_algorithms}
    \end{table*}

\section{Experiments}

    In this section, we aim to verify the following three goals
    which are our main results.
    \begin{enumerate}[\textbf{G}1]
      \item Collaboration is unnecessary if we allow the computational cost on each client to be $O(K)$.\\
            For FOMD-OMS with $R=T$, we set $J=K$.
            In this case,
            the per-round running time on each client is $O(K)$.
            We aim to verify that FOMD-OMS enjoys similar prediction performance
            with the noncooperative algorithm,
            NCO-OMS (see Definition \ref{def:ICML2024:trivial_approach}).
      \item Collaboration is necessary if we limit the computational cost on each client to $o(K)$.\\
            For FOMD-OMS with $R=T$, we set $J=2$.
            In this case,
            the per-round running time on each client is $O(1)$.
            We aim to verify that FOMD-OMS enjoys better prediction performance than NCO-OMS.
      \item FOMD-OMS $(R=T)$ improves the regret bounds of algorithms for distributed OMKL.\\
            FOMD-OMS $(R=T)$ with $J=2$ enjoys similar prediction performance with
            eM-KOFL \cite{Hong2022Communication},
            and enjoys better prediction performance than POF-MKL \cite{Ghari2022Personalized}
            at a smaller computational cost on each client.

            Although there are more baseline algorithms,
            such as vM-KOFL \cite{Hong2022Communication},
            pM-KOFL \cite{Hong2022Communication}
            and OFSKL \cite{Ghari2022Personalized},
            we do not compare with the three algorithms
            since they do not perform as well as eM-KOFL and POF-MKL.
    \end{enumerate}

\subsection{Experimental setting}

    We will execute three experiments and each one verifies a goal.
    For simplicity,
    we do not measure the actual communication cost and
    use serial implementation to simulate the distributed implementation.

    To verify \textbf{G}1 and \textbf{G}2,
    we use the instance of online model selection given in Example \ref{ex:ICML24:regularization_tunning}.
    The first experiment verifies \textbf{G}1.
    We construct $10$ nested hypothesis spaces (i.e., $K=10$) as follows
    \begin{align*}
        \forall i\in[10],\quad
        \mathcal{F}_i=\left\{f({\bf x})=\langle {\bf w},{\bf x}\rangle,\Vert{\bf w}\Vert_2\leq U_i\right\},
    \end{align*}
    where $U_i=\frac{i}{10}$.
    We use FOMD-OMS with $R=T$ and set $J=K$.
    Since $J=K$,
    we have $O^{(j)}_t=[K]$ and
    $\mathbb{P}\left[i\in O^{(j)}_t\right]=1$.
    The learning rates $\eta_t, \lambda_{t,i},i\in[K]$ of FOMD-OMS
    follow Theorem \ref{thm:ICML2024:regret_bound:FOMD_DOKS}.
    For NCO-OMS,
    we set $J=K$ and set the learning rate $\eta_t, \lambda_{t,i},i\in[K]$
    following Theorem \ref{thm:ICML2024:regret_bound:FOMD_DOKS} in which $M=1$, i.e.,
    \begin{align*}
        \forall t\in[T],~
        \eta_t=\frac{\sqrt{\ln{(KT)}}}{2\sqrt{T}},\quad\lambda_{t,i}
              =\frac{U_i}{2G_i\sqrt{t}}.
    \end{align*}
    We use the square loss function $\ell(f({\bf x}),y)=(f({\bf x})-y)^2$.
    For both FOMD-OMS and NCO-OMS,
    we tune $G_i=(U_i+1)\times\{1,2,4,6,8,10\}$ and set $C_i=(U_i+1)^2$.

    The second experiment verifies \textbf{G}2.
    We use FOMD-OMS with $R=T$ and set $J=2$.
    The learning rates of FOMD-OMS also follow Theorem \ref{thm:ICML2024:regret_bound:FOMD_DOKS}.
    For NCO-OMS,
    we also set $J=2$ and set the learning rate $\eta_t, \lambda_{t,i},i\in[K]$
    following Theorem \ref{thm:ICML2024:regret_bound:FOMD_DOKS} in which $M=1$, i.e.,
    \begin{align*}
        \forall t\in[T],\quad\eta_t=
        \frac{\sqrt{\ln{(KT)}}}{2\sqrt{\left(K-1\right)T}}\wedge\frac{1}{2(K-2)},\quad
        \lambda_{t,i}=\frac{U_i}{2G_i\sqrt{\left(K-1\right)
        \cdot\left((K-2)^2\vee t\right)}}.
    \end{align*}
    Similar to the first experiment, we tune $G_i=(U_i+1)\times\{1,2,4,6,8,10\}$ and set $C_i=(U_i+1)^2$.

    The third experiment verifies \textbf{G}3.
    We consider online kernel selection (as known as online multi-kernel learning)
    which is an instance of online model selection given in Example \ref{ex:ICML24:OKS}.
    We select the Gaussian kernel with $8$ different kernel widths (i.e., $K=8$),
    \begin{align*}
        \forall i\in[8],\quad
        \kappa_i({\bf x},{\bf v})=\exp\left(-\frac{\Vert {\bf x}-{\bf v}\Vert^2_2}{2\sigma^2_i}\right),
        \sigma_i=2^{i-2},
    \end{align*}
    and construct the corresponding hypothesis space $\mathcal{F}_i$ and $\mathbb{H}_i$
    following \eqref{eq:ICML24:restricted_RKHS} in which we set $U_i=U$ and $D_i=D$ for all $i\in[K]$
    and tune $U\in\{1,2,4\}$.
    Note that $U_i$ is same for all $i\in[K]$.
    We replace the initial distribution ${\bf p}_1$ in Theorem \ref{thm:ICML2024:regret_bound:FOMD_DOKS}
    with a uniform distribution $(\frac{1}{K},\ldots,\frac{1}{K})$.
    We set $D=100$ for FOMD-OMS, eM-KOFL and POF-MKL.
    $D$ is the number of random features.
    We set $J=2$ and $C=U+1$ in FOMD-OMS.
    Thus the per-round time complexity on each client is $O(D)$
    and the per-round communication cost is $O(MD+M\log{K})$.
    There are three hyper-parameters in eM-KOFL, i.e., $\eta_g$, $\eta_l$ and $\lambda$.
    $\eta_g$ is the global learning rate,
    $\eta_l$ is the local learning rate and $\lambda$ is a regularization parameter.
    There are $2M+3$ hyper-parameters in POF-MKL,
    i.e., $\eta_g$, $\eta_j,\xi_j, j\in[M]$, $m$, $\lambda$
    in which $M/m$ plays the same role with $J$ in FOMD-OMS.
    $\eta_g$ is the global learning rate,
    $\eta_j$ is the local learning rate,
    $\xi_j$ is called exploration rate and
    $\lambda$ is a regularization parameter.
    Since $J=2$ in FOMD-OMS,
    we can set $m=M/2$ for FOMD-OMS.
    Following the original paper \cite{Ghari2022Personalized},
    we set $\xi_j=1$.
    For a fair comparison,
    we change the learning rates of FOMD-OMS, eM-KOFL and POF-MKL.
    Following the parameter setting of eM-KOFL \cite{Hong2022Communication},
    we tune $\eta_g,\eta_l,\eta_j\in\{0.1,0.5,1,4,8,16\}$ and
    $\lambda\in\{0.1,0.001,0.0001\}$ for eM-KOFL and POF-MKL.
    For FOMD-OMS,
    we also tune $\eta_t,\lambda_{t,i}\in\{0.1,0.5,1,4,8,16\}$.

    For all of the three experiments,
    we set $10$ clients, i.e., $M=10$.
    We use 8 regression datasets shown in Table \ref{tab:ICML2023:datasets}
    from WEKA and UCI machine learning repository
    \footnote{https://archive.ics.uci.edu/ml/index.php},
    and rescale the target variables and features of all datasets to fit in [0,1] and [-1,1] respectively.
    For each dataset,
    we randomly divide it into $10$ subsets
    and each subset simulates the data on a client.
    We randomly permutate the instances in the datasets 10 times and report the average results.
    All algorithms are implemented with R on a Windows machine with 2.8 GHz Core(TM) i7-1165G7 CPU.

    \begin{table}[!t]
      \centering
      \caption{Basic information of datasets.}
      \label{tab:ICML2023:datasets}
      \begin{tabular}{lrr}
        \Xhline{0.8pt}
        Dataset       &Number of Instances & Number of features \\
        \hline
        bank          & 8,190    & 32  \\
        ailerons      & 13,750   & 40  \\
        calhousing    & 14,000   & 8   \\
        elevators     & 16,590   & 18  \\
        TomsHardware  & 28,170   & 96  \\
        Twitter       & 50,000   & 77  \\
        Year          & 51,630   & 90  \\
        Slice         & 53,500   & 384 \\
        \Xhline{0.8pt}
      \end{tabular}
    \end{table}

    We use the square loss function and define the mean squared error (MSE) of all algorithms,
    i.e.,
    \begin{align*}
        \mathrm{MSE}=\frac{1}{MT}\sum^M_{j=1}\sum^T_{t=1}\left(f^{(j)}_t\left({\bf x}^{(j)}_t\right)-y^{(j)}_t\right)^2.
    \end{align*}
    We record the mean of MSE over 10 random experiments, and the standard deviation of the mean of MSE.
    We also record the mean of of the total running time on each client,
    and the standard deviation of the mean of running time.

\subsection{Results of the First and the Second Experiment}

    We summary the experimental results of the first and the second experiments in
    Table \ref{tab:ICML2024:experiment_one}.

    \begin{table}[!t]
      \centering
      \setlength{\tabcolsep}{1.8mm}
      \caption{
      Comparison with the noncooperative algorithm.
      $\Delta$ is the difference of MSE between NCO-OMS and FOMD-OMS.
      $3\mathrm{E}\text{-}4=3\times 10^{-4}$ and $1\mathrm{E}\text{-}5=1\times 10^{-5}$.
      }
      \label{tab:ICML2024:experiment_one}
      \begin{tabular}{l|rrrr|rrrr}
        \Xhline{0.8pt}
        \multirow{2}{*}{Algorithm}&\multicolumn{4}{c|}{elevator}&\multicolumn{4}{c}{bank}\\
        \cline{2-9}&MSE$\times 10^2$ &$J$       &Time (s) & $\Delta$ &MSE$\times 10^2$ &$J$       &Time (s) & $\Delta$\\
        \hline
        NCO-OMS      & 0.991 $\pm$ 0.002 &  $K$  & 1.31 $\pm$ 0.10  &\multirow{2}{*}{0.0001}
                     & 2.158 $\pm$ 0.022 &  $K$  & 0.88 $\pm$ 0.05  &\multirow{2}{*}{0.001}\\
        FOMD-OMS     & \textbf{0.980} $\pm$ \textbf{0.005}  & $K$   & 0.65 $\pm$ 0.08  &
                     & \textbf{2.020} $\pm$ \textbf{0.005}  & $K$   & 0.33 $\pm$ 0.08  &\\
        \hline
        NCO-OMS      & 1.168 $\pm$ 0.005  &  $2$  & 0.58 $\pm$ 0.04  &\multirow{2}{*}{0.001}
                     & 2.321 $\pm$ 0.021  &  $2$  & 0.39 $\pm$ 0.05  &\multirow{2}{*}{0.002}\\
        FOMD-OMS     & \textbf{1.024} $\pm$ \textbf{0.002}  & $2$   & 0.14 $\pm$ 0.04  &
                     & \textbf{2.118} $\pm$ \textbf{0.003}  & $2$   & 0.08 $\pm$ 0.03  &\\
        \Xhline{0.8pt}
        \multirow{2}{*}{Algorithm}&\multicolumn{4}{c|}{TomsHardware}&\multicolumn{4}{c}{Twitter}\\
        \cline{2-9}&MSE$\times 10^2$ &$J$       &Time (s) & $\Delta$ &MSE$\times 10^2$ &$J$       &Time (s) & $\Delta$ \\
        \hline
        NCO-OMS      & 0.090 $\pm$ 0.004  & $K$  & 3.02 $\pm$ 0.29 & \multirow{2}{*}{0.0001}
                     & 0.017 $\pm$ 0.000  & $K$  & 5.11 $\pm$ 0.24 &\multirow{2}{*}{0}\\
        FOMD-OMS     & 0.083 $\pm$ 0.008  & $K$  & 1.48 $\pm$ 0.28 &
                     & 0.017 $\pm$ 0.000  & $K$  & 2.07 $\pm$ 0.07 &\\
        \hline
        NCO-OMS      & 0.150 $\pm$ 0.002  & $2$  & 1.11 $\pm$ 0.07 &\multirow{2}{*}{0.0004}
                     & 0.018 $\pm$ 0.000  & $K$  & 2.24 $\pm$ 0.25 &\multirow{2}{*}{1E-5}\\
        FOMD-OMS     & \textbf{0.107} $\pm$ \textbf{0.003}  & $2$   & 0.44 $\pm$ 0.09&
                     & \textbf{0.017} $\pm$ \textbf{0.000}  & $2$   & 0.51 $\pm$ 0.05&\\
        \Xhline{0.8pt}
        \multirow{2}{*}{Algorithm}&\multicolumn{4}{c|}{ailerons}&\multicolumn{4}{c}{calhousing}\\
        \cline{2-9}&MSE$\times 10^2$ &$J$       &Time (s)& $\Delta$ &MSE$\times 10^2$ &$J$       &Time (s) & $\Delta$\\
        \hline
        NCO-OMS      & 19.506 $\pm$ 0.033  & $K$  & 1.41 $\pm$ 0.04 &\multirow{2}{*}{0.0003}
                     & 10.166 $\pm$ 0.029  & $K$  & 1.07 $\pm$ 0.05 &\multirow{2}{*}{3E-4}\\
        FOMD-OMS     & 19.480 $\pm$ 0.046  & $K$  & 0.74 $\pm$ 0.08 &
                     & 10.136 $\pm$ 0.012  & $K$  & 0.43 $\pm$ 0.05 &\\
        \hline
        NCO-OMS      & 20.323 $\pm$ 0.036  & $2$  & 0.65 $\pm$ 0.05 &\multirow{2}{*}{0.0050}
                     & 10.372 $\pm$ 0.021  & $2$  & 0.53 $\pm$ 0.04 &\multirow{2}{*}{0.001}\\
        FOMD-OMS     & \textbf{19.820} $\pm$ \textbf{0.032}  & $2$   & 0.20 $\pm$ 0.04 &
                     & \textbf{10.227} $\pm$ \textbf{0.014}  & $2$   & 0.12 $\pm$ 0.05 &\\
        \Xhline{0.8pt}
        \multirow{2}{*}{Algorithm}&\multicolumn{4}{c|}{year }&\multicolumn{4}{c}{Slice}\\
        \cline{2-9}&MSE$\times 10^2$ &$J$       &Time (s) & $\Delta$ &MSE$\times 10^2$ &$J$       &Time (s) & $\Delta$\\
        \hline
        NCO-OMS      & 20.322 $\pm$ 0.040  & $K$  & 5.94 $\pm$ 0.25 &\multirow{2}{*}{0.0023}
                     & 13.097 $\pm$ 0.009  & $K$  & 10.40 $\pm$ 0.94 &\multirow{2}{*}{0.001}\\
        FOMD-OMS     & \textbf{20.096} $\pm$ \textbf{0.045}  & $K$   & 2.25 $\pm$ 0.20 &
                     & \textbf{12.964} $\pm$ \textbf{0.007}  & $K$   & 4.12 $\pm$ 0.18 &\\
        \hline
        NCO-OMS      & 24.334 $\pm$ 0.021  & $2$  & 2.95 $\pm$ 0.63 &\multirow{2}{*}{0.0163}
                     & 13.364 $\pm$ 0.012  & $2$  & 3.60 $\pm$ 0.23 &\multirow{2}{*}{0.003}\\
        FOMD-OMS     & \textbf{22.705} $\pm$ \textbf{0.040}  & $2$   & 0.59 $\pm$ 0.09 &
                     & \textbf{13.038} $\pm$ \textbf{0.009}  & $2$   & 1.41 $\pm$ 0.12 &\\
        \Xhline{0.8pt}
      \end{tabular}
    \end{table}

    In Table \ref{tab:ICML2024:experiment_one},
    $\Delta$ is defined as the difference of MSE between NCO-OM and FOMD-OMS.
    Thus $\Delta$ shows whether collaboration improves the prediction performance of
    the noncooperative algorithm.
    Times (s) records the total running time on all clients.

    We first consider the case $J=K$
    in which the per-round time complexity on each client is $O(K)$.
    It is obvious that the MSE of NCO-OMS is similar with that of FOMD-OMS.
    Although there are four datasets on which FOMD-OMS performs better than NCO-OMS,
    such as the \textit{elevator}, \textit{bank}, \textit{Year} and \textit{Slice} datsets,
    the improvement is very limited.
    Beside, the value of $\Delta$ is very small.
    Thus collaboration does not significantly improve the prediction performance of
    the noncooperative algorithm.
    The results verify the first goal \textbf{G}1.

    Next we consider the case $J=2$ in which the per-round time complexity on each client is $O(1)$.
    It is obvious that FOMD-OMS performs better than NCO-OMS on all datasets.
    Besides,
    the value of $\Delta$ in the case of $J=2$
    is much larger than that in the case of $J=K$,
    such as the \textit{elevators}, \textit{ailerons}, \textit{ailerons} and \textit{Year} datasets.
    Thus collaboration indeed improves the prediction performance of
    the noncooperative algorithm.
    The results verify the second goal \textbf{G}2.

    Finally we compare the running time of all algorithms.
    It is obvious that FOMD-OMS with $J=2$ runs faster than the other algorithms.
    The results coincide with our theoretical analysis.
    NCO-OMS runs slower than FOMD-OMS.
    The reason is that NCO-OMS must solve the sampling probability ${\bf p}_t$ using an additional binary search
    on each client
    (see Section \ref{sec:ICML24:OMD:negative_entropy}).
    In other words,
    NCO-OMS must execute binary search $M$ times at each round.
    FOMD-OMS only executes one binary search on server at each round.
    The improvement on the computational cost is benefit from
    decoupling model selection and prediction.

\subsection{Results of the Third Experiment}

    We summary the experimental results of the third experiment in Table \ref{tab:ICML2024:experiment_three}.

    \begin{table}[!t]
      \centering
      \caption{
      Comparison with the state-of-the-art algorithms.
      }
      \label{tab:ICML2024:experiment_three}
      \begin{tabular}{l|rrr|rrr}
        \Xhline{0.8pt}
        \multirow{2}{*}{Algorithm}&\multicolumn{3}{c|}{elevator}&\multicolumn{3}{c}{bank}\\
        \cline{2-7}  & MSE   &$J$       &Time (s)  & MSE  & $J$       &Time (s) \\
        \hline
        eM-KOFL      & \textbf{0.00292} $\pm$ \textbf{0.00013} & -  & 2.67 $\pm$ 0.05
                     & \textbf{0.01942} $\pm$ \textbf{0.00066} & -  & 1.41 $\pm$ 0.06 \\
        POF-MKL      & 0.00806 $\pm$ 0.00026 & -   & 3.12 $\pm$ 0.14
                     & 0.02292 $\pm$ 0.00036 & -   & 1.59 $\pm$ 0.13 \\
        FOMD-OMS     & \textbf{0.00318} $\pm$ \textbf{0.00021}  & $2$   & 0.52 $\pm$ 0.08
                     & \textbf{0.01917} $\pm$ \textbf{0.00110}  & $2$   & 0.27 $\pm$ 0.06 \\
        \Xhline{0.8pt}
        \multirow{2}{*}{Algorithm}&\multicolumn{3}{c|}{TomsHardware}&\multicolumn{3}{c}{Twitter}\\
        \cline{2-7}&MSE &$J$       &Time (s)  &MSE &$J$       &Time (s) \\
        \hline
        eM-KOFL      & \textbf{0.00048} $\pm$ \textbf{0.00003} &  -  & 5.88 $\pm$ 0.69
                     & \textbf{0.00007} $\pm$ \textbf{0.00000} &  -  & 9.60 $\pm$ 0.77 \\
        POF-MKL      & 0.00188 $\pm$ 0.00004  & -     & 6.60  $\pm$ 0.93
                     & 0.00020 $\pm$ 0.00001  & $2$   & 10.44 $\pm$ 0.54 \\
        FOMD-OMS     & 0.00059 $\pm$ 0.00003  & $2$   & 1.46  $\pm$ 0.12
                     & 0.00010 $\pm$ 0.00001  & $2$   & 2.23  $\pm$ 0.18 \\
        \Xhline{0.8pt}
        \multirow{2}{*}{Algorithm}&\multicolumn{3}{c|}{ailerons}&\multicolumn{3}{c}{calhousing}\\
        \cline{2-7}&MSE &$J$       &Time (s)  &MSE &$J$       &Time (s) \\
        \hline
        eM-KOFL      & \textbf{0.00370} $\pm$ \textbf{0.00011} &  -  & 2.40 $\pm$ 0.19
                     & \textbf{0.02242} $\pm$ \textbf{0.00043} &  -  & 2.28 $\pm$ 0.06 \\
        POF-MKL      & 0.01335 $\pm$ 0.00046 & -   & 2.66 $\pm$ 0.12
                     & 0.05248 $\pm$ 0.00197 & -   & 2.68 $\pm$ 0.08 \\
        FOMD-OMS     & 0.00429 $\pm$ 0.00021 & $2$   & 0.48 $\pm$ 0.04
                     & \textbf{0.02373} $\pm$ \textbf{0.00126}  & $2$   & 0.39 $\pm$ 0.07 \\
        \Xhline{0.8pt}
        \multirow{2}{*}{Algorithm}&\multicolumn{3}{c|}{year }&\multicolumn{3}{c}{Slice}\\
        \cline{2-7}&MSE &$J$       &Time (s)  &MSE &$J$       &Time (s) \\
        \hline
        eM-KOFL      & \textbf{0.01481} $\pm$ \textbf{0.00108} &  -  & 9.60  $\pm$ 0.51
                     & \textbf{0.05781} $\pm$ \textbf{0.00230} &  -  & 12.74 $\pm$ 0.95 \\
        POF-MKL      & 0.01896 $\pm$ 0.00036 & -   & 10.73 $\pm$ 0.29
                     & 0.08675 $\pm$ 0.00402 & -   & 14.22 $\pm$ 0.54 \\
        FOMD-OMS     & \textbf{0.01534} $\pm$ \textbf{0.00121}  & $2$   & 2.26 $\pm$ 0.10
                     & \textbf{0.05698} $\pm$ \textbf{0.00480}  & $2$   & 4.82 $\pm$ 0.21 \\
        \Xhline{0.8pt}
      \end{tabular}
    \end{table}

    We first compare FOMD-OMS with eM-KOFL.
    As a whole,
    the MSE of the two algorithms is similar.
    On the \textit{TomsHardware}, \textit{Twitter} and \textit{ailerons} datasets,
    eM-KOFL enjoys slightly better prediction performance than FOMD-OMS.
    However,
    the running time of eM-KOFL is much larger than that of FOMD-OMS.
    The results coincide with the theoretical observations
    that FOMD-OMS enjoys a similar regret bound with eM-KOFL at a much smaller
    computational cost on the clients.

    Next we compare FOMD-OMS with POF-MKL.
    Both the MSE and running time of FOMD-OMS
    are much smaller than that of POF-MKL.
    The results coincide with the theoretical observations
    that FOMD-OMS enjoys a smaller regret bound than POF-MKL at a much smaller
    computational cost on the clients.

    Thus the results in Table \ref{tab:ICML2024:experiment_three} verifies
    the third goal \textbf{G}3.

    Finally,
    we explain that why POF-MKL performs worse than FOMD-OMS.
    There are three reasons.
    \begin{enumerate}[(1)]
      \item POF-MKL does not use federated learning to learn a global probability distribution
            denoted by ${\bf p}_t$,
            but learns a personalized probability distribution denoted by ${\bf p}_{t,j}$ on each client.
            Thus POF-MKL converges to the best kernel function at a lower rate.
      \item POF-MKL uniformly samples two kernel functions and then learns two global hypotheses,
            while FOMD-OMS uses ${\bf p}_t$ to sample a kernel function and learns a global hypothesis.
            Thus POF-MKL can learn a better global hypothesis.
      \item On each client,
            POF-MKL executes model selection
            and combines the predictions of $K$ hypotheses using ${\bf p}_{t,j}$.
            Thus the time complexity is in $O(DK)$.
            FOMD-OMS executes model selection on server,
            and only uses the sampled hypothesis to make prediction.
            Thus the time complexity on each client is in $O(D)$.
    \end{enumerate}

\section{Conclusion}

    In this paper,
    we have studied the necessity of collaboration in OMS-DecD
    from the perspective of computational constraints.
    We demonstrate that collaboration is unnecessary when there are no computational constrains on clients,
    while it becomes necessary if the time complexity on each client is limited to $o(K)$.
    Our work clarifies the unnecessary nature of collaboration in previous algorithms
    for the first time, gives conditions under which collaboration is necessary,
    and provides inspirations for studying the problem from constraints beyond computational constrains.

\bibliographystyle{IEEEtran}
\bibliography{FOMS}

\newpage

\section{Notation table}

    For the sake of clarity,
    Table \ref{tab:ICML2024:notations} summaries the main notations appearing in the appendix.

    \begin{table}[!h]
      \centering
      \caption{Main notations in the appendix.}
      \begin{tabular}{l|l}
      \toprule
        Notations       & Descriptions\\
      \toprule
        $T$             & time horizon\\
        $[T]$           & $\{1,2,\ldots,T\}$\\
        $M$             & the number of clients\\
        $K$             & the number of candidate hypothesis spaces\\
        $J$             & the number of sampled hypotheses on each client\\
        $R$             & the rounds of communicaitons, $R\leq T$\\
        $N$             & $T/R$, the number of rounds between two continuous communications\\
        $T_r$           & $\left\{(r-1)N+1,(r-1)N+2,\ldots,rN\right\}$, $r=1,\ldots,R$, the $r$-th epoch\\
        $({\bf x}_t,y_t)$  & an example, ${\bf x}_t$ is call an instance, $y_t$ is the true output\\
        $({\bf x}^{(j)}_t,y^{(j)}_t)$  & the example received by the $j$-th client at the $t$-th round,
        $j\in[K]$\\
        $\mathcal{F}_i$ & $\left\{f={\bf w}^\top\phi_i(\cdot):
                            \phi_i(\cdot)\in \mathbb{R}^{d_i},
                            \Vert{\bf w}\Vert_{\mathcal{F}_i}\leq U_i\right\}$,
                            the $i$-th hypothesis space\\
        $U_i$           & regularization parameter, $U_i>0$\\
        $\Vert \cdot\Vert_{\mathcal{F}_i}$ & the Euclidean norm defined on $\mathcal{F}_i$\\
        $\phi_i$        & $\mathbb{R}^d\rightarrow \mathbb{R}^{d_i}$, a feature mapping\\
        $\kappa_i$      & $\mathbb{R}^d\times \mathbb{R}^d\rightarrow \mathbb{R}$,
        a positive semi-definite kernel function \\
        $\mathfrak{C}_i$ & the complexity of hypothesis space $\mathcal{F}_i$\\
        $\mathfrak{C}$   & $\max_{i\in[K]}\mathfrak{C}_i$\\
        $\ell(\cdot,\cdot)$ & convex loss function\\
        $f^{(j)}_{t,i}$ & the hypothesis of the $j$-th client on the $t$-th round \\
        $c^{(j)}_{t,i}$ & $\ell(f^{(j)}_{t,i}({\bf x}^{(j)}_t),y^{(j)}_t)$,
        the prediction loss of $c^{(j)}_{t,i}$ on $({\bf x}^{(j)}_t,y^{(j)}_t)$ \\
        $\nabla^{(j)}_{t,i}$      & $\nabla_{f^{(j)}_{t,i}}\ell(f^{(j)}_{t,i}({\bf x}^{(j)}_t),y^{(j)}_t)$,
        the gradient of $\ell(f_{t,i}({\bf x}^{(j)}_t),y^{(j)}_t)$ w.r.t. $f_{t,i}$\\
        $C_i$           & $\max_{j,t}c^{(j)}_{t,i}$,
                            an upper bound on the loss\\
        $G_i$           & $\max_{i,t}\Vert \nabla^{(j)}_{t,i}\Vert_{\mathcal{F}_i}$,
                            the Lipschitz constant\\
        ${\bf p}_t$     & a $K-1$ dimensional probability distribution on the $t$-th round\\
        $\Omega$        & convex and bounded set\\
        $l^{(j)}_t$     & $\Omega\rightarrow\mathbb{R}$, convex loss function\\
        $g^{(j)}_t$     & $\nabla_{{\bf u}^{(j)}_t} l^{(j)}_t({\bf u}^{(j)}_t)$,
        the gradient of $l^{(j)}_t({\bf u}^{(j)}_t)$ w.r.t. ${\bf u}^{(j)}_t$\\
        $\tilde{g}^{(j)}_t$     & an estimator of $g^{(j)}_t$\\
        $O^{(j)}_t$     & $\{A_{t,1},A_{t,2},\ldots,A_{t,J}\}$, the indexes of sampled hypotheses
        on the $j$-th client\\
        $\bar{g}_t$     & $\frac{1}{M}\sum^M_{j=1}\left(\frac{1}{N}\sum_{t\in T_r}\tilde{g}^{(j)}_{t}\right)$\\
        $\psi$          & $\Omega\rightarrow \mathbb{R}$, a strongly convex regularizer\\
        $\mathcal{D}_{\psi}(\cdot,\cdot)$ & the Bregman divergence defined on $\psi$  \\
        $\eta_t$        &a time-variant learning rate\\
        $\lambda_{t,i}$ &a time-variant learning rate\\
        $\mathbb{P}[A]$ &the probability that an event $A$ occurs\\
        $D$ & the number of random features  \\
      \bottomrule
      \end{tabular}
      \label{tab:ICML2024:notations}
    \end{table}

\section{Regret Analysis of NCO-OMS}
\label{sec:ICML24:regret_analysis_of_NCO-OMS}

    Following the definition of NCO-OMS and Algorithm \ref{alg:ICML2024:NCO-OMS},
    it is obvious that the regret bound of NCO-OMS on each client
    is same with Theorem \ref{thm:ICML2024:regret_bound:FOMD_DOKS} in which we set $M=1$.
    The regret bound on $M$ clients is $M$ times of that of a client.
    Thus we have Theorem \ref{thm:ICML2024:regret_bound:NCO-OMS}.
    \begin{theorem}[Regret Bound of NCO-OMS]
    \label{thm:ICML2024:regret_bound:NCO-OMS}
        Let the learning rate $\eta$, $\lambda_{t,i}$
        and the initial distribution ${\bf p}_1$ be same for each client $j\in[M]$.
        The values of $\eta$, $\lambda_{t,i}$ and ${\bf p}_1$
        follow Theorem \ref{thm:ICML2024:regret_bound:FOMD_DOKS}
        in which $M=1$.
        With probability at least $1-\Theta\left(M\log(KT)\right)\cdot\delta$,
        the regret of NCO-OMS satisfies:
        \begin{align*}
            \forall i\in[K],~\mathrm{Reg}_D&(\mathcal{F}_i)=
            O\left(M\times \left(B_{i,1}\sqrt{\left(1+\frac{K-J}{(J-1)}\right)T}+
            \frac{B_{i,2}(K-J)}{J-1}\ln\frac{1}{\delta}+
            B_{i,3}\sqrt{\frac{(K-J)T}{J-1}\ln\frac{1}{\delta}}\right)\right),
        \end{align*}
        where $B_{i,1}=U_iG_i+C_i\sqrt{\ln(KT)}$,
        $B_{i,2}=C+U_iG_i$ and $B_{i,3}=U_iG_i+\sqrt{CC_i}$ and $C=\max_iC_i$.

    \end{theorem}

\section{Proof of Theorem \ref{lemma:ICML2024:regret_OMD}}

        We first state a technical lemma.
        \begin{lemma}[\cite{Boyd2004Convex}]
        \label{lemma:ICML2024:first_order_optimization_condition}
            Assuming that $\psi(\cdot):\mathcal{X}\rightarrow\mathbb{R}$
            is a convex and differential function,
            and $\mathcal{X}$ is a convex domain.
            Let $f^\ast=\mathrm{argmin}_{f\in\mathcal{X}}\psi(f)$.
            Then it must be
            $$
                \forall g\in\mathcal{X},\quad
                \langle \nabla\psi(f^\ast),g-f^\ast\rangle \geq 0.
            $$
        \end{lemma}
    Lemma \ref{lemma:ICML2024:first_order_optimization_condition} gives the first-order optimality condition.
    \begin{proof}[Proof of Theorem \ref{lemma:ICML2024:regret_OMD}]
        The main idea is to give an lower bound and upper bound on
        $\langle \bar{g}_t,{\bf u}_{t+1}-{\bf v}\rangle$, respectively.

        We first give an upper bound.
        \begin{align*}
            \forall {\bf v}\in\Omega\quad&\langle \bar{g}_t,{\bf u}_{t+1}-{\bf v}\rangle\\
            =&\langle \nabla_{{\bf u}_t}\psi_t({\bf u}_t)
            -\nabla_{\bar{{\bf u}}_{t+1}}\psi_t(\bar{{\bf u}}_{t+1}),{\bf u}_{t+1}-{\bf v}\rangle\\
            =&\langle \nabla_{{\bf u}_t}\psi_t({\bf u}_t)
            -\nabla_{{\bf u}_{t+1}}\psi_t({\bf u}_{t+1}),{\bf u}_{t+1}-{\bf v}\rangle
            +\langle\nabla_{{\bf u}_{t+1}}\psi_t({\bf u}_{t+1})
            -\nabla_{\bar{{\bf u}}_{t+1}}\psi_t(\bar{{\bf u}}_{t+1}),{\bf u}_{t+1}-{\bf v}\rangle\\
            =&\mathcal{D}_{\psi_t}({\bf v},{\bf u}_t)
            -\mathcal{D}_{\psi_t}({\bf v},{\bf u}_{t+1})
            -\mathcal{D}_{\psi_t}({\bf u}_{t+1},{\bf u}_t)
            -\langle\nabla_{{\bf u}_{t+1}}\mathcal{D}_{\psi_t}({\bf u}_{t+1},\bar{{\bf u}}_{t+1}),
            {\bf v}-{\bf u}_{t+1}\rangle\\
            \leq&\mathcal{D}_{\psi_t}({\bf v},{\bf u}_t)
            -\mathcal{D}_{\psi_t}({\bf v},{\bf u}_{t+1})
            -\mathcal{D}_{\psi_t}({\bf u}_{t+1},{\bf u}_t).
        \end{align*}
        The last inequality comes from Lemma \ref{lemma:ICML2024:first_order_optimization_condition}.

        Next we give a lower bound.
        \begin{align*}
            \langle \bar{g}_t,{\bf u}_{t+1}-{\bf v}\rangle
            =&
            \frac{1}{M}\sum^M_{j=1}\left[\left\langle g^{(j)}_t,{\bf u}_{t+1}-{\bf v}\right\rangle
            +\left\langle \tilde{g}^{(j)}_t-g^{(j)}_t,
            {\bf u}_{t+1}-{\bf v}\right\rangle\right]\\
            =&
            \frac{1}{M}\sum^M_{j=1}\left\langle g^{(j)}_t,{\bf u}^{(j)}_{t}-{\bf v}\right\rangle
            +\underbrace{\frac{1}{M}\sum^M_{j=1}\left\langle g^{(j)}_t,{\bf u}_{t+1}-{\bf u}_{t}\right\rangle}_{\Xi_1}
            +\underbrace{\frac{1}{M}\sum^M_{j=1}\left\langle \tilde{g}^{(j)}_t-g^{(j)}_t,
            {\bf u}_{t+1}-{\bf v}\right\rangle}_{\Xi_2},
        \end{align*}
        where ${\bf u}^{(j)}_{t}={\bf u}_{t}$.

        Next we analyze $\Xi_1$ and $\Xi_2$.

        To analyze $\Xi_1$,
        we introduce an auxiliary variable ${\bf r}_{t+1}$ defined as follows
        $$
            \nabla_{{\bf r}_{t+1}}\psi_t({\bf r}_{t+1})
            =\nabla_{{\bf u}_{t}}\psi_t({\bf u}_{t})-\frac{2}{M}\sum^M_{j=1}g^{(j)}_t.
        $$
        Then we have
        \begin{align*}
            \Xi_1
            =&\frac{1}{2}\left\langle \frac{2}{M}\sum^M_{j=1}g^{(j)}_t,{\bf u}_{t+1}-{\bf u}_{t}\right\rangle\\
            =&\frac{1}{2}\left\langle \nabla_{{\bf u}_{t}}\psi_t({\bf u}_{t})-\nabla_{{\bf r}_{t+1}}\psi_t({\bf r}_{t+1}),{\bf u}_{t+1}-{\bf u}_{t}\right\rangle\\
            =&\frac{1}{2}\left(\mathcal{D}_{\psi}({\bf u}_{t+1},{\bf r}_{t+1})
            -\mathcal{D}_{\psi}({\bf u}_{t+1},{\bf u}_{t})
            -\mathcal{D}_{\psi}({\bf u}_{t},{\bf r}_{t+1})\right)\\
            \geq&-\frac{1}{2}\left(
            \mathcal{D}_{\psi}({\bf u}_{t+1},{\bf u}_{t})
            +\mathcal{D}_{\psi}({\bf u}_{t},{\bf r}_{t+1})\right).
        \end{align*}
        Before analyzing $\Xi_2$, we introduce also an auxiliary variable ${\bf q}_{t+1}$ defined as follows
        $$
            \nabla_{{\bf q}_{t+1}}\psi_t({\bf q}_{t+1})
            =\nabla_{{\bf u}_{t}}\psi_t({\bf u}_{t})
            -\frac{2}{M}\sum^M_{j=1}\left(\tilde{g}^{(j)}_t-g^{(j)}_t\right).
        $$
        Now we can analyze $\Xi_2$.
        We have
        \begin{align*}
            \Xi_2=&
            \frac{1}{2}\left\langle \frac{2}{M}\sum^M_{j=1}\left(\tilde{g}^{(j)}_t-g^{(j)}_t\right),
            {\bf u}_{t+1}-{\bf u}_t\right\rangle
            +\underbrace{\left\langle \frac{1}{M}\sum^M_{j=1}\left(\tilde{g}^{(j)}_t-g^{(j)}_t\right),
            {\bf u}_{t}-{\bf v}\right\rangle}_{\Xi_3}\\
            =&\frac{1}{2}
            \left\langle \nabla_{{\bf u}_{t}}\psi_t({\bf u}_{t})-\nabla_{{\bf q}_{t+1}}\psi_t({\bf u}_{t+1}),
            {\bf u}_{t+1}-{\bf u}_{t}\right\rangle+\Xi_3\\
            =&\frac{1}{2}\left(\mathcal{D}_{\psi}({\bf u}_{t+1},{\bf q}_{t+1})
            -\mathcal{D}_{\psi}({\bf u}_{t+1},{\bf u}_{t})
            -\mathcal{D}_{\psi}({\bf u}_{t},{\bf q}_{t+1})\right)+\Xi_3\\
            \geq&-\frac{1}{2}\left(\mathcal{D}_{\psi}({\bf u}_{t+1},{\bf u}_{t})
            +\mathcal{D}_{\psi}({\bf u}_{t},{\bf q}_{t+1})\right)+\Xi_3.
        \end{align*}
        Combining the lower bound and upper bound gives
        \begin{align*}
            \frac{1}{M}\sum^M_{j=1}\left[\left\langle g^{(j)}_t,{\bf u}^{(j)}_{t}-{\bf v}\right\rangle\right]
            \leq\mathcal{D}_{\psi_t}({\bf v},{\bf u}_t)
            -\mathcal{D}_{\psi_t}({\bf v},{\bf u}_{t+1})
            +\Xi_3+\frac{1}{2}\mathcal{D}_{\psi}({\bf u}_{t},{\bf q}_{t+1})
            +\frac{1}{2}\mathcal{D}_{\psi}({\bf u}_{t},{\bf r}_{t+1}).
        \end{align*}
        Using the convexity of $l^{(j)}_t$,
        that is, $l^{(j)}_t({\bf u}^{(j)}_{t})-l^{(j)}_t({\bf v})
        \leq\left\langle g^{(j)}_t,{\bf p}^{(j)}_{t}-{\bf v}\right\rangle$,
        we further obtain
        \begin{align*}
            \frac{1}{M}\sum^M_{j=1}\left(l^{(j)}_t({\bf u}^{(j)}_{t})-l^{(j)}_t({\bf v})\right)
            \leq\mathcal{D}_{\psi_t}({\bf v},{\bf u}_t)
            -\mathcal{D}_{\psi_t}({\bf v},{\bf u}_{t+1})
            +\frac{1}{2}\mathcal{D}_{\psi}({\bf u}_{t},{\bf q}_{t+1})
            +\frac{1}{2}\mathcal{D}_{\psi}({\bf u}_{t},{\bf r}_{t+1})+\Xi_3,
        \end{align*}
        which concludes the proof.
    \end{proof}

\section{Proof of Theorem \ref{lemma:ICML2024:regret_OMD_2}}

    \begin{proof}
    Recalling that $\mathcal{R}=\{N,2N,3N\ldots,RN\}$
    and
    $$
        T_r=\{(r-1)N+1,(r-1)N+2,\ldots,rN\},\quad r=1,\ldots,R.
    $$
    For any batch $T_r$, $r=1,\ldots,R$,
    we define a new loss function $\bar{l}^{(j)}_{rN}(\cdot)$ at the end of this batch,
    $$
        \forall j\in[M],~\forall {\bf u}\in\Omega,\quad
        \bar{l}^{(j)}_{rN}({\bf u})=\frac{1}{N}\sum_{\tau\in T_r}l^{(j)}_{\tau}({\bf u}).
    $$
    During each batch,
    our algorithmic framework does not change the decision, i.e.,
    $$
        \forall j\in[M],~t\in T_r,\quad {\bf u}^{(j)}_{t}={\bf u}^{(j)}_{(r-1)N+1}.
    $$
    Thus the regret can be decomposed as follows,
    \begin{align*}
        \frac{1}{M}\sum^T_{t=1}\sum^M_{j=1}\left(l^{(j)}_t({\bf u}^{(j)}_{t})-l^{(j)}_t({\bf v})\right)
        =&\frac{1}{M}\sum^R_{r=1}\left[\sum_{t\in T_r}
        \sum^M_{j=1}\left(l^{(j)}_t({\bf u}^{(j)}_{(r-1)N+1})-l^{(j)}_t({\bf v})\right)\right]\\
        =&\frac{N}{M}\sum^R_{r=1}\left[
        \sum^M_{j=1}\sum_{t\in T_r}\frac{1}{N}\left(l^{(j)}_t({\bf u}^{(j)}_{(r-1)N+1})-l^{(j)}_t({\bf v})\right)\right]\\
        =&\frac{N}{M}\sum^R_{r=1}
        \sum^M_{j=1}
        \left(\bar{l}^{(j)}_{rN}({\bf u}^{(j)}_{(r-1)N+1})-\bar{l}^{(j)}_{rN}({\bf v})\right).
    \end{align*}
    Now we can use FOMD-No-LU with $T=R$ to the new loss functions
    $\{\bar{l}^{(1)}_{rN},\ldots,\bar{l}^{(M)}_{rN}\}_{r=1,\ldots,R}$,
    and use Theorem \ref{lemma:ICML2024:regret_OMD} to obtain
    \begin{align*}
        \frac{1}{M}\sum^T_{t=1}\sum^M_{j=1}\left(l^{(j)}_t({\bf u}^{(j)}_{t})-l^{(j)}_t({\bf v})\right)
        &\leq N
        \cdot\left(\sum_{t\in\mathcal{R}}\left(\mathcal{D}_{\psi_t}({\bf v},{\bf u}_t)
            -\mathcal{D}_{\psi_t}({\bf v},{\bf u}_{t+1})\right)+\frac{1}{2}\sum_{t\in\mathcal{R}}\mathcal{D}_{\psi_t}({\bf u}_{t},{\bf q}_{t+1})
            +\right.\\
            &\left.\frac{1}{2}\sum_{t\in\mathcal{R}}\mathcal{D}_{\psi_t}({\bf u}_{t},{\bf r}_{t+1})+
            \frac{1}{M}\sum_{t\in\mathcal{R}}\sum^M_{j=1}
            \left\langle \tilde{g}^{(j)}_t-g^{(j)}_t,{\bf u}_{t}-{\bf v}\right\rangle\right),
    \end{align*}
    which concludes the proof.
    \end{proof}

\section{Proof of Lemma \ref{lemma:AISTATS2020:improved:Bernstein_ineq_for martingales}}

    \begin{lemma}[Bernstein's inequality for martingale]
    \label{lemma:ALT2021:BOKS:Bernstein_ineq_for martingales}
        Let $X_1,\ldots,X_n$ be a bounded martingale difference sequence w.r.t. the filtration
        $\mathcal{H}=(\mathcal{H}_k)_{1\leq k\leq n}$ and with $\vert X_k\vert\leq a$.
        Let $Z_t=\sum^t_{k=1}X_{k}$ be the associated martingale.
        Denote the sum of the conditional variances by
        $$
            \Sigma^2_n=\sum^n_{k=1}\mathbb{E}\left[X^2_k\vert\mathcal{H}_{k-1}\right]\leq v.
        $$
        Then for all constants $a,v>0$,
        with probability at least $1-\delta$,
        $$
            \max_{t=1,\ldots,n}Z_t < \frac{2}{3}a\ln\frac{1}{\delta}+\sqrt{2v\ln\frac{1}{\delta}}.
        $$
    \end{lemma}
    Note that $v$ must be a constant.
    Lemma \ref{lemma:ALT2021:BOKS:Bernstein_ineq_for martingales}
    is derived from Lemma A.8 in \cite{Cesa-Bianchi2006Prediction}.

    \begin{proof}
        Let $v\in[0,B]$ is a random variable and $B\geq 2$ is a constant.
        We use the well-known peeling technique \cite{Bartlett2005Local}.
        We divide the interval $[0,B]$ as follows
        $$
            [0,B]\subseteq\left[0,2^{-\lceil\log{B}\rceil}\right]
            \bigcup^{\lceil\log{B}\rceil}_{j=-\lceil\log{B}\rceil+1}
            \left(2^{j-1},2^j\right].
        $$
        First, we consider the case $v>2^{-\lceil\log{B}\rceil}$.
        Let
        \begin{align*}
            \epsilon=\frac{2}{3}a\ln\frac{1}{\delta}+2\sqrt{v\ln\frac{1}{\delta}}
            >\frac{2}{3}a\ln\frac{1}{\delta}+2\sqrt{2^{-1-\log{B}}\ln\frac{1}{\delta}}
            =\frac{2}{3}a\ln\frac{1}{\delta}+\sqrt{\frac{2}{B}\ln\frac{1}{\delta}}.
        \end{align*}
        We decompose the random event as follows,
        \begin{align*}
            &\mathbb{P}\left[\max_{t=1,\ldots,n}Z_t>\epsilon,
            \Sigma^2_n\leq v,v>2^{-\lceil\log{B}\rceil}\right]\\
            =&\mathbb{P}\left[\max_{t\leq n}Z_t>\epsilon,
            \Sigma^2_n\leq v,\cup^{\lceil\log{B}\rceil}_{i=-\lceil\log{B}\rceil+1}
            2^{i-1}< v\leq 2^i\right]\\
            \leq&\mathbb{P}\left[\max_{t\leq n}Z_t>\epsilon_i,
            \Sigma^2_n\leq v,\cup^{\lceil\log{B}\rceil}_{i=-\lceil\log{B}\rceil+1}
            2^{i-1}< v\leq 2^i\right]\\
            \leq&\sum^{\lceil\log{B}\rceil}_{i=-\lceil\log{B}\rceil+1}
            \mathbb{P}\left[\max_{t\leq n}Z_t>\epsilon_i,\Sigma^2_n\leq v,2^{i-1}<v\leq 2^i\right],
        \end{align*}
        where
        $\epsilon_i=\frac{2}{3}a\ln\frac{1}{\delta}
            +2\sqrt{2^{i-1}\ln\frac{1}{\delta}}$.
        For each sub-event,
        Lemma \ref{lemma:ALT2021:BOKS:Bernstein_ineq_for martingales} yields
        \begin{align*}
            \mathbb{P}\left[\max_{t\leq n}Z_t>\epsilon_i,\Sigma^2_n\leq v,2^{i-1}<v\leq 2^i\right]
            \leq \delta.
        \end{align*}
        Thus we have
        \begin{align*}
            \mathbb{P}\left[\max_{t\in[n]}Z_t>\epsilon,
            \Sigma^2_n\leq v,v>2^{-\lceil\log{B}\rceil}\right]
            \leq\sum^{\lceil\log{B}\rceil}_{i=-\lceil\log{B}\rceil+1}\delta
            \leq 2\lceil\ln{B}\rceil\delta.
        \end{align*}
        Then we consider the case
        $v\leq2^{-\lceil\log{B}\rceil}\leq\frac{1}{B}$.
        Lemma \ref{lemma:ALT2021:BOKS:Bernstein_ineq_for martingales} yields,
        with probability at least $1-\delta$,
        \begin{align*}
            \max_{t=1,\ldots,n}Z_t
            \leq& \frac{2}{3}a\ln\frac{1}{\delta}
            +\sqrt{2^{1-\lceil\log{B}\rceil}\ln\frac{1}{\delta}}
            \leq \frac{2}{3}a\ln\frac{1}{\delta}
            +\sqrt{\frac{2}{B}\ln\frac{1}{\delta}}.
        \end{align*}
        Combining the two cases,
        with probability at least $1-2\lceil\log{B}\rceil\delta$,
        \begin{align*}
            \max_{t=1,\ldots,n}Z_t\leq \frac{2a}{3}\ln\frac{1}{\delta}
            +\sqrt{\frac{2}{B}\ln\frac{1}{\delta}}+2\sqrt{v\ln\frac{1}{\delta}},
        \end{align*}
        which concludes the proof.
    \end{proof}

\section{Properties of OMD}

\subsection{OMD with the weighted negative entropy regularizer}
\label{sec:ICML24:OMD:negative_entropy}

    Let $\Omega=\Delta_K$ and $\psi_t({\bf p})=\sum^K_{i=1}\frac{C_i}{\eta_t}p_i\ln{p_i}$.
    Then we have
    $$
        \forall {\bf p}\in\mathbb{R}^K,\quad
        \nabla_{p_i}\psi_t({\bf p})=\frac{C_i}{\eta_t}\left(\ln{p_i}+1\right),\quad
        \nabla^2_{i,i}\psi_t({\bf p})=\frac{C_i}{\eta_t p_i}.
    $$
    The Bregman divergence associated with the negative entropy regularizer is
    \begin{align}
        \mathcal{D}_{\psi_t}({\bf p},{\bf q})
        =&\psi_t({\bf p})-\psi_t({{\bf q}})-\langle \nabla_{{\bf q}}\psi_t({\bf q}),{\bf p}-{\bf q}\rangle\nonumber\\
        =&\frac{1}{\eta_t}\sum^K_{i=1}C_i\left(p_i\ln\frac{p_i}{q_i} +q_i-p_i\right).
        \label{eq:ICML2024:Bregman_divergance_Tsallis}
    \end{align}
        Denote by $\bar{\bf c}_t=\frac{1}{M}\sum^M_{j=1}\tilde{{\bf c}}^{(j)}_t$.
        Recalling that the OMD is defined as follows,
        $$
            \nabla_{\bar{\mathbf{p}}_{t+1}}\psi_t(\bar{\mathbf{p}}_{t+1})
            =\nabla_{\mathbf{p}_{t}}\psi_t(\mathbf{p}_{t})-\bar{{\bf c}}_t,\quad
            \mathbf{p}_{t+1}=\mathop{\arg\min}_{\mathbf{p}\in\Delta_{K}}
            \mathcal{D}_{\psi_t}(\mathbf{p},\bar{\mathbf{p}}_{t+1}).
        $$
        Substituting into the gradient of $\psi_t$,
        the mirror updating can be simplified.
        \begin{align*}
        \forall i\in[K],\quad\bar{p}_{t+1,i} =p_{t,i}\cdot\exp\left(-\frac{\eta_t\bar{c}_{t,i}}{C_i}\right).
        \end{align*}
        Now we use the Lagrangian multiplier method to solve the projection associated with Bregman divergence.
        $$
            L({\bf p},\lambda)=\frac{1}{\eta_t}\sum^K_{i=1}C_i\left(p_i\ln\frac{p_i}{\bar{p}_{t+1,i}} +\bar{p}_{t+1,i}-p_i\right)
            +\lambda\left(\sum^K_{i=1}p_{i}-1\right)-\sum^K_{i=1}\beta_ip_i.
        $$
        The KKT conditions are
        \begin{align*}
            &\frac{\partial L}{\partial p_i}
            =C_i\frac{\ln{p_i}+1-\ln{\bar{p}_{t+1,i}}-1}{\eta_t}+\lambda=0,\\
            &\frac{\partial L}{\partial\lambda}=\left(\sum^K_{i=1}p_{i}-1\right)=0,\\
            &\beta_ip_i=0.
        \end{align*}
        Let ${\bf p}_{t+1}$, $\lambda^\ast$ and $\{\beta^\ast_i\}^K_{i=1}$ be the optimal solution.
        \begin{align*}
            p_{t+1,i}=&\bar{p}_{t+1,i}\cdot\exp\left(-\frac{\eta_t\lambda^\ast}{C_i}\right),\\ \sum^K_{i=1}\bar{p}_{t+1,i}\cdot\exp\left(-\frac{\eta_t\lambda^\ast}{C_i}\right)
            =&\sum^K_{i=1}p_{t,i}\cdot\exp\left(-\frac{\eta_t(\lambda^\ast+\bar{c}_{t,i})}{C_i}\right)=1,\quad
            \beta^\ast_i=0,i\in[K].
        \end{align*}
        Then we can obtain the solution ${\bf p}_{t+1}$, i.e.,
        \begin{equation}
        \label{eq:ICML2024:probability_updating}
            \forall i\in[K],\quad p_{t+1,i}
            =p_{t,i}\cdot\exp\left(-\frac{\eta_t(\lambda^\ast+\bar{c}_{t,i})}{C_i}\right).
        \end{equation}
        Next we prove that $\lambda^\ast$ can be found by the binary search.

        If $\lambda^\ast\geq 0$,
        then $\sum^K_{i=1}p_{t,i}\cdot\exp\left(-\frac{\eta_t(\lambda^\ast+\bar{c}_{t,i})}{C_i}\right)
        \leq \sum^K_{i=1}p_{t,i} \leq1$.\\
        If $\lambda^\ast\leq -\max_i\bar{c}_{t,i}$,
        then $\sum^K_{i=1}p_{t,i}\cdot\exp\left(-\frac{\eta_t(\lambda^\ast+\bar{c}_{t,i})}{C_i}\right)
        \geq \sum^K_{i=1}p_{t,i} \geq1$.\\
        Thus it must be $-\max_i\bar{c}_{t,i}\leq\lambda^\ast\leq 0$.
        For any $0\geq \lambda_1\geq \lambda_2 \geq -\max_i\bar{c}_{t,i}$,
        we can obtain
        $$
            \sum^K_{i=1}p_{t,i}\cdot\exp\left(-\frac{\eta_t(\lambda_1+\bar{c}_{t,i})}{C_i}\right)
            \leq \sum^K_{i=1}p_{t,i}\cdot\exp\left(-\frac{\eta_t(\lambda_2+\bar{c}_{t,i})}{C_i}\right).
        $$
        Thus $\sum^K_{i=1}p_{t,i}\cdot\exp\left(-\frac{\eta_t(\lambda^\ast+\bar{c}_{t,i})}{C_i}\right)$
        is non-increasing w.r.t. $\lambda^\ast$.\\
        We can use the binary search to find $\lambda^\ast$.

\subsection{OMD with the Euclidean regularizer}

    Let $\Omega=\mathcal{F}_i$ and $\psi_{t,i}({\bf w})=\frac{1}{2\lambda_{t,i}}\Vert{\bf w}\Vert^2_{\mathcal{F}_i}$.
    Then we have
    $$
        \forall {\bf w}\in\mathbb{R}^{d_i},\quad
        \nabla_{{\bf w}}\psi_{t,i}({\bf w})=\frac{1}{\lambda_{t,i}}{\bf w},\quad
        \nabla^2_{{\bf w}}\psi_{t,i}({\bf w})=\frac{1}{\lambda_{t,i}},\quad
        \mathcal{D}_{\psi_{t,i}}({\bf w},{\bf v})
        =\frac{1}{2\lambda_{t,i}}\Vert {\bf w}-{\bf v}\Vert^2_{\mathcal{F}_i}.
    $$
    Recalling that the OMD is defined as follows,
    $$
        \nabla_{\bar{{\bf w}}_{t+1,i}}\psi_{t,i}(\bar{{\bf w}}_{t+1,i})
        =\nabla_{{\bf w}_{t,i}}\psi_t({\bf w}_{t,i})-\bar{\nabla}_{t,i},\quad
        {\bf w}_{t+1,i}=\mathop{\arg\min}_{{\bf w}\in\mathcal{F}_i}
        \mathcal{D}_{\psi_{t,i}}({\bf w},\bar{{\bf w}}_{t+1,i}).
    $$
    The mirror updating is as follows,
    \begin{align*}
        \forall i\in[K],\quad\bar{{\bf w}}_{t+1,i} =&{\bf w}_{t,i}
        -\lambda_{t,i}\cdot \bar{\nabla}_{t,i},\\
        {\bf w}_{t+1,i}=&\min\left\{1,\frac{U_i}{\Vert\bar{{\bf w}}_{t+1,i}\Vert_{\mathcal{F}_i}}\right\}\cdot\bar{{\bf w}}_{t+1,i}.
    \end{align*}
    Thus OMD with the Euclidean regularizer is OGD \cite{Zinkevich2003Online}.

\section{Proof of Lemma \ref{lemma:ICML24:lipschitz_assumption}}

    Recalling that $c^{(j)}_{t,i}=\ell(f^{(j)}_{t,i}({\bf x}^{(j)}_t),y^{(j)}_t)$,
    in which
    \begin{align*}
        f^{(j)}_{t,i}({\bf x}^{(j)}_t)
        =\langle {\bf w}^{(j)}_{t,i},\phi_i({\bf x}^{(j)}_t)\rangle
        \leq U_ib_i.
    \end{align*}
    Since $\vert y^{(j)}_t\vert$ is uniformly bounded for all $j\in[M]$ and $t\in[T]$,
    there is a constant $C_i$ that depends on $U_i$ and $b_i$
    such that $c^{(j)}_{t,i}\leq C_i$.

    Recalling that
    $\nabla^{(j)}_{t,i}=\ell'(f^{(j)}_{t,i}({\bf x}^{(j)}_t),y^{(j)}_t)\cdot\phi_i({\bf x}^{(j)}_t)$.
    Since $\ell(f^{(j)}_{t,i}({\bf x}^{(j)}_t),y^{(j)}_t)$ can be upper bounded by $C_i$
    and $\Vert \phi_i({\bf x}^{(j)}_t)\Vert_2\leq b_i$,
    there is a constant $G_i$ that depends on $U_i$ and $b_i$
    such that $\Vert \nabla^{(j)}_{t,i}\Vert_2\leq G_i$.

\section{Proof of Theorem \ref{thm:ICML2024:regret_bound:FOMD_DOKS}}

    The regret w.r.t. any $f\in\mathcal{F}_i$ can be decomposed as follows.
    \begin{align*}
            &\sum^T_{t=1}\sum^M_{j=1}\ell\left(f^{(j)}_{t,A_{t,1}}(\mathbf{x}^{(j)}_t),y^{(j)}_t\right)
            -\sum^T_{t=1}\sum^M_{j=1}\ell\left(f(\mathbf{x}^{(j)}_t),y^{(j)}_t\right)\\
        =&\sum^T_{t=1}\sum^M_{j=1}
        \left[\ell\left(f^{(j)}_{t,A_{t,1}}(\mathbf{x}^{(j)}_t),y^{(j)}_t\right)
        -\ell\left(f^{(j)}_{t,i}(\mathbf{x}^{(j)}_t),y^{(j)}_t\right)\right]
        +\sum^T_{t=1}\sum^M_{j=1}\left[\ell\left(f^{(j)}_{t,i}(\mathbf{x}^{(j)}_t),y^{(j)}_t\right)
        -\ell\left(f(\mathbf{x}^{(j)}_t),y^{(j)}_t\right)\right]\\
        =& \underbrace{\sum^T_{t=1}\sum^M_{j=1}\left[c^{(j)}_{t,A_{t,1}}-c^{(j)}_{t,i}\right]}_{\Xi_4}+
        \underbrace{\sum^T_{t=1}\sum^M_{j=1}\left[\ell\left(f^{(j)}_{t,i}(\mathbf{x}^{(j)}_t),y^{(j)}_t\right)
        -\ell\left(f(\mathbf{x}^{(j)}_t),y^{(j)}_t\right)\right]}_{\Xi_5}.
    \end{align*}

    Next we separately give an upper bound on $\Xi_4$ and $\Xi_5$.

\subsection{Analyzing $\Xi_4$}

    We start with Lemma \ref{lemma:ICML2024:regret_OMD} and instantiate some notations.
    \begin{align*}
        &\Omega=\Delta_K,\quad {\bf v}={\bf v}\in\Delta_K,\\
        \forall t\in[T],\quad &g^{(j)}_t={\bf c}^{(j)}_t,\quad
        \tilde{g}^{(j)}_t=\tilde{{\bf c}}^{(j)}_t,\quad
        \bar{g}_t=\bar{{\bf c}}_t,\quad
        {\bf u}^{(j)}_t={\bf p}^{(j)}_t,\quad
        {\bf u}_t={{\bf p}}_t,\\
        &l^{j}_{t}({\bf u}^j_t)=\left\langle {\bf c}^{(j)}_t,{\bf p}^{(j)}_t\right\rangle,\quad
        l^{j}_{t}({\bf v})=\left\langle{\bf c}^{(j)}_t,{\bf v}\right\rangle.
    \end{align*}
    Lemma \ref{lemma:ICML2024:regret_OMD} gives
        \begin{equation}
        \label{eq:ICML2024:general_regret:LimA}
        \begin{split}
            \forall {\bf v}\in\Delta_K,\quad \frac{1}{M}&
                \sum^T_{t=1}\sum^M_{j=1}\left\langle {\bf c}^{(j)}_t,{\bf p}^{(j)}_{t}-{\bf v}\right\rangle\\
            \leq&\sum^T_{t=1}\left(\mathcal{D}_{\psi_t}({\bf v},{\bf p}_t)
            -\mathcal{D}_{\psi_t}({\bf v},{\bf p}_{t+1})\right)+
            \frac{1}{2}\sum^T_{t=1}\mathcal{D}_{\psi_t}({\bf p}_{t},{\bf q}_{t+1})
            +\frac{1}{2}\sum^T_{t=1}\mathcal{D}_{\psi_t}({\bf p}_{t},{\bf r}_{t+1})+\\
            &\frac{1}{M}\sum^T_{t=1}\sum^M_{j=1}
            \left\langle \tilde{{\bf c}}^{(j)}_t-{\bf c}^{(j)}_t,{\bf p}_{t}-{\bf v}\right\rangle.
        \end{split}
        \end{equation}
        In \eqref{eq:ICML2024:OMD_auxiliary:q},
        we redefine $\Omega=\Delta_K$ and $\psi_t({\bf p})=\sum^K_{i=1}\frac{C_i}{\eta_t}p_i\ln{p_i}$,
        and in \eqref{eq:ICML2024:OMD_auxiliary:r},
        we redefine $\Omega=\Delta_K$ and $\psi_t({\bf p})=\sum^K_{i=1}\frac{2C_i}{\eta_t}p_i\ln{p_i}$.
        Using the results in Section \ref{sec:ICML24:OMD:negative_entropy}, we can obtain
        \begin{equation}
        \label{eq:ICML2024:auxiliary_variable_mirror_updating}
        \begin{split}
            \forall i\in[K],\quad
            q_{t+1,i} =&p_{t,i}\exp\left(-\frac{\eta_t\delta_{t,i}}{C_i}\right),\quad
            \delta_{t,i}=\frac{2}{M}\sum^M_{j=1}\left(\tilde{c}^{(j)}_{t,i}-c^{(j)}_{t,i}\right),\\
            r_{t+1,i} =&p_{t,i}\exp\left(-\frac{\eta_t\hat{c}_{t,i}}{2C_i}\right),\quad
            \hat{c}_{t,i}=\frac{2}{M}\sum^M_{j=1}c^{(j)}_{t,i}.
        \end{split}
        \end{equation}
         It can be verified that $\delta_{t,i}\in[-2C_i,2\frac{K-J}{J-1}C_i]$
         and $\hat{c}_{t,i}\in[0,2C_i]$.

        Recalling the definition of learning rate $\eta_t$ in Theorem \ref{thm:ICML2024:regret_bound:FOMD_DOKS}.
        We can obtain $\frac{\eta_t\delta_{t,i}}{C_i}\geq -1$ and $\frac{\eta_t\hat{c}_{t,i}}{2C_i}\geq -1$.

        Next we use \eqref{eq:ICML2024:Bregman_divergance_Tsallis}
        and \eqref{eq:ICML2024:auxiliary_variable_mirror_updating}
        to analyze the following two Bregman divergences.
        \begin{align*}
            \sum^T_{t=1}\mathcal{D}_{\psi_t}({\bf p}_{t},{\bf r}_{t+1})
            =&\sum^T_{t=1}\frac{1}{\eta_t}\sum^K_{i=1}2C_i\cdot
            \left(p_{t,i}\ln\frac{p_{t,i}}{r_{t+1,i}} +r_{t+1,i}-p_{t,i}\right)\\
            =&\sum^T_{t=1}\frac{1}{\eta_t}\sum^K_{i=1}2C_i\cdot\left(\frac{p_{t,i}\eta_t\hat{c}_{t,i}}{2C_i} +p_{t,i}\cdot\exp\left(-\frac{\eta_t\hat{c}_{t,i}}{2C_i}\right)-p_{t,i}\right)\\
            \leq&\sum^T_{t=1}\frac{1}{\eta_t}\sum^K_{i=1}2C_i\cdot\left(\frac{p_{t,i}\eta_t\hat{c}_{t,i}}{2C_i} +p_{t,i}\cdot\left(1-\frac{\eta_t\hat{c}_{t,i}}{2C_i}
            +\left(\frac{\eta_t\hat{c}_{t,i}}{2C_i}\right)^2\right)-p_{t,i}\right)\\
            \leq&\sum^T_{t=1}\eta_t\sum^K_{i=1}\frac{p_{t,i}}{2C_i}
            \left(\frac{2}{M}\sum^M_{j=1}c^{(j)}_{t,i}\right)^2\\
            \leq&2\sum^T_{t=1}\eta_t\cdot \frac{1}{M}\sum^M_{j=1}\sum^K_{i=1}p_{t,i}c^{(j)}_{t,i},
        \end{align*}
        and
        \begin{align*}
            \sum^T_{t=1}\mathcal{D}_{\psi_t}({\bf p}_{t},{\bf q}_{t+1})
            =&\sum^T_{t=1}\frac{1}{\eta_t}\sum^K_{i=1}C_i\left(p_{t,i}\ln\frac{p_{t,i}}{q_{t+1,i}} +q_{t+1,i}-p_{t,i}\right)\\
            =&\sum^T_{t=1}\frac{1}{\eta_t}\sum^K_{i=1}C_i\left(\frac{p_{t,i}\eta_t\delta_{t,i}}{C_i} +p_{t,i}\cdot\exp\left(-\frac{\eta_t\delta_{t,i}}{C_i}\right)-p_{t,i}\right)\\
            \leq&4\sum^T_{t=1}\eta_t\sum^K_{i=1}\frac{p_{t,i}}{C_i}\left(\frac{1}{M}\sum^M_{j=1}
            \left(\tilde{c}^{(j)}_{t,i}-c^{(j)}_{t,i}\right)\right)^2,
        \end{align*}
        in where we use the fact $\exp(-x)\leq 1-x+x^2$ for all $x\geq -1$.

        Substituting the two upper bounds into \eqref{eq:ICML2024:general_regret:LimA}
        gives
        \begin{align*}
            \forall {\bf v}\in\Delta_K,\quad& \frac{1}{M}
                \underbrace{\sum^T_{t=1}\sum^M_{j=1}\left\langle {\bf c}^{(j)}_t,{\bf p}^{(j)}_{t}-{\bf v}\right\rangle}
                _{\Xi_{4,1}}\\
            \leq&\underbrace{\sum^T_{t=1}\left(\mathcal{D}_{\psi_t}({\bf v},{\bf p}_t)
            -\mathcal{D}_{\psi_t}({\bf v},{\bf p}_{t+1})\right)}_{\Xi_{4,2}}+
            2\underbrace{\sum^T_{t=1}\eta\sum^K_{i=1}\frac{p_{t,i}}{C_i}\left(\frac{1}{M}\sum^M_{j=1}
            \left(\tilde{c}^{(j)}_{t,i}-c^{(j)}_{t,i}\right)\right)^2}_{\Xi_{4,3}}
            +\sum^T_{t=1}\frac{\eta}{M}\sum^M_{j=1}\sum^K_{i=1}p_{t,i}c^{(j)}_{t,i}+\\
            &\underbrace{\frac{1}{M}\sum^T_{t=1}\sum^M_{j=1}
            \left\langle \tilde{{\bf c}}^{(j)}_t-{\bf c}^{(j)}_t,{\bf p}_{t}-{\bf v}\right\rangle}_{\Xi_{4,4}}.
        \end{align*}

\subsection*{Bounding $\Xi_{4,1}$}

        We define a random variable $X_t$ as follows,
        $$
            X_t=c^{(j)}_{t,A_{t,1}}-\left\langle {\bf c}^{(j)}_t,{\bf p}^{(j)}_{t}\right\rangle.
        $$
        Let $H_t=\{O^{(1)}_t,\ldots,O^{(M)}_t\}$.
        Then we have
        $\mathbb{E}[X_t\vert H_{[t-1]}]=0$ and $\vert X_t\vert\leq C$
        where $C=\max_iC_i$.
        Thus $X_{[T]}$ is a bounded martingale difference sequence w.r.t. the filtration
        $H_{[T]}$.
        The sum of condition variance
        $$
            \sum^T_{t=1}\mathbb{E}\left[\left\vert X_t\right\vert^2\vert H_{[t-1]}\right]
            \leq\sum^T_{t=1}\mathbb{E}\left[\left\vert c^{(j)}_{t,A_{t,1}}\right\vert^2\vert H_{[t-1]}\right]
            \leq C\cdot\sum^T_{t=1}\left\langle {\bf c}^{(j)}_t,{\bf p}^{(j)}_{t}\right\rangle
            \leq C^2T.
        $$
        The upper bound is a random variable.
        Lemma \ref{lemma:AISTATS2020:improved:Bernstein_ineq_for martingales} yields,
        with probability at least $1-M\log(C^2T)\delta$,
        $$
            \Xi_{4,1}\geq \sum^T_{t=1}\sum^M_{j=1}c^{(j)}_{t,A_{t,1}}
            -\sum^T_{t=1}\sum^M_{j=1}\left\langle {\bf c}^{(j)}_t,{\bf v}\right\rangle
            -\frac{2CM}{3}\ln\frac{1}{\delta}
            -2\sqrt{CM\cdot
            \sum^M_{j=1}\sum^T_{t=1}\left\langle {\bf c}^{(j)}_t,{\bf p}^{(j)}_{t}\right\rangle\cdot\ln\frac{1}{\delta}},
        $$
        where the fail probability comes from the union-of-events.

\subsection*{Bounding $\Xi_{4,2}$}

    According to \eqref{eq:ICML2024:Bregman_divergance_Tsallis},
    we have
        $$
             \Xi_{4,2}
            \leq\mathcal{D}_{\psi_1}({\bf v},{\bf p}_1)
            =\frac{1}{\eta}\sum^K_{i=1}C_i\left(v_i\ln\frac{v_i}{p_{1,i}} +p_{1,i}-v_i\right)
            \leq \frac{C_i}{\eta}\ln{\frac{1}{p_{1,i}}}
            +\frac{1}{\eta}\sum^K_{k=1}C_kp_{1,k}-\frac{C_i}{\eta}.
        $$

\subsection*{Bounding $\Xi_{4,3}$}

    We define a random variable $X_t$ as follows,
    $$
        X_t=
        \sum^K_{i=1}\frac{p_{t,i}}{C_i}\left(\frac{1}{M}\sum^M_{j=1}
            \left(\tilde{c}^{(j)}_{t,i}-c^{(j)}_{t,i}\right)\right)^2
            -\mathbb{E}_t\left[\sum^K_{i=1}\frac{p_{t,i}}{C_i}\left(\frac{1}{M}\sum^M_{j=1}
            \left(\tilde{c}^{(j)}_{t,i}-c^{(j)}_{t,i}\right)\right)^2
            \right].
    $$
    It can be verified that
    $\mathbb{E}[X_t\vert H_{[t-1]}]=0$
    and
    $\vert X_t\vert \leq \frac{K-J}{J-1}C$.
    Next we upper bound the sum of condition variance.
    \begin{align*}
        \sum^T_{t=1}\mathbb{E}_t[X^2_t]
        \leq&\sum^T_{t=1}\mathbb{E}_t\left[\left(\sum^K_{i=1}\frac{p_{t,i}}{C_i}\left(\frac{1}{M}\sum^M_{j=1}
            \left(\tilde{c}^{(j)}_{t,i}-c^{(j)}_{t,i}\right)\right)^2\right)^2\right]\\
        \leq&\sum^T_{t=1}\mathbb{E}_t\left[\sum^K_{i=1}p_{t,i}\left(\frac{1}{C_i}\left(\frac{1}{M}\sum^M_{j=1}
            \left(\tilde{c}^{(j)}_{t,i}-c^{(j)}_{t,i}\right)\right)^2\right)^2\right]\\
        \leq&\frac{(K-J)^2}{(J-1)^2}\sum^T_{t=1}\mathbb{E}_t\left[\sum^K_{i=1}p_{t,i}\left(\frac{1}{M}\sum^M_{j=1}
            \left(\tilde{c}^{(j)}_{t,i}-c^{(j)}_{t,i}\right)\right)^2\right]\\
        =&\frac{(K-J)^2}{(J-1)^2}\frac{1}{M^2}\sum^T_{t=1}\sum^K_{i=1}p_{t,i}\mathbb{E}_t\left[\sum^M_{j=1}
        \left(\tilde{c}^{(j)}_{t,i}-c^{(j)}_{t,i}\right)^2\right]\\
            \leq&\frac{(K-J)^3}{(J-1)^3M^2}C\cdot\sum^T_{t=1}\sum^M_{j=1}
            \left\langle {\bf c}^{(j)}_t,{\bf p}^{(j)}_{t}\right\rangle \leq\frac{(K-J)^3}{(J-1)^3M}C^2T,
    \end{align*}
    where we use the fact $\tilde{c}^{(j)}_{t,i}=\frac{c^{(j)}_{t,i}}{\mathbb{P}\left[i\in O^{(j)}_t\right]}
    \geq \frac{K-1}{J-1}c^{(j)}_{t,i}$.
    Lemma \ref{lemma:AISTATS2020:improved:Bernstein_ineq_for martingales} yields,
    with probability at least $1-\log(C^2K^3T/M)\delta$,
    $$
        \Xi_{4,3}\leq
        \eta\left(\frac{K-J}{(J-1)M^2}\sum^T_{t=1}\sum^M_{j=1}\left\langle {\bf c}^{(j)}_t,{\bf p}^{(j)}_{t}\right\rangle
        +\frac{2C(K-J)}{3(J-1)}\ln\frac{1}{\delta}
        +2\sqrt{\frac{(K-J)^3}{(J-1)^3M^2}C\cdot\sum^T_{t=1}\sum^M_{j=1}
        \left\langle {\bf c}^{(j)}_t,{\bf p}^{(j)}_{t}\right\rangle\cdot\ln\frac{1}{\delta}}\right).
    $$

\subsection*{Bounding $\Xi_{4,4}$}

        We define a random variable $X_t$ as follows,
        $$
            X_t=
            \left\langle \frac{1}{M}\sum^M_{j=1}
            \left(\tilde{{\bf c}}^{(j)}_t-{\bf c}^{(j)}_t\right),{\bf p}_{t}-{\bf v}\right\rangle
            =\frac{1}{M}\sum^M_{j=1}\left(\sum^K_{i=1}(p_{t,i}-v_i)
            \left(\tilde{c}^{(j)}_{t,i}-c^{(j)}_{t,i}\right)\right).
        $$
        $\{X_t\}^T_{t=1}$ is a bounded martingale difference sequence
        and $\vert X_t\vert\leq \frac{K-J}{J-1}C$.
        We further have
        \begin{align*}
            \sum^T_{t=1}\mathbb{E}_t[ X^2_t]
            =&\frac{1}{M^2}\sum^T_{t=1}\mathbb{E}_t\left[
            \sum^M_{j=1}\left(\sum^K_{i=1}(p_{t,i}-v_i)\left(\tilde{c}^{(j)}_{t,i}-c^{(j)}_{t,i}\right)
            \right)^2\right]+\\
            &\frac{1}{M^2}\sum^T_{t=1}\mathbb{E}_t\left[
            \sum_{j\neq r}
            \left(\sum^K_{i=1}(p_{t,i}-v_i)
            \left(\tilde{c}^{(j)}_{t,i}-c^{(j)}_{t,i}\right)\right)
            \left(\sum^K_{i=1}(p_{t,i}-v_i)
            \left(\tilde{c}^{(r)}_{t,i}-c^{(r)}_{t,i}\right)\right)\right]\\
            =& \frac{1}{M^2}\sum^T_{t=1}\sum^M_{j=1}\mathbb{E}_t\left[
            \left(\sum^K_{i=1}(p_{t,i}-v_i)\left(\tilde{c}^{(j)}_{t,i}-c^{(j)}_{t,i}\right)
            \right)^2\right]\\
            =& \frac{2}{M^2}\sum^T_{t=1}\sum^M_{j=1}\mathbb{E}_t\left[
            \left(\sum^K_{i=1}p_{t,i}\left(\tilde{c}^{(j)}_{t,i}-c^{(j)}_{t,i}\right)
            \right)^2\right]
            +\frac{2}{M^2}\sum^T_{t=1}\sum^M_{j=1}\mathbb{E}_t\left[
            \left(\sum^K_{i=1}v_i\left(\tilde{c}^{(j)}_{t,i}-c^{(j)}_{t,i}\right)
            \right)^2\right]\\
            \leq& \frac{2}{M^2}\sum^T_{t=1}\sum^M_{j=1}\mathbb{E}_t\left[
            \sum^K_{i=1}p_{t,i}\left(\tilde{c}^{(j)}_{t,i}-c^{(j)}_{t,i}\right)^2\right]
            +\frac{2}{M^2}\sum^T_{t=1}\sum^M_{j=1}\mathbb{E}_t\left[
            \sum^K_{i=1}v_i\left(\tilde{c}^{(j)}_{t,i}-c^{(j)}_{t,i}\right)^2\right]\\
            \leq& 2\frac{K-J}{(J-1)M^2}C\cdot\sum^T_{t=1}\sum^M_{j=1}\left\langle {\bf c}^{(j)}_t,{\bf p}^{(j)}_{t}\right\rangle
            +2\frac{K-J}{(J-1)M^2}\cdot\sum^T_{t=1}\sum^M_{j=1}
            \left\langle {\bf c}^{(j)}_t\otimes{\bf c}^{(j)}_t,{\bf v}\right\rangle
            \leq\frac{4C^2KT}{M},
        \end{align*}
        where $\left\langle {\bf c}^{(j)}_t\otimes{\bf c}^{(j)}_t,{\bf v}\right\rangle=\sum^K_{i=1}v_i(c^{(j)}_{t,i})^2$.

        With probability at least $1-\log(4C^2KT/M)\delta$,
        $$
            \Xi_{4,4}\leq \frac{2C(K-J)}{3(J-1)}\ln\frac{1}{\delta}+2\sqrt{2\frac{K-J}{(J-1)M^2}\ln\frac{1}{\delta}}\cdot\sqrt{C
            \sum^T_{t=1}\sum^M_{j=1}\left\langle {\bf c}^{(j)}_t,{\bf p}^{(j)}_{t}\right\rangle
            +\sum^T_{t=1}\sum^M_{j=1}\left\langle {\bf c}^{(j)}_t\otimes{\bf c}^{(j)}_t,{\bf v}\right\rangle}.
        $$
        For simplicity,
        we introduce some new notations
        $$
            g_{K,J}=\frac{K-J}{J-1},\quad
            \bar{L}_{T}=\sum^T_{t=1}\sum^M_{j=1}\left\langle {\bf c}^{(j)}_t,{\bf p}^{(j)}_{t}\right\rangle,\quad
            \bar{L}_{T}({\bf v})=\sum^T_{t=1}\sum^M_{j=1}\left\langle {\bf c}^{(j)}_t,{\bf v}\right\rangle,\quad
            \tilde{L}_{T}({\bf v})=\sum^T_{t=1}\sum^M_{j=1}\left\langle {\bf c}^{(j)}_t\otimes{\bf c}^{(j)}_t,{\bf v}\right\rangle.
        $$

\subsection*{Combining all}

    Combining all gives,
    with probability at least $1-\Theta(\log(CKT/M))\cdot\delta$,
    \begin{align*}
        &\bar{L}_{T}-\bar{L}_{T}({\bf v})\\
        \leq&\frac{M}{\eta}\left(C_i\ln{\frac{1}{p_{1,i}}}+\sum^K_{k=1}C_kp_{1,k}-C_i\right)
        +\eta\left(\left(1+\frac{g_{K,J}}{M}\right)\bar{L}_{T}
        +\frac{2MC}{3}g_{K,J}\ln\frac{1}{\delta}
        +2\sqrt{g^3_{K,J}C\cdot\bar{L}_{T}\cdot\ln\frac{1}{\delta}}\right)+\\
        &\frac{2MCg_{K,J}}{3}\ln\frac{1}{\delta}+2\sqrt{2g_{K,J}\ln\frac{1}{\delta}}
        \cdot\sqrt{C\bar{L}_{T}+\tilde{L}_{T}({\bf v})}.
    \end{align*}
    Rearranging terms gives
    \begin{align*}
        &\left(1-\eta\left(1+\frac{g_{K,J}}{M}\right)\right)\bar{L}_{T}
        -\left(2\eta\sqrt{g^3_{K,J}C\ln\frac{1}{\delta}}+2\sqrt{2g_{K,J}C\ln\frac{1}{\delta}}\right)\sqrt{\bar{L}_{T}}\\
        \leq&\bar{L}_{T}({\bf v})+\frac{M}{\eta}\left(C_i\ln{\frac{1}{p_{1,i}}}+\sum^K_{k=1}C_kp_{1,k}-C_i\right)
        +\frac{4MCg_{K,J}}{3}\ln\frac{1}{\delta}
        +2\sqrt{2g_{K,J}\ln\frac{1}{\delta}}
        \cdot\sqrt{\tilde{L}_{T}({\bf v})}.
    \end{align*}
    Recalling that, the solution of the following inequality
    $$
        x-a\sqrt{x}-b\leq 0, x>0, a>0, b>0,
    $$
    is $x\leq a^2+b+a\sqrt{b}$.
    Solving for $\bar{L}_{T}$ gives
    \begin{align*}
        &\bar{L}_{T}-\bar{L}_{T}({\bf v})
        \leq \frac{\left(2\eta\sqrt{g^3_{K,J}C\ln\frac{1}{\delta}}+2\sqrt{2g_{K,J}C\ln\frac{1}{\delta}}\right)^2}
        {\left(1-\eta\left(1+\frac{g_{K,J}}{M}\right)\right)^2}+
        \frac{2\eta\sqrt{g^3_{K,J}C\ln\frac{1}{\delta}}+2\sqrt{2g_{K,J}C\ln\frac{1}{\delta}}}
        {\left(1-\eta\left(1+\frac{g_{K,J}}{M}\right)\right)^{\frac{3}{2}}}\cdot\\
        &\sqrt{\bar{L}_{T}({\bf v})+\frac{M}{\eta}\left(C_i\ln{\frac{1}{p_{1,i}}}+\sum^K_{k=1}C_kp_{1,k}-C_i\right)
        +\frac{4MCg_{K,J}}{3}\ln\frac{1}{\delta}+2\sqrt{2g_{K,J}\ln\frac{1}{\delta}}\cdot\sqrt{\tilde{L}_{T}({\bf v})}}+\\
        &\frac{\eta\left(1+\frac{g_{K,J}}{M}\right)}{{1-\eta\left(1+\frac{g_{K,J}}{M}\right)}}\bar{L}_{T}({\bf v})+
        \frac{\frac{M}{\eta}\left(C_i\ln{\frac{1}{p_{1,i}}}+\sum^K_{k=1}C_kp_{1,k}-C_i\right)
        +\frac{4MCg_{K,J}}{3}\ln\frac{1}{\delta}+2\sqrt{2g_{K,J}\ln\frac{1}{\delta}}\cdot\sqrt{\tilde{L}_{T}({\bf v})}}
        {1-\eta\left(1+\frac{g_{K,J}}{M}\right)}.
    \end{align*}
    Denote by $A_m=\mathrm{argmin}_{i\in[K]}C_i$.
    Let the learning rate and initial distribution ${\bf p}_1$ satisfy
    \begin{align*}
        \eta=&\frac{\sqrt{\ln{(KT)}}}{2\sqrt{\left(1+\frac{K-J}{(J-1)M}\right)T}}\wedge\frac{J-1}{2(K-J)},\\
        p_{1,k}=&\left(1-\frac{\sqrt{K}}{\sqrt{T}}\right)\frac{1}{\vert A_m\vert}+\frac{1}{\sqrt{KT}}, k\in A_m,\quad
        p_{1,j}=\frac{1}{\sqrt{KT}},j\neq A_m.
    \end{align*}
    Then we have
    \begin{align*}
        &C_i\ln{\frac{1}{p_{1,i}}}+\sum^K_{k=1}C_kp_{1,k}-C_i\\
        \leq& C_i\ln(\sqrt{KT})
        +\frac{C\cdot(K-\vert A_m\vert)}{\sqrt{KT}}
        +\min_iC_i\cdot\vert A_m\vert
        \cdot\left(\left(1-\frac{\sqrt{K}}{\sqrt{T}}\right)\frac{1}{\vert A_m\vert}+\frac{1}{\sqrt{KT}}\right)
        -C_i\\
        \leq&C_i\ln(\sqrt{KT})+\frac{C\sqrt{K}}{\sqrt{T}}.
    \end{align*}
    We further simplify $\bar{L}_{T}-\bar{L}_{T}({\bf v})$.
    \begin{align*}
        &\bar{L}_{T}-\bar{L}_{T}({\bf v})
        \leq 64g_{K,J}C\ln\frac{1}{\delta}+
        \\
        &11\sqrt{g_{K,J}C\ln\frac{1}{\delta}}\cdot\sqrt{C_iTM+\frac{M}{\eta}\left(C_i\ln(\sqrt{KT})
        +\frac{C\sqrt{K}}{\sqrt{T}}\right)
        +\frac{4MCg_{K,J}}{3}\ln\frac{1}{\delta}+2\sqrt{2g_{K,J}\ln\frac{1}{\delta}}\cdot\sqrt{\tilde{L}_{T}({\bf v})}}+\\
        &2\eta\left(1+\frac{g_{K,J}}{M}\right)C_iT+
        \frac{\frac{M}{\eta}\left(C_i\ln(\sqrt{KT})+\frac{C\sqrt{K}}{\sqrt{T}}\right)
        +\frac{4MCg_{K,J}}{3}\ln\frac{1}{\delta}+2C_i\sqrt{2g_{K,J}\ln\frac{1}{\delta}}\cdot\sqrt{TM}}{\frac{1}{2}}\\
        \leq&\underbrace{(64+3M)g_{K,J}C\ln\frac{1}{\delta}
        +17\sqrt{Mg_{K,J}CC_iT\ln\frac{1}{\delta}}
        +\frac{4}{\sqrt{2}}C_iM\sqrt{\left(1+\frac{K-J}{(J-1)M}\right)T\ln{(KT)}}}_{\Xi_{4,5}},
    \end{align*}
    in which we omit the lower order terms such as $O(T^{\frac{1}{4}})$
    and $O(\sqrt{g_{K,J}C\ln\frac{1}{\delta}})$.

    Finally, using the upper bound on $\Xi_{4,1}$ gives,
    with probability at least $1-\Theta(M\log(CT)+\log(CKT/M))\cdot\delta$,
    \begin{align*}
        \Xi_4
        \leq&
        \bar{L}_{T}-\bar{L}_{T}({\bf v})+\frac{2CM}{3}\ln\frac{1}{\delta}
        +2\sqrt{CM\cdot
        \sum^M_{j=1}\sum^T_{t=1}\left\langle {\bf c}^{(j)}_t,{\bf p}^{(j)}_{t}\right\rangle\cdot\ln\frac{1}{\delta}}\\
        \leq&
        \bar{L}_{T}-\bar{L}_{T}({\bf v})+\frac{2CM}{3}\ln\frac{1}{\delta}
        +2\sqrt{CM\cdot\left(\bar{L}_{T}({\bf v})+\Xi_{4,5}\right)\cdot\ln\frac{1}{\delta}}\\
        \leq&(64+3M)g_{K,J}C\ln\frac{1}{\delta}
            +17\sqrt{Mg_{K,J}CC_iT\ln\frac{1}{\delta}}
            +\frac{4}{\sqrt{2}}C_iM\sqrt{\left(1+\frac{K-J}{(J-1)M}\right)T\ln{(KT)}}+\\
            &2M\sqrt{CC_iT\ln\frac{1}{\delta}},
    \end{align*}
    where we omit $O(\sqrt{CM\Xi_{4,5}\cdot\ln\frac{1}{\delta}})$ which is a lower order term.

\subsection{Analyzing $\Xi_5$}

    We also start with Lemma \ref{lemma:ICML2024:regret_OMD}.

    We just a fixed $i\in\mathcal{F}_i$.
    We instantiate some notations.
    \begin{align*}
        &\Omega=\mathcal{F}_i,\quad {\bf v}={\bf w}\in\mathcal{F}_i,\\
        \forall t\in[T],\quad &g^{(j)}_t=\nabla^{(j)}_{t,i},\quad
        \tilde{g}^{(j)}_t=\tilde{\nabla}^{(j)}_{t,i},\quad
        \bar{g}_t=\bar{\nabla}_{t,i},\quad
        {\bf u}^{(j)}_t={\bf w}^{(j)}_{t,i},\quad
        {\bf u}_t={{\bf w}}_t,\\
        &l^{j}_{t}({\bf u}^j_t)=\ell\left(f^{(j)}_{t,i}(\mathbf{x}^{(j)}_t),y^{(j)}_t\right),\quad
        l^{j}_{t}({\bf v})=\ell\left(f(\mathbf{x}^{(j)}_t),y^{(j)}_t\right).
    \end{align*}
    Lemma \ref{lemma:ICML2024:regret_OMD} gives
        \begin{align*}
            \forall {\bf w}\in\mathcal{F}_i,\quad \frac{1}{M}&\sum^T_{t=1}\sum^M_{j=1}
            \left[\ell\left(f^{(j)}_{t,i}(\mathbf{x}^{(j)}_t),y^{(j)}_t\right)
                -\ell\left(f(\mathbf{x}^{(j)}_t),y^{(j)}_t\right)\right]\\
            \leq&\sum^T_{t=1}\left(\mathcal{D}_{\psi_{t,i}}({\bf w},{\bf w}_t)
            -\mathcal{D}_{\psi_t}({\bf w},{\bf w}_{t+1})\right)+
            \frac{1}{2}\sum^T_{t=1}\mathcal{D}_{\psi_{t,i}}({\bf w}_{t},{\bf q}_{t+1})
            +\frac{1}{2}\sum^T_{t=1}\mathcal{D}_{\psi_{t,i}}({\bf w}_{t},{\bf r}_{t+1})+\\
            &\frac{1}{M}\sum^T_{t=1}\sum^M_{j=1}
            \left\langle \tilde{\nabla}^{(j)}_{t,i}-\nabla^{(j)}_{t,i},{\bf w}_{t}-{\bf w}\right\rangle,
        \end{align*}
        where the Bregman divergence is
        $$
            \mathcal{D}_{\psi_{t,i}}({\bf w},{\bf v})=\frac{1}{2\lambda_{t,i}}\Vert{\bf w}-{\bf v}\Vert^2_2.
        $$
        Besides, \eqref{eq:ICML2024:OMD_auxiliary:q} and \eqref{eq:ICML2024:OMD_auxiliary:r}
        can be instantiated as follows
        \begin{align*}
            {\bf q}_{t+1}
            =&{\bf w}_{t}
            -\lambda_{t,i}\cdot\frac{2}{M}\sum^M_{j=1}\left(\tilde{\nabla}^{(j)}_{t,i}-{\nabla}^{(j)}_{t,i}\right),\\
            {\bf r}_{t+1}
            =&{\bf w}_{t}
            -\lambda_{t,i}\cdot\frac{2}{M}\sum^M_{j=1}{\nabla}^{(j)}_{t,i}.
        \end{align*}
        Thus we have,
        \begin{align*}
            \forall {\bf w}\in\mathcal{F}_i,\quad&\frac{1}{M}\Xi_5\\
            \leq&\sum^T_{t=1}\frac{\Vert{\bf w}-{\bf w}_t\Vert^2_2-\Vert{\bf w}-{\bf w}_{t+1}\Vert^2_2}{2\lambda_{t,i}}+
            2\sum^T_{t=1}\lambda_{t,i}\left\Vert\frac{1}{M}\sum^M_{j=1}\left(\tilde{\nabla}^{(j)}_{t,i}-{\nabla}^{(j)}_{t,i}\right)\right\Vert^2_2+\\
            &2\sum^T_{t=1}\lambda_{t,i}\left\Vert\frac{1}{M}\sum^M_{j=1}{\nabla}^{(j)}_{t,i}\right\Vert^2_2+
            \frac{1}{M}\sum^T_{t=1}\sum^M_{j=1}
            \left\langle \tilde{\nabla}^{(j)}_{t,i}-\nabla^{(j)}_{t,i},{\bf w}_{t}-{\bf w}\right\rangle\\
            \leq&\frac{2U^2_i}{\lambda_{T,i}}
            +2G^2_i\sum^T_{t=1}\lambda_{t,i}
            +2\underbrace{\sum^T_{t=1}\lambda_{t,i}\left\Vert\frac{1}{M}\sum^M_{j=1}\left(\tilde{\nabla}^{(j)}_{t,i}-{\nabla}^{(j)}_{t,i}\right)\right\Vert^2_2}_{\Xi_{5,1}}
            +
            \underbrace{\frac{1}{M}\sum^T_{t=1}\sum^M_{j=1}
            \left\langle \tilde{\nabla}^{(j)}_{t,i}-\nabla^{(j)}_{t,i},{\bf w}_{t}-{\bf w}\right\rangle}_{\Xi_{5,2}}.
        \end{align*}
        Next we separately give a high-probability upper bound on $\Xi_{4,1}$ and $\Xi_{4,2}$.\\

\subsection*{Bounding $\Xi_{5,2}$}

        We define a random variable $X_t$ as follows,
        $$
            X_t=
            \left\langle \frac{1}{M}\sum^M_{j=1}
            \left(\tilde{\nabla}^{(j)}_{t,i}-\nabla^{(j)}_{t,i}\right),{\bf w}_{t}-{\bf w}\right\rangle.
        $$
        $X_{[T]}$ is a bounded martingale difference sequence w.r.t. $H_{[T]}$
        and $\vert X_t\vert\leq 2\frac{K-J}{J-1}G_iU_i$.
        We further have
        $$
            \sum^T_{t=1}\mathbb{E}_t[\vert X_t\vert^2]\leq
            \sum^T_{t=1}4U^2_i\mathbb{E}_t\left[\left\Vert\frac{1}{M}\sum^M_{j=1}
            \left(\tilde{\nabla}^{(j)}_{t,i}-\nabla^{(j)}_{t,i}\right)\right\Vert^2_2\right]
            \leq 4U^2_iG^2_i\frac{K-J}{(J-1)M}T.
        $$
        The upper bound on the sum of conditional variance is a constant.
        Lemma \ref{lemma:ALT2021:BOKS:Bernstein_ineq_for martingales}
        gives, with probability at least $1-\delta$,
        $$
            \Xi_{5,2}\leq \frac{4G_iU_i(K-J)}{3(J-1)}\ln\frac{1}{\delta}+2G_iU_i\sqrt{2\frac{K-J}{(J-1)M}T\ln\frac{1}{\delta}}.
        $$

\subsection*{Bounding $\Xi_{5,1}$}

    Recalling that
    $$
        \lambda_{t,i}=\left\{
        \begin{array}{ll}
        \frac{U_i}{2G_i\sqrt{\left(1+\frac{K-J}{(J-1)M}\right)\frac{(K-J)^2}{(J-1)^2}}}&\mathrm{if}~t\leq \frac{(K-J)^2}{(J-1)^2},\\
        \frac{U_i}{2G_i\sqrt{\left(1+\frac{K-J}{(J-1)M}\right)t}}&\mathrm{otherwise}.
        \end{array}
        \right.
    $$
    It can be found that $\lambda_{t,i}\leq \frac{(J-1)U_i}{2(K-J)G_i}$.\\
    \textbf{Case 1}: $T>\frac{(K-J)^2}{(J-1)^2}$.\\
    We decompose $\Xi_{5,1}$ as follows,
    $$
        \Xi_{5,1}=\underbrace{\sum^{\frac{(K-J)^2}{(J-1)^2}}_{t=1}\lambda_{t,i}\left\Vert\frac{1}{M}\sum^M_{j=1}\left(\tilde{\nabla}^{(j)}_{t,i}-{\nabla}^{(j)}_{t,i}\right)\right\Vert^2_2}_{\Xi_{5,1,1}}
        +\underbrace{\sum^T_{t=\frac{(K-J)^2}{(J-1)^2}+1}\lambda_{t,i}\left\Vert\frac{1}{M}\sum^M_{j=1}\left(\tilde{\nabla}^{(j)}_{t,i}-{\nabla}^{(j)}_{t,i}\right)\right\Vert^2_2}_{\Xi_{5,1,2}}.
    $$
    We separately analyze $\Xi_{5,1,1}$ and $\Xi_{5,1,2}$.
    Let
    $$
        X_t=\lambda_{t,i}\left\Vert\frac{1}{M}\sum^M_{j=1}\left(\tilde{\nabla}^{(j)}_{t,i}-{\nabla}^{(j)}_{t,i}\right)\right\Vert^2_2
        -\lambda_{t,i}\mathbb{E}_t\left[\left\Vert\frac{1}{M}\sum^M_{j=1}\left(\tilde{\nabla}^{(j)}_{t,i}-{\nabla}^{(j)}_{t,i}\right)\right\Vert^2_2\right].
    $$
    $X_{[T]}$ is a martingale difference sequence and satisfies
    $\vert X_t\vert\leq \lambda_{t,i}\cdot \frac{(K-J)^2}{(J-1)^2}G^2_i\leq \frac{(K-J)U_iG_i}{2(J-1)}$.\\
        We further have
        \begin{align*}
            &\sum^{\frac{(K-J)^2}{(J-1)^2}}_{t=1}\mathbb{E}_t[\vert X_t\vert^2]\leq
            \sum^{\frac{(K-J)^2}{(J-1)^2}}_{t=1}\mathbb{E}_t\left[\lambda^2_{t,i}\left\Vert\frac{1}{M}\sum^M_{j=1}
            \left(\tilde{\nabla}^{(j)}_{t,i}-\nabla^{(j)}_{t,i}\right)\right\Vert^4_2\right]
            \leq U^2_iG^2_i\frac{(K-J)^3}{4M(J-1)^3},\\
            &\sum^T_{t=\frac{(K-J)^2}{(J-1)^2}+1}\mathbb{E}_t[\vert X_t\vert^2]
            \leq U^2_iG^2_i\frac{K-J}{4M(J-1)}\left(T-\frac{(K-J)^2}{(J-1)^2}\right).
        \end{align*}
        With probability at least $1-2\delta$,
        \begin{align*}
            \Xi_{5,1}
            \leq&\sum^T_{t=1}\lambda_{t,i}\mathbb{E}_t\left[\left\Vert\frac{1}{M}\sum^M_{j=1}\left(\tilde{\nabla}^{(j)}_{t,i}-{\nabla}^{(j)}_{t,i}\right)\right\Vert^2_2\right]
            +\frac{2(K-J)G_iU_i}{3(J-1)}\ln\frac{1}{\delta}+G_iU_i\sqrt{2\frac{K-J}
            {(J-1)M}T\ln\frac{1}{\delta}}\\
            \leq&\frac{K-J}{(J-1)M}G^2_i\sum^T_{t=1}\lambda_{t,i}+\frac{2(K-J)G_iU_i}{3(J-1)}\ln\frac{1}{\delta}
            +G_iU_i\sqrt{2\frac{K-J}{(J-1)M}T\ln\frac{1}{\delta}}.
        \end{align*}
    Combining with all results gives, with probability at leat $1-3\delta$,
    \begin{align*}
        &\frac{1}{M}\Xi_5\\
        \leq&\frac{2U^2_i}{\lambda_{T,i}}
        +2G^2_i\left(1+\frac{K-J}{(J-1)M}\right)\left(\sum^{\frac{(K-J)^2}{(J-1)^2}}_{t=1}\lambda_{t,i}
        +\sum^T_{t=\frac{(K-J)^2}{(J-1)^2}+1}\lambda_{t,i}\right)
            +\frac{2(K-J)G_iU_i}{J-1}\ln\frac{1}{\delta}+3G_iU_i\sqrt{\frac{2(K-J)T}{(J-1)M}\ln\frac{1}{\delta}}\\
            \leq&\frac{2U^2_i}{\lambda_{T,i}}
        +G_iU_i\sqrt{1+\frac{K-J}{(J-1)M}}\left(\frac{K-J}{J-1}+\int^T_{t=\frac{(K-J)^2}{(J-1)^2}+1}\frac{1}{\sqrt{t}}\mathrm{d}~t\right)
            +\frac{2(K-J)G_iU_i}{J-1}\ln\frac{1}{\delta}+3G_iU_i\sqrt{\frac{2(K-J)T}{(J-1)M}\ln\frac{1}{\delta}}\\
        \leq&6U_iG_i\sqrt{\left(1+\frac{K-J}{(J-1)M}\right)T}
        +\frac{2(K-J)G_iU_i}{J-1}\ln\frac{1}{\delta}+3G_iU_i\sqrt{\frac{2(K-J)T}{(J-1)M}\ln\frac{1}{\delta}}.
    \end{align*}
    \textbf{Case 2}: $T\leq\frac{(K-J)^2}{(J-1)^2}$.\\
        In this case,
        we do not decompose $\Xi_{5,1}$
        and $\lambda_{t,i}=\frac{U_i}{2G_i\sqrt{\left(1+\frac{K-J}{(J-1)M}\right)\frac{(K-1)^2}{(J-1)^2}}}$.
        With probability at least $1-\delta$,
        \begin{align*}
            \Xi_{5,1}
            \leq&\frac{K-J}{(J-1)M}G^2_i\sum^T_{t=1}\lambda_{t,i}+\frac{(K-J)G_iU_i}{3(J-1)}\ln\frac{1}{\delta}
            +G_iU_i\sqrt{\frac{K-J}{2(J-1)M}T\ln\frac{1}{\delta}}.
        \end{align*}
        Furthermore,
        with probability at least $1-2\delta$,
    \begin{align*}
        &\frac{1}{M}\Xi_5\\
        \leq&\frac{2U^2_i}{\lambda_{T,i}}
        +2G^2_i\left(1+\frac{K-J}{(J-1)M}\right)\sum^{T}_{t=1}\lambda_{t,i}
            +\frac{5(K-J)G_iU_i}{3(J-1)}\ln\frac{1}{\delta}+4G_iU_i\sqrt{\frac{K-J}{(J-1)M}T\ln\frac{1}{\delta}}\\
        \leq&5U_iG_i\sqrt{\left(1+\frac{K-J}{(J-1)M}\right)}\cdot\frac{K-J}{J-1}
        +\frac{5(K-J)G_iU_i}{3(J-1)}\ln\frac{1}{\delta}+4G_iU_i\sqrt{\frac{K-J}{(J-1)M}T\ln\frac{1}{\delta}}.
    \end{align*}

    Combining the two cases gives,
    with probability at least $1-(M+5)\delta$,
    $$
        \frac{1}{M}\Xi_5
        \leq
        6U_iG_i\sqrt{\left(1+\frac{K-J}{(J-1)M}\right)}\left(\sqrt{T}+\frac{K-J}{J-1}\right)
        +\frac{2(K-J)G_iU_i}{J-1}\ln\frac{1}{\delta}+3G_iU_i\sqrt{2\frac{K-J}{(J-1)M}T\ln\frac{1}{\delta}}.
    $$

\subsection{Combining all}

    Combining the upper bounds on $\Xi_4$ and $\Xi_5$ gives an upper bound on the regret.\\
    With probability at least $1-\Theta\left(M\log(CT)+\log(CKT/M)\right)\cdot\delta$,
    \begin{align*}
        &\sum^T_{t=1}\sum^M_{j=1}\ell\left(f^{(j)}_{t,A_{t,1}}(\mathbf{x}^{(j)}_t),y^{(j)}_t\right)
            -\sum^T_{t=1}\sum^M_{j=1}\ell\left(f(\mathbf{x}^{(j)}_t),y^{(j)}_t\right)\\
        \leq&
        M\sqrt{\left(1+\frac{K-J}{(J-1)M}\right)}\left(6U_iG_i
        \left(\sqrt{T}+\frac{K-1}{J-1}\right)+\frac{4}{\sqrt{2}}C_i\sqrt{T\ln{(KT)}}\right)+\\
        &(64C+3MC+2U_iG_i)g_{K,J}\ln\frac{1}{\delta}+(3\sqrt{2}G_iU_i+17\sqrt{CC_i})
        \sqrt{2Mg_{K,J}T\ln\frac{1}{\delta}}
        +2MC_i\sqrt{T\ln\frac{1}{\delta}}.
    \end{align*}
    Omitting the constant terms and lower order terms concludes the proof.

\section{Proof of Theorem \ref{thm:NeurIPS24:lower_bound}}

    We first establish a technical lemma.
    \begin{lemma}
    \label{lem:NIPS24:expection_gaussian_variable_maximum}
        Let $X_1,...,X_K$ be a sequence of independent standard normal random variables.
        Let $Z_K=\max\{X_1,...,X_K\}$.
        If $K\geq 5$,
        then $\mathbb{E}[Z_K]\geq \left(1-\frac{1}{\sqrt{\mathrm{e}}}\right)\sqrt{2\ln{K}}$.
    \end{lemma}

    \begin{proof}[Proof of Lemma \ref{lem:NIPS24:expection_gaussian_variable_maximum}]
        Proposition 2.1.2 in \cite{Roman2018High}
        gives a lower bound on the tail probability of standard normal distribution.
        \begin{align*}
            \forall x>0, \mathbb{P}[X_1\geq x]
            =\int^{+\infty}_x\frac{1}{\sqrt{2\pi}}\exp\left(-\frac{\mu^2}{2}\right)\mathrm{d}\,\mu
            \geq \frac{1}{\sqrt{2\pi}}\left(\frac{1}{x}-\frac{1}{x^3}\right)\exp\left(-\frac{x^2}{2}\right).
        \end{align*}
        Then we have
        \begin{align*}
            \mathbb{E}[Z_K]=&\mathbb{E}\left[Z_K\vert \exists i\in[K], X_i\geq \varepsilon\right]\cdot
            \mathbb{P}\left[\exists i\in[K], X_i\geq\varepsilon\right]
            +\mathbb{E}[Z_K\vert \forall X_i< \varepsilon]\cdot
            \mathbb{P}\left[\forall X_i< \varepsilon\right]\\
            \geq&\mathbb{P}\left[\exists i\in[K], X_i\geq \varepsilon\right]\cdot \varepsilon\\
            =&\left(1-\mathbb{P}\left[\forall X_i <\varepsilon\right]\right)\cdot \varepsilon\\
            =&\left(1-\prod^K_{i=1}\mathbb{P}\left[X_i <\varepsilon\right]\right)\cdot \varepsilon\\
            =&\left(1-\prod^K_{i=1}\left(1-\mathbb{P}\left[X_i \geq\varepsilon\right]\right)\right)\cdot \varepsilon\\
            \geq&\left(1-\left(1-
            \frac{1}{\sqrt{2\pi}}\left(\frac{1}{\varepsilon}-\frac{1}{\varepsilon^3}\right)
            \exp\left(-\frac{\varepsilon^2}{2}\right)\right)^K\right)\cdot \varepsilon.
        \end{align*}
        Let $\varepsilon=\sqrt{2\ln{K}}$. If $K>5$,
        then we have
        \begin{align*}
            \left(1-\frac{1}{\sqrt{2\pi}}\left(\frac{1}{\varepsilon}-\frac{1}{\varepsilon^3}\right)
            \exp\left(-\frac{\varepsilon^2}{2}\right)\right)^K
            =& \left(1-\frac{1}{\sqrt{2\pi}}\left(\frac{1}{\sqrt{2\ln{K}}}-\frac{1}{\ln^{1.5}{K^2}}\right)\frac{1}{K}\right)^K\\
            \leq&\left(1-\frac{1}{K^2}\right)^K\\
            \leq&\frac{1}{\sqrt{\mathrm{e}}}.
        \end{align*}
        Substituting into the lower bound of $\mathbb{E}[Z_K]$ concludes the proof.
    \end{proof}

\subsection{Proof of the First Lower Bound}

    \begin{proof}
        Let $d\geq K$, $\mathcal{X}\subseteq\mathbb{R}^d$ and $\mathcal{Y}\in\{0,1\}$.
        We use the absolute loss function $\ell(f({\bf x}_t),y_t)=\vert f({\bf x}_t)-y_t\vert$.
        Recalling that
        \begin{align*}
            \mathcal{F}_i=\left\{f_i({\bf x})=\langle {\bf e}_i,{\bf x}\rangle\right\},\quad
            i=1,2,...,K,
        \end{align*}
        where ${\bf e}_i$ is the standard basis vector in $\mathbb{R}^d$.
        It is obvious that the time complexity of computing $f_i({\bf x})=x_i$ is $O(1)$.
        At each client $j$,
        let the selected hypothesis be $f^{(j)}_t$ and the prediction be $f^{(j)}_t({\bf x}^{(j)}_t)$.
        Since there are no computational constraints on each client,
        $f^{(j)}_t({\bf x}^{(j)}_t)$ can be a weighted combination of
        $K$ predictions, i.e., $f^{(j)}_t({\bf x}^{(j)}_t)=\sum^K_{i=1}w^{(j)}_{t,i}f_i({\bf x}^{(j)}_t)$.
        The time complexity of computing $f^{(j)}_t({\bf x}^{(j)}_t)$ is $O(K)$.
        We will follow the techniques used in the proof of Theorem 3.1 in \cite{Patel2023Federated}
        and Theorem 3.7 in \cite{Cesa-Bianchi2006Prediction}.

        Following the proof of Theorem 3.1 in \cite{Patel2023Federated},
        the adversary gives a sequence of same examples for each client.
        To be specific,
        we define
        \begin{align*}
            \left({\bf x}^{(j)}_t,y^{(j)}_t\right)=({\bf x}_t,y_t),\quad t=1,...,T,\quad j=1,...,M,
        \end{align*}
        where
        ${\bf x}_t=(b_{t,1},b_{t,2},...,b_{t,K},0,\ldots,0)\in\mathbb{R}^d$,
        and $b_{t,1},b_{t,2},...,b_{t,K}, y_t$ is a sequence of symmetric i.i.d. Bernoulli random variables, i.e., $\mathbb{P}[y_t=1]=\mathbb{P}[y_t=0]=\frac{1}{2}$.

        At any round $t$,
        the minimax regret against the best hypothesis can be simplified as follows
        \begin{align*}
            &\inf_{f^{(1)}_1,...,f^{(M)}_T}
            \sup_{({\bf x}^{(j)}_t,y^{(j)}_t),,j\in[M],t\in[T]}\max_{i\in[K]}\mathrm{Reg}_D(\mathcal{F}_i)\\
            \geq&\inf_{f^{(1)}_1,...,f^{(M)}_T}
            \sup_{({\bf x}_t,y_t),t\in[T]}\max_{i\in[K]}\mathrm{Reg}_D(\mathcal{F}_i)\\
            \geq&
            \inf_{f^{(1)}_1,...,f^{(M)}_T}\mathop{\mathbb{E}}_{({\bf x}_t,y_t),t\in[T]}\left[
            \sum^T_{t=1}\sum^M_{j=1}\ell\left(f^{(j)}_t({\bf x}_t),y_t\right)
            -\min_{i\in[K]}\sum^T_{t=1}\sum^M_{j=1}\ell\left(f_i({\bf x}_t),y_t\right)\right]\\
            =&\inf_{f^{(1)}_1,...,f^{(M)}_T}\mathop{\mathbb{E}}_{({\bf x}_t,y_t),t\in[T]}
            \left[\sum^T_{t=1}\sum^M_{j=1}\vert f^{(j)}_t({\bf x}_t)-y_t\vert
            -M\min_{i\in[K]}\sum^T_{t=1}\vert f_i({\bf x}_t)-y_t\vert\right]\\
            =&\frac{MT}{2}-M\mathop{\mathbb{E}}_{({\bf x}_t,y_t),t\in[T]}
            \left[\min_{i\in[K]}\sum^T_{t=1}\vert f_i({\bf x}_t)-y_t\vert\right]\\
            =&M\mathop{\mathbb{E}}_{({\bf x}_t,y_t),t\in[T]}
            \left[\max_{i\in[K]}\sum^T_{t=1}\left(\frac{1}{2}- f_i({\bf x}_t)\right)\cdot
            \left(1-2y_t\right)\right],
        \end{align*}
        in which $f_i({\bf x}_t)=b_{t,i}$ is a Bernoulli random variable and
        \begin{align*}
            \mathop{\mathbb{E}}_{({\bf x},y_t),t\in[T]}
            \left[\sum^T_{t=1}\sum^M_{j=1}\vert f^{(j)}_t({\bf x}_t)-y_t\vert\right]
            =\mathop{\mathbb{E}}_{y_t,t\in[T]}
            \left[\sum^T_{t=1}\sum^M_{j=1}y_t\right]
            =\frac{MT}{2}.
        \end{align*}
        We further obtain
        \begin{align*}
            \inf_{f^{(1)}_1,...,f^{(M)}_T}
            \sup_{({\bf x}^{(j)}_t,y^{(j)}_t),j\in[M],t\in[T]}\max_{i\in[K]}\mathrm{Reg}_D(\mathcal{F}_i)
            \geq&
            \frac{M}{2}\mathop{\mathbb{E}}_{\sigma_t,Z_{t,i},t\in[T],i\in[K]}
            \left[\max_{i\in[K]}\sum^T_{t=1}Z_{t,i}\cdot\sigma_t\right]\\
            =&\frac{M}{2}\mathop{\mathbb{E}}_{Z_{t,i},t\in[T],i\in[K]}\left[\max_{i\in[K]}\sum^T_{t=1}Z_{t,i}\right],
        \end{align*}
        where both $\{Z_{t,i}\}_{t\in[T],i\in[K]}$ and $\{\sigma_t\}_{t\in[T]}$
        are i.i.d. Rademacher random variables.

        By Lemma A.11 in \cite{Cesa-Bianchi2006Prediction},
        we obtain
        \begin{align*}
            \lim_{T\rightarrow \infty}\mathbb{E}\left[\max_{i\in[K]}\frac{1}{\sqrt{T}}\sum^T_{t=1}Z_{t,i}\right]
            =\mathbb{E}\left[\max_{i\in[K]}G_i\right],
        \end{align*}
        where $G_1,...,G_N$ are independent standard normal random variables.

        By Lemma \ref{lem:NIPS24:expection_gaussian_variable_maximum},
        we obtain
        \begin{align*}
            \lim_{T\rightarrow
            \infty}\inf_{f^{(1)}_1,...,f^{(M)}_T}
            \sup_{({\bf x}^{(j)}_t,y^{(j)}_t),j\in[M],t\in[T]}
            \max_{i\in[K]}\mathrm{Reg}_D(\mathcal{F}_i)
            \geq\frac{1}{2}\left(1-\frac{1}{\sqrt{\mathrm{e}}}\right)M\sqrt{2T\ln{K}},
        \end{align*}
        which concludes the proof.
    \end{proof}

\subsection{Proof of the Second Lower Bound}

    We mainly use the techniques in the proof of Theorem 2 in \cite{Seldin2014Prediction},
    but also require a new technique.
    The idea of our proof is to reduce the online model selection on each client to
    multi-armed bandit problem with additional observations.

    \begin{proof}
        Now we prove the second lower bound in Theorem \ref{thm:NeurIPS24:lower_bound}.

        Let $d\geq K$, $\mathcal{X}\subseteq\mathbb{R}^d$ and $\mathcal{Y}\in\{0,1\}$.
        We use a linear loss function $\ell(f({\bf x}_t),y_t)=1-y_tf({\bf x}_t)$.
        Recalling that
        \begin{align*}
            \mathcal{F}_i=\left\{f_i({\bf x})=\langle {\bf e}_i,{\bf x}\rangle\right\},\quad
            i=1,2,...,K.
        \end{align*}
        It is obvious that the time complexity of computing $f_i({\bf x})=x_i$ is $O(1)$.
        Under the constraint that the time complexity on each client is limited to $O(J)$,
        on each client, any algorithm can only select $J$ hypotheses and then output a prediction.

        One of challenges is that the prediction may be a combination of $J$ predictions.
        To be specific,
        for each client $j\in[M]$,
        $f^{(j)}_t({\bf x}^{(j)}_t)=\sum_{i\in O^{(j)}_t}w_{t,i}f_i({\bf x}^{(j)}_t)$,
        where $O^{(j)}_t$ contains the index of selected $J$ hypotheses by some algorithm.
        To address this challenge,
        we introduce a virtual strategy that randomly selects a hypothesis $f^{(j)}_{I^{(j)}_t}
        \in\{f_{A_{t,1}},f_{A_{t,2}},...,f_{A_{t,J}}\}$ following
        the distribution $(w_{t,A_{t,1}},w_{t,A_{t,2}},...,w_{t,A_{t,J}})$
        where $A_{t,a}\in O^{(j)}_t$, $a=1,...,J$.
        Since the loss function is a linear function,
        it is easy to prove that,
        \begin{align*}
            \mathbb{E}\left[\ell(f^{(j)}_{I^{(j)}_t}({\bf x}^{(j)}_t),y^{(j)}_t)\right]
            =\ell\left(\mathbb{E}\left[f^{(j)}_{I^{(j)}_t}({\bf x}^{(j)}_t)\right],y^{(j)}_t\right)
            =\ell\left(f^{(j)}_t({\bf x}^{(j)}_t),y^{(j)}_t\right),
        \end{align*}
        where the expectation is taken over $I^{(j)}_t$.
        Assuming that $\ell(f_i({\bf x}^{(j)}_t),y^{(j)}_t)\leq C$ for all $i=1,...,K$.
        Lemma A.7 in \cite{Cesa-Bianchi2006Prediction} gives,
        with probability at least $1-\delta$,
        \begin{align*}
            \sum^T_{t=1}\left[\ell(f^{(j)}_{I^{(j)}_t}({\bf x}^{(j)}_t),y^{(j)}_t)
            -\ell\left(f^{(j)}_t({\bf x}^{(j)}_t),y^{(j)}_t\right)\right]
            \leq -C\sqrt{\frac{T}{2}\ln\frac{1}{\delta}}.
        \end{align*}
        Note that we assume $w_{t,i}\geq 0$ and $\sum_{i\in O^{(j)}_t}w_{t,i}=1$ for all $t=1,...,T$.
        Recalling that Theorem \ref{thm:NeurIPS24:lower_bound} assumes the outputs of algorithm
        belong to $[\min_{i\in[K],{\bf x}\in\mathcal{X}}f_i({\bf x}),\max_{i\in[K],{\bf x}\in\mathcal{X}}f_i({\bf x})]$.
        If $w_{t,i}< 0$ or $\sum_{i\in O^{(j)}_t}w_{t,i}>1$,
        we can still find a weight vector $w'_{t,i}\geq 0$ and $\sum_{i\in O^{(j)}_t}w'_{t,i}=1$,
        such that
        \begin{align*}
            f^{(j)}_t({\bf x}^{(j)}_t)=\sum_{i\in O^{(j)}_t}w_{t,i}f_i({\bf x}^{(j)}_t)
            =\sum_{i\in O^{(j)}_t}w'_{t,i}f_i({\bf x}^{(j)}_t).
        \end{align*}
        Then we sample $I^{(j)}_t$ following $(w'_{t,A_{t,1}},w'_{t,A_{t,2}},...,w'_{t,A_{t,J}})$.
        We can replace $(w_{t,A_{t,1}},w_{t,A_{t,2}},...,w_{t,A_{t,J}})$ with
        $(w'_{t,A_{t,1}},w'_{t,A_{t,2}},...,w'_{t,A_{t,J}})$.

        Since the algorithm is non-cooperative,
        the total regret can be decomposed into the summation of the regret on each client.
        With probability at least $1-M\delta$,
        \begin{align}
            \forall i\in[K],\quad\mathrm{Reg}_D(\mathcal{F}_i)
            =&\sum^M_{j=1}
            \left[\sum^T_{t=1}\ell\left(f^{(j)}_t({\bf x}^{(j)}_t),y^{(j)}_t\right)
            -\sum^T_{t=1}\ell\left(f_i({\bf x}^{(j)}_t),y^{(j)}_t\right)\right]\nonumber\\
            =& \sum^M_{j=1}
            \left[\sum^T_{t=1}\ell\left(f^{(j)}_{I^{(j)}_t}({\bf x}^{(j)}_t),y^{(j)}_t\right)
            -\sum^T_{t=1}\ell\left(f_i({\bf x}^{(j)}_t),y^{(j)}_t\right)\right]+\nonumber\\
            &\sum^M_{j=1}
            \left[\sum^T_{t=1}\ell\left(f^{(j)}_t({\bf x}^{(j)}_t),y^{(j)}_t\right)
            -\sum^T_{t=1}\ell\left(f^{(j)}_{I^{(j)}_t}({\bf x}^{(j)}_t),y^{(j)}_t\right)\right]\nonumber\\
            \geq&\underbrace{\sum^M_{j=1}\left[
            \sum^T_{t=1}\ell\left(f^{(j)}_{I^{(j)}_t}({\bf x}^{(j)}_t),y^{(j)}_t\right)
            -\sum^T_{t=1}\ell\left(f_i({\bf x}^{(j)}_t),y^{(j)}_t\right)\right]}_
            {\overline{\mathrm{Reg}}_D(\mathcal{F}_i)}+
            C\sqrt{\frac{T}{2}\ln\frac{1}{\delta}}.
        \label{eq:NIPS24:lower_bound_decomposition}
        \end{align}
        If the prediction is not a combination of $J$ predictions,
        but just $f^{(j)}_{I^{(j)}_t}({\bf x}^{(j)}_t)$, then we have
        \begin{equation}
        \label{eq:NIPS24:lower_bound_decomposition_2}
            \forall i\in[K],\quad\mathrm{Reg}_D(\mathcal{F}_i)
            =\underbrace{\sum^M_{j=1}\left[
            \sum^T_{t=1}\ell\left(f^{(j)}_{I^{(j)}_t}({\bf x}^{(j)}_t),y^{(j)}_t\right)
            -\sum^T_{t=1}\ell\left(f_i({\bf x}^{(j)}_t),y^{(j)}_t\right)\right]}_
            {\overline{\mathrm{Reg}}_D(\mathcal{F}_i)}.
        \end{equation}
        Combining with the two cases,
        we just need to analyze $\overline{\mathrm{Reg}}_D(\mathcal{F}_i)$.

        The adversary first uniformly samples a same $h\in[K]$ for all clients,
        and then constructs $\{({\bf x}^{(j)}_t,y_t)\}^T_{t=1}$ as follows
        \begin{align*}
            {\bf x}^{(j)}_t={\bf x}_t:=
            \left(b_{t,1},b_{t,2},...,b_{t,K},0,\ldots,0\right),\quad y^{(j)}_t=1,
            \quad j=1,...,M,
        \end{align*}
        in which $b_{t,i}$ satisfies
        \begin{align*}
            \mathbb{P}_h\left[b_{t,i}=1\right]=&\frac{1-\rho}{2},\quad
            \mathbb{P}_h\left[b_{t,i}=0\right]=\frac{1+\rho}{2},\quad i\neq h,\\
            \mathbb{P}_h\left[b_{t,h}=1\right]=&\frac{1+\rho}{2},\quad
            \mathbb{P}_h\left[b_{t,h}=0\right]=\frac{1-\rho}{2}.
        \end{align*}
        Let $\mathbb{E}_h[\cdot]$ and $\mathbb{P}_h[\cdot]$ separately
        be the expectation and probability operator conditioned on $h$ is selected.
        Then we have
        \begin{align*}
            \mathbb{P}_h\left[\ell(f_i({\bf x}_t),1)=1\right]
            =&\frac{1+\rho}{2},\quad
            \mathbb{P}_h\left[\ell(f_i({\bf x}_t),1)=0\right]=\frac{1-\rho}{2},\quad i\neq h,\\
            \mathbb{P}_h\left[\ell(f_h({\bf x}_t),1)=1\right]
            =&\frac{1-\rho}{2},\quad
            \mathbb{P}_h\left[\ell(f_h({\bf x}_t),1)=0\right]=\frac{1+\rho}{2}.
        \end{align*}
        It is obvious that online model selection can be reduced to a $K$-armed bandit problem,
        in which $f_i$ is the $i$-th arm.
        At each round $t$,
        let $I^{(j)}_t$ be the selected arm.
        Besides,
        any algorithm can select another $J-1$ arms.
        Thus any algorithm can observe $J$ losses.
        Let $O^{(j)}_t$ be the set of the selected $J$ arms.
        Note that $f^{(j)}_{I^{(j)}_t}=f_{I^{(j)}_t}$ for any $I^{(j)}_t\in[K]$.

        Assuming that the algorithm is deterministic,
        that is, $I^{(j)}_t$ and $O^{(j)}_t$ are determined by $\{I^{(j)}_\tau,O^{(j)}_\tau\}^{t-1}_{\tau=1}$
        and the observed losses.
        Let $N_{T,i}=\sum^T_{t=1}\mathbb{I}_{I^{(j)}_t=i}$.
        Taking expectation w.r.t. $(b_{t,1},...,b_{t,K})^T_{t=1}$ yields
        \begin{align*}
            &\mathbb{E}_h\left[\sum^T_{t=1}\ell\left(f^{(j)}_{I^{(j)}_t}({\bf x}_t),1\right)-
             \min_{i\in[K]}\sum^T_{t=1}\ell\left(f_i({\bf x}_t),1\right)\right]\\
            \geq&\mathbb{E}_h\left[
             \sum^T_{t=1}\ell\left(f^{(j)}_{I^{(j)}_t}({\bf x}_t),1\right)\right]-
             \min_{i\in[K]}\mathbb{E}_h\left[\sum^T_{t=1}\ell\left(f_i({\bf x}_t),1\right)\right]\\
            =&\rho\cdot \mathbb{E}_h\left[\sum^T_{t=1}\mathbb{I}_{I^{(j)}_t\neq h}\right]\\
            =&\rho T\cdot \left(1-\frac{1}{T}\mathbb{E}_h\left[N_{T,h}\right]\right).
        \end{align*}
        Following the techniques in the proof of Theorem 2 in \cite{Seldin2014Prediction},
        we have
        \begin{align*}
            \frac{1}{KT}\sum^K_{h=1}\mathbb{E}_h\left[N_{T,h}\right]
            \leq\frac{1}{K}+\sqrt{-\frac{JT}{K}\frac{2\rho^2}{1-\rho^2}}.
        \end{align*}
        Recalling that $T\geq K\geq 5$.
        Let $\rho=\frac{\sqrt{K}}{3\sqrt{JT}}$.
        We further have
        \begin{align}
            &\frac{1}{K}\sum^K_{h=1}\left[\mathbb{E}_h\left[
             \sum^T_{t=1}\ell\left(f^{(j)}_{I^{(j)}_t}({\bf x}_t),1\right)\right]-
             \min_{i\in[K]}\mathbb{E}_h\left[\sum^T_{t=1}\ell\left(f_i({\bf x}_t),1\right)\right]\right]\nonumber\\
            \geq&\rho T\cdot \left(1-\frac{1}{K}-\frac{3}{2}\rho\sqrt{\frac{JT}{K}}\right)\nonumber\\
            \geq& 0.1\frac{\sqrt{KT}}{\sqrt{J}}.
            \label{eq:NeurIPS24:lower_bound:J:one_client}
        \end{align}

        For any deterministic algorithm,
        we can prove
        \begin{align*}
            &\sup_{({\bf x}^{(j)}_t,y^{(j)}_t),t\in[T],j\in[M]}
            \max_{i\in[K]}\overline{\mathrm{Reg}}_D(\mathcal{F}_i)\\
            \geq&
            \sup_{({\bf x}_t,1),t\in[T],h\in[K]}
            \left[\sum^T_{t=1}\sum^M_{j=1}\ell\left(f^{(j)}_t({\bf x}_t),1\right)
            -\min_{i\in[K]}\sum^T_{t=1}\sum^M_{j=1}\ell\left(f_i({\bf x}_t),1\right)\right]\\
            =&
            \sup_{({\bf x}_t,1),t\in[T],h\in[K]}\left[
            \sum^T_{t=1}\sum^M_{j=1}\ell\left(f^{(j)}_t({\bf x}_t),1\right)
            -M\min_{i\in[K]}\sum^T_{t=1}\ell\left(f_i({\bf x}_t),1\right)\right]\\
            \geq&\sup_{h\in[K]}\mathop{\mathbb{E}_h}_{{\bf x}_t,t\in[T]}\left[\sum^M_{j=1}
            \left[
            \sum^T_{t=1}\ell\left(f^{(j)}_{I^{(j)}_t}({\bf x}_t),1\right)-
            \min_{i\in[K]}\sum^T_{t=1}\ell\left(f_i({\bf x}_t),1\right)\right]\right]\\
            \geq&\sup_{h\in[K]}\sum^M_{j=1}\left[\mathop{\mathbb{E}_h}_{{\bf x}_t,t\in[T]}
            \left[
            \sum^T_{t=1}\ell\left(f^{(j)}_{I^{(j)}_t}({\bf x}_t),1\right)\right]-
            \min_{i\in[K]}
            \mathop{\mathbb{E}_h}_{{\bf x}_t,t\in[T]}\left[\sum^T_{t=1}\ell\left(f_i({\bf x}_t),1\right)\right]\right]\\
            \geq&\mathop{\mathbb{E}}_{h\in[K]}\sum^M_{j=1}\left[
            \mathbb{E}_h\left[\sum^T_{t=1}\ell\left(f^{(j)}_{I^{(j)}_t}({\bf x}_t),1\right)\right]-
            \min_{i\in[K]}\mathbb{E}_h\left[\sum^T_{t=1}\ell\left(f_i({\bf x}_t),1\right)\right]\right]\\
            =&\sum^M_{j=1}\frac{1}{K}\sum^K_{h=1}\left[
            \mathbb{E}_h\left[\sum^T_{t=1}\ell\left(f^{(j)}_{I^{(j)}_t}({\bf x}_t),1\right)\right]-
            \min_{i\in[K]}\mathbb{E}_h\left[\sum^T_{t=1}\ell\left(f_i({\bf x}_t),1\right)\right]\right]\\
            \geq&0.1M\sqrt{\frac{K}{J}T},
        \end{align*}
        where the last inequality comes from \eqref{eq:NeurIPS24:lower_bound:J:one_client}.
        As claimed in the proof of Theorem 6.11 in \cite{Cesa-Bianchi2006Prediction},
        the lower bound of any randomized algorithm is same with that of any deterministic algorithm, i.e.,
        \begin{align*}
            &\sup_{({\bf x}^{(j)}_t,y^{(j)}_t),t\in[T],j\in[M]}
            \mathbb{E}\left[\max_{i\in[K]}\overline{\mathrm{Reg}}_D(\mathcal{F}_i)\right]\\
            =&\sup_{({\bf x}^{(j)}_t,y^{(j)}_t),t\in[T],j\in[M]}
            \left[\mathbb{E}\left[\sum^T_{t=1}\ell\left(f^{(j)}_{I^{(j)}_t}({\bf x}^{(j)}_t),y^{(j)}_t\right)\right]-
             \min_{i\in[K]}\sum^T_{t=1}\ell\left(f_i({\bf x}^{(j)}_t),y^{(j)}_t\right)\right]\\
            \geq&0.1M\frac{\sqrt{KT}}{\sqrt{J}},
        \end{align*}
        in which the expectation is taken over the internal randomness of algorithm.
        Substituting into \eqref{eq:NIPS24:lower_bound_decomposition} in which $C=1$,
        or \eqref{eq:NIPS24:lower_bound_decomposition_2}
        concludes the proof.
    \end{proof}

\section{Proof of Theorem \ref{thm:ICML2024:regret_bound:FOMD_DOKS_R}}

    \begin{proof}
        If FOMD-OMS $(R=T)$ runs on a sequence of examples with length $T=R$,
        then Theorem \ref{thm:ICML2024:regret_bound:FOMD_DOKS} gives,
        with probability at least $1-\Theta\left(M\log(CR)+\log(CKR/M)\right)\cdot\delta$,
        \begin{align*}
            \mathrm{Reg}_D&(\mathcal{F}_i)=
            O\left(
            MB_{i,1}\sqrt{\left(1+\frac{K-J}{(J-1)M}\right)R}+
            \frac{B_{i,2}(K-J)}{J-1}\ln\frac{1}{\delta}
            +B_{i,3}\sqrt{\frac{(K-J)MR}{J-1}\ln\frac{1}{\delta}}\right).
        \end{align*}
        According to Theorem \ref{lemma:ICML2024:regret_OMD_2},
        the regret bound of FOMD-OMS $(R<T)$ satisfies,
        with probability at least $1-\Theta\left(\frac{T}{R}M\log(CR)+\frac{T}{R}\log(CKR/M)\right)\cdot\delta$,
        \begin{align*}
            \mathrm{Reg}_D(\mathcal{F}_i)=&
            O\left(
            NMB_{i,1}\sqrt{\left(1+\frac{K-J}{(J-1)M}\right)R}+
            N\frac{B_{i,2}(K-J)}{J-1}\ln\frac{1}{\delta}
            +NB_{i,3}\sqrt{\frac{(K-J)MR}{J-1}\ln\frac{1}{\delta}}\right)\\
            =&O\left(
            \frac{T}{\sqrt{R}}MB_{i,1}\sqrt{1+\frac{K-J}{(J-1)M}}+
            \frac{T}{R}\cdot\frac{B_{i,2}(K-J)}{J-1}\ln\frac{1}{\delta}
            +\frac{T}{\sqrt{R}}B_{i,3}\sqrt{\frac{(K-J)M}{J-1}\ln\frac{1}{\delta}}\right),
        \end{align*}
        which concludes the proof.
    \end{proof}

\section{Proof of Theorem \ref{thm:ICML2024:OKS}}

\subsection{Algorithm}

    We give the pseudo-code in Algorithm \ref{alg:ICML2024:DOKS}.

    To implement Algorithm \ref{alg:ICML2024:DOKS},
    we require one more technique,
    i.e., the random features \cite{Rahimi2007Random}.
    We will use the random features to construct an approximation of the implicity kernel mapping.
    The are two reasons.
    The first one is that we can avoid transferring the data itself
    and thus the privacy is protected.
    The second one is that we can avoid the $O(T)$ computational cost on the clients.

    For any $i\in[K]$,
    we consider the kernel function $\kappa_i(\mathbf{x},\mathbf{v})$ that has an integral representation, i.e.,
    \begin{equation}
    \label{eq:IJCAI2022:definition:kernels}
        \kappa_i(\mathbf{x},\mathbf{v})=\int_{\Gamma}\varphi_i(\mathbf{x},\omega)\varphi_i(\mathbf{v},\omega)
        \mathrm{d}\,\mu_{i}(\omega),~\forall \mathbf{x},\mathbf{v}\in\mathcal{X},
    \end{equation}
    where
    $\varphi_i:\mathcal{X}\times \Gamma\rightarrow \mathbb{R}$
    is the eigenfunctions and
    $\mu_{i}(\cdot)$ is a distribution function on $\Gamma$.
    Let $p_{i}(\cdot)$ be the density function of $\mu_i(\cdot)$.
    We sample $\{\omega_j\}^{D}_{j=1}\sim p_i(\omega)$ independently and compute
    \begin{align*}
        \tilde{\kappa}_i(\mathbf{x},\mathbf{v})
        =\frac{1}{D}\sum^{D}_{j=1}\varphi_i(\mathbf{x},\omega_j)\varphi_i(\mathbf{v},\omega_j).
    \end{align*}
    For any $f(\mathbf{x})=\int_{\Gamma}\alpha(\omega)\varphi_i(\mathbf{x},\omega)p_i(\omega)\mathrm{d}\,\omega$.
    We can approximate $f(\mathbf{x})$
    by $\hat{f}(\mathbf{x})=\frac{1}{D}\sum^{D}_{j=1}\alpha(\omega_j)\varphi_i(\mathbf{x},\omega_j)$.
    It can be verified that
    $\mathbb{E}[\hat{f}(\mathbf{x})]=f(\mathbf{x})$.
    Such an approximation scheme also defines
    an explicit feature mapping denoted by
    \begin{align*}
        \phi_i(\mathbf{x})=\frac{1}{\sqrt{D}}
        \left(\varphi_i(\mathbf{x},\omega_1),\ldots,\varphi_i(\mathbf{x},\omega_D)\right).
    \end{align*}
    For each $\kappa_i, i\in[K]$,
    we define two hypothesis spaces \cite{Rahimi2008Ali,Li2022Improved} as follows
    \begin{equation}
    \label{eq:ICML24:restricted_RKHS}
    \begin{split}
        \mathcal{F}_i=&\left\{f(\mathbf{x})=\int_{\Gamma}\alpha(\omega)
        \varphi_i(\mathbf{x},\omega)p_{i}(\omega)\mathrm{d}\,\omega
        \left\vert \vert\alpha(\omega)\vert \leq U_i\right.\right\},\\
        \mathbb{H}_i=&\left\{\hat{f}(\mathbf{x})
        =\sum^{D}_{j=1}\alpha_j\varphi_{i}(\mathbf{x},\omega_j)
        \left\vert \vert \alpha_j\vert \leq \frac{U_i}{D}\right.\right\}\\
        =&\left\{\hat{f}(\mathbf{x})
        ={\bf w}^\top\phi_i({\bf x})
        \left\vert \mathbf{w}=\sqrt{D}(\alpha_1,\ldots,\alpha_{D})\in\mathbb{R}^{D},\vert \alpha_j\vert \leq \frac{U_i}{D}\right.\right\},
    \end{split}
    \end{equation}
    in which $\mathcal{F}_i$ is exact the hypothesis space defined in \eqref{eq:ICML2024:definition:hypothesis_space}.

    It can be verified that $\Vert \mathbf{w}\Vert^2_2\leq U^2_i$.
    Let $\mathcal{W}_i=\{{\bf w}\in\mathbb{R}^{D}:\Vert {\bf w}\Vert_{\infty}\leq \frac{U_i}{\sqrt{D}}\}$.
    We replace \eqref{eq:ICML2024:updating_hypothesis} with
    \eqref{eq:IJCAI2022:IOKS:type_2_updating_RF},
    \begin{equation}
    \label{eq:IJCAI2022:IOKS:type_2_updating_RF}
    \left\{
    \begin{split}
        \nabla_{\bar{{\bf w}}_{t+1,i}}\psi_{t,i}(\bar{{\bf w}}_{t+1,i})
        =&\nabla_{{\bf w}_{t,i}}\psi_{t,i}({\bf w}_{t,i})-\frac{1}{M}\sum^M_{j=1}\tilde{\nabla}^{(j)}_{t,i},\quad
        i=1,...,K,\\
        {\bf w}_{t+1,i}=&\mathop{\arg\min}_{{\bf w}\in\mathcal{W}_i}
        \mathcal{D}_{\psi_{t,i}}({\bf w},\bar{{\bf w}}_{t+1,i}),\\
        \psi_{t,i}({\bf w})=&\frac{1}{2\lambda_{t,i}}\cdot\Vert{\bf w}\Vert^2_2.\\
    \end{split}
    \right.
    \end{equation}

    \begin{algorithm}[!t]
        \caption{\small{FOMD-OMS for Distributed OMKL}}
        \footnotesize
        \label{alg:ICML2024:DOKS}
        \begin{algorithmic}[1]
        \REQUIRE{$U$, $T$, $R$, $J$.}
        \ENSURE{$f^{(j)}_{1,i}=0$, $p_{1,i}$, $i\in[K]$, $j\in[M]$}
        \FOR{$r=1,2,\ldots,R$}
            \FOR{$t\in T_r$}
                \IF{$t==(r-1)N+1$}
                    \FOR{$j=1,\ldots,M$}
                        \STATE Server samples $O^{(j)}_t$ following \eqref{eq:ICML2024:kernel_selection}
                        \STATE Server transmits ${\bf w}_{t,i},i\in O^{(j)}_t$ to the $j$-th client
                    \ENDFOR
                \ENDIF
                \FOR{$j=1,\ldots,M$ in parallel}
                    \FOR{$i\in O^{(j)}_t$}
                        \STATE Computing $\phi_i({\bf x}^{(j)}_t)$
                    \ENDFOR
                    \STATE Outputting ${\bf w}^\top_{t,A_{t,1}}\phi_{A_{t,1}}({\bf x}^{(j)}_t)$
                            and receiving $y^{(j)}_t$
                    \FOR{$i\in O^{(j)}_t$}
                        \STATE Computing ${\nabla}^{(j)}_{t,i}$ and ${c}^{(j)}_{t,i}$
                    \ENDFOR
                \ENDFOR
            \IF{$t==rN$}
                \STATE Clients transmit
                        $\{\frac{1}{N}\sum_{t\in T_r}\nabla^{(j)}_{t,i},
                        \frac{1}{N}\sum_{t\in T_r}c^{(j)}_{t,i}\}_{i\in O^{(j)}_t}$ to server
                \STATE Server computes ${\bf p}_{t+1}$ following \eqref{eq:ICML2024:updating_probability}
                \STATE Server computes ${\bf w}_{t+1,i},i\in[K]$
                        following \eqref{eq:IJCAI2022:IOKS:type_2_updating_RF}
            \ENDIF
            \ENDFOR
        \ENDFOR
        \end{algorithmic}
    \end{algorithm}

\subsection{Regret Analysis}

    We first give an assumption and a technique lemma.

  \begin{assumption}[\cite{Li2019Towards}]
    \label{ass:IJCAI2022:bounded_feature_mapping}
        For any $i\in[K]$, if $\kappa_i$ satisfies \eqref{eq:IJCAI2022:definition:kernels},
        then there is a bounded constant $b_i$ such that,
        $\forall \mathbf{x}\in\mathcal{X}$,
        $\vert \varphi_i(\mathbf{x},\omega)\vert\leq b_i$.
    \end{assumption}

    Under Assumption \ref{ass:IJCAI2022:bounded_feature_mapping},
    we have
    $\vert f(\mathbf{x})\vert \leq U_ib_i$ for any $f\in\mathbb{H}_i$ and $f\in\mathcal{F}_i$.
    It is worth mentioning that
    if Assumption \ref{ass:IJCAI2022:bounded_feature_mapping} holds,
    then Assumption \ref{ass:ICML24:assumption} holds with the same $b_i$.

    \begin{lemma}
    \label{lemma:IJCAI2022:approximate_error}
        For any $i\in[K]$,
        let $\mathcal{F}_i$ and $\mathbb{H}_i$ follow \eqref{eq:ICML24:restricted_RKHS}.
        With probability at least $1-\delta$,
        $\forall f\in\mathcal{F}_i$,
        there is a $\hat{f}\in\mathbb{H}_i$
        such that $\vert f(\mathbf{x})-\hat{f}(\mathbf{x})\vert
        \leq \frac{Ub_i}{\sqrt{D}}\sqrt{2\ln\frac{1}{\delta}}$.
    \end{lemma}

    The lemma is adopted from Lemma 5 in \cite{Li2023CoRRImproved}.
    Thus we omit the proof.

    Now we begin to prove Theorem \ref{thm:ICML2024:OKS}.
    \begin{proof}[Proof of Theorem \ref{thm:ICML2024:OKS}]
        The regret w.r.t. any $f\in\mathcal{F}_i$ can be decomposed as follows.
        \begin{align*}
                &\sum^T_{t=1}\sum^M_{j=1}\ell\left(f^{(j)}_{t,A_{t,1}}(\mathbf{x}^{(j)}_t),y^{(j)}_t\right)
                -\sum^T_{t=1}\sum^M_{j=1}\ell\left(f(\mathbf{x}^{(j)}_t),y^{(j)}_t\right)\\
            =&\underbrace{\sum^T_{t=1}\sum^M_{j=1}
            \left[\ell\left(f^{(j)}_{t,A_{t,1}}(\mathbf{x}^{(j)}_t),y^{(j)}_t\right)
            -\ell\left(f^{(j)}_{t,i}(\mathbf{x}^{(j)}_t),y^{(j)}_t\right)
            +\ell\left(f^{(j)}_{t,i}(\mathbf{x}^{(j)}_t),y^{(j)}_t\right)
            -\ell\left(\hat{f}(\mathbf{x}^{(j)}_t),y^{(j)}_t\right)\right]}_{\mathrm{Reg}_D(\mathbb{H}_i)}\\
            &+\underbrace{\sum^T_{t=1}\sum^M_{j=1}\left[\ell\left(\hat{f}(\mathbf{x}^{(j)}_t),y^{(j)}_t\right)
            -\ell\left(f(\mathbf{x}^{(j)}_t),y^{(j)}_t\right)\right]}_{\Xi_6}\\
            =& \mathrm{Reg}_D(\mathbb{H}_i)+\Xi_6.
        \end{align*}
        $\mathrm{Reg}_D(\mathbb{H}_i)$ is the regret that
        we run FOMD-OKS with hypothesis spaces $\{\mathbb{H}_i\}^K_{i=1}$.
        $\hat{f}\in\mathbb{H}_i$ satisfies Lemma \ref{lemma:IJCAI2022:approximate_error}.
        In other words,
        $\Xi_6$ is induced by the approximation error that we use $\hat{f}$ to approximate $f$.

        $\mathrm{Reg}_D(\mathbb{H}_i)$ has been given by
        Theorem \ref{thm:ICML2024:regret_bound:FOMD_DOKS_R}.
        Next we analyze $\Xi_6$.

        Using the convexity of $\ell(\cdot,\cdot)$,
        with probability at least $1-TM\delta$,
        \begin{align*}
            \Xi_6
            \leq& \sum^T_{t=1}\sum^M_{j=1}
            \frac{\mathrm{d}\,\ell\left(\hat{f}(\mathbf{x}^{(j)}_t),y^{(j)}_t\right)}
            {\mathrm{d}\,\hat{f}(\mathbf{x}^{(j)}_t)}
            \cdot\left(\hat{f}(\mathbf{x}^{(j)}_t)-f(\mathbf{x}^{(j)}_t)\right)\\
            \leq&\sum^T_{t=1}\sum^M_{j=1}\left\vert\frac{\mathrm{d}\,\ell\left(\hat{f}(\mathbf{x}^{(j)}_t),y^{(j)}_t\right)}
            {\mathrm{d}\,\hat{f}(\mathbf{x}^{(j)}_t)}\right\vert
            \cdot \left\vert\hat{f}(\mathbf{x}^{(j)}_t)-f(\mathbf{x}^{(j)}_t)\right\vert\\
            \leq&g_ib_iU_i\frac{MT}{\sqrt{D}}\sqrt{2\ln\frac{1}{\delta}}\\
            \leq&G_iU_i\frac{MT}{\sqrt{D}}\sqrt{2\ln\frac{1}{\delta}}.
        \end{align*}
        Under Assumption \ref{ass:IJCAI2022:bounded_feature_mapping},
        there is a constant $g_i$ such that $\left\vert\frac{\mathrm{d}\,\ell\left(\hat{f}(\mathbf{x}^{(j)}_t),y^{(j)}_t\right)}
            {\mathrm{d}\,\hat{f}(\mathbf{x}^{(j)}_t)}\right\vert\leq g_i$.
        The last inequality comes from the definition of Lipschitz constant
        (see Lemma \ref{lemma:ICML24:lipschitz_assumption}).

        Combining the upper bounds on $\mathrm{Reg}_D(\mathbb{H}_i)$ and $\Xi_6$
        concludes the proof.
    \end{proof}

\end{document}